\documentclass[thnuscript,10pt,number,sort&compress]{elsarticle}
\usepackage{lineno}
\usepackage{float}
\usepackage{calrsfs}
\usepackage{physics}
\usepackage{mathtools}  
\usepackage{amsmath}
\usepackage{amssymb}
\usepackage{tabulary}
\usepackage{booktabs}
\usepackage{hyperref}
\usepackage{geometry}
\usepackage[page]{appendix}
\usepackage{graphics}
\usepackage[labelfont=bf,justification=justified,singlelinecheck=false]{caption}
\usepackage{subcaption}
\usepackage{bbm}
\usepackage{tabularx}

\usepackage{indentfirst}

\usepackage{array}
\usepackage{longtable}
\usepackage{xcolor}

\usepackage{comment}

\usepackage[shortlabels]{enumitem}

\usepackage{wrapfig}

\usepackage{titlesec}
\titleformat{\paragraph}[runin]
{\normalfont\scshape\bfseries}{\theparagraph}{1em}{}

\usepackage{chngcntr}
\usepackage{subcaption}

\usepackage{blindtext}
\usepackage{amsmath,amsthm}

\usepackage{makecell}

\DeclareMathOperator*{\argmin}{arg\,min}

\DeclareMathAlphabet{\pazocal}{OMS}{zplm}{m}{n}

\usepackage[ruled,vlined,linesnumbered]{algorithm2e}
\SetKwComment{Comment}{$\triangleright$\ }{}
\usepackage{xcolor}

\newcommand{\RR}{\mathbb{R}} 
\newcommand{\bg}{\bar g}
\newcommand{\uc}{\textit{\underbar c}}
\newcommand{\bc}{\bar c}
\DeclareMathOperator{\vect}{vec}
\newcommand{\T}{^\top}
\DeclareMathOperator{\dom}{dom}

\newtheorem{assumption}{Assumption}

\theoremstyle{plain}
\ifx\theorem\undefined
\newtheorem{theorem}{Theorem}
\newtheorem{lemma}{Lemma}
\newtheorem{corollary}{Corollary}
\newtheorem{proposition}{Proposition}
\theoremstyle{definition}
\newtheorem{definition}{Definition}

\newcommand{\nc}{the optimal stationarity condition}
\newcommand{\msec}{MSE relation condition}

\begin{document}
	\begin{frontmatter}
		
		\title{Meta-learning PINN loss functions}
		
		\author[mymainaddress]{Apostolos~F~Psaros}
		\author[Kenjiaddress]{Kenji~Kawaguchi}
		\author[mymainaddress]{George~Em~Karniadakis\corref{mycorrespondingauthor}}\cortext[mycorrespondingauthor]{Corresponding Author}\ead{george_karniadakis@brown.edu}
		\address[mymainaddress]{Division of Applied Mathematics, Brown University, Providence, RI 02906, USA}
		\address[Kenjiaddress]{Center of Mathematical Sciences and Applications, Harvard University, Cambridge, MA 02138, USA}
		
		\begin{abstract}
			We propose a meta-learning technique for offline discovery of physics-informed neural network (PINN) loss functions.
			We extend earlier works on meta-learning, and develop a gradient-based meta-learning algorithm for addressing diverse task distributions based on parametrized partial differential equations (PDEs) that are solved with PINNs.
			Furthermore, based on new theory we identify two desirable properties of meta-learned losses in PINN problems, which we enforce by proposing a new regularization method or using a specific parametrization of the loss function.  
			In the computational examples, the meta-learned losses are employed at test time for addressing regression and PDE task distributions. 
			Our results indicate that significant performance improvement can be achieved by using a shared-among-tasks offline-learned loss function even for out-of-distribution meta-testing. 
			In this case, we solve for test tasks that do not belong to the task distribution used in meta-training, and we also employ PINN architectures that are different from the PINN architecture used in meta-training. 
			To better understand the capabilities and limitations of the proposed method, we consider various parametrizations of the loss function and describe different algorithm design options and how they may affect meta-learning performance.						
		\end{abstract}
		
		
	\end{frontmatter}

\section{Introduction}
\subsection{Related work and motivation}

The physics-informed neural network (PINN) is a recently proposed method for solving forward and inverse problems involving partial differential equations (PDEs); see, e.g., \cite{raissi2019physicsinformed,lu2021deepxde,jagtap2020adaptive,meng2020composite,pang2020npinns,kharazmi2021hpvpinns,yang2021bpinns,cai2021physicsinformed,shukla2021parallel,jin2021nsfnets} for different versions of PINNs.
PINNs are based on (a) constructing a neural network (NN) approximator for the PDE solution that is inserted via automatic differentiation in the nonlinear operators describing the PDE, and (b) learning the solution by minimizing a composite objective function comprised of the residual terms for PDEs and boundary and initial conditions in the strong form; other PINN types in variational (weak) form have also been developed in \cite{kharazmi2021hpvpinns}.

Similar to solving supervised learning tasks with NNs, optimally solving PDEs with PINNs requires selecting the architecture, the optimizer, the learning rate schedule, and other hyperparameters (considered as such in a broad sense), each of which plays a different role in training.
Moreover, optimally enforcing the physics-based constraints, e.g., by controlling the number and locations of residual points for each term in the composite objective function introduces additional PINN-specific hyperparameters to be selected.   
Overall, partly because of the above and despite the conceptual simplicity of PINNs, the resulting learning problem is theoretically and practically challenging; see, e.g., \cite{wang2020understanding}.

In general, the loss function in NN training interacts with the optimization algorithm and affects both the convergence rate and the performance of the obtained minimum. 
In addition, the loss function interacts with the NN for shaping the loss landscape; i.e., the training objective as a function of the trainable NN parameters. 
As a result, from this point of view, deciding whether to use mean squared error (MSE), mean absolute error (MAE) or a different loss function can be considered as an additional hyperparameter to be selected; trial and error is often employed in practice, with MSE being the most popular option for PINNs.
For facilitating automatic selection, a parametrized loss function has been proposed in \cite{barron2019general}, which includes standard losses as special cases and can be optimized online for improving performance.
Although such an adaptive loss is shown in \cite{barron2019general} to be highly effective in the computer vision problems considered therein, it increases training computational cost and does not benefit from prior knowledge regarding the particular problem being solved. 
The following question, therefore, arises in the context of PINNs: \textit{Can we develop a framework for encoding the underlying physics of a parametrized PDE in a loss function optimized offline?}

Meta-learning is an emerging field that aims to optimize various parts of learning by infusing prior knowledge of a given task distribution such that new tasks drawn from the distribution can be solved faster and more accurately; see \cite{hospedales2020metalearning} for a comprehensive review. 
One form of meta-learning, namely gradient-based, alternates between an inner optimization in which dependence of updates on meta-learned parameters is tracked, and an outer optimization in which meta-learned parameters are updated based on differentiating the inner optimization paths. 
Such differentiation can be performed either exactly or approximately for reducing computational cost; see, e.g., \cite{grefenstette2019generalized,nichol2018firstorder,rajeswaran2019metalearning,lorraine2020optimizing}.  

In this regard, a recent research direction is concerned with loss function meta-learning, with diverse applications in supervised and reinforcement learning \cite{sung2017learning,houthooft2018evolved,wu2018learning,xu2018metagradient,zheng2018learning,antoniou2019learning,grabocka2019learning,huang2019addressing,zou2019reward,bechtle2020metalearning,gonzalez2020effective,gonzalez2020improved,kirsch2020improving,sicilia2021multidomain}.
Although different works utilize different meta-learning techniques and have different goals, it has been shown that loss functions obtained via meta-learning can lead to an improved convergence of the gradient-descent-based optimization. 
Moreover, meta-learned loss functions can improve test performance under few-shot and semi-supervised conditions as well as cases involving a mismatch between train and test distributions, or between train loss functions and test evaluation metrics (e.g., because of non-differentiability of the latter). 
For simplicity, we refer to meta-learned loss functions as \textit{learned losses} in this paper, while loss functions optimized during training are referred to as \textit{online adaptive losses}.

\subsection{Overview of the proposed method}

In this work, we propose a method for offline discovery via meta-learning of PINN loss functions by utilizing information from task distributions defined based on parametrized PDEs.
In the learned loss function, we encode information specific to the considered PDE task distribution by differentiating the PINN optimization path with respect to the loss function parametrization.
Discovering loss functions by differentiating the physics-informed optimization path can enhance our understanding about the complex problem of solving PDEs with PINNs.

Following the PDE task distribution definition and the learned loss parametrization, the learned loss parameters are optimized via meta-training by repeating the following steps until a stopping criterion is met: 
(a) PDE tasks are drawn from the task distribution; 
(b) in the inner optimization part, they are solved with PINNs for a few iterations using the current learned loss, and the gradient of the learned loss parameters is tracked throughout optimization; and 
(c) in the outer optimization part, the learned loss parameters are updated based on MSE of the final (optimized) PINN parameters. 
After meta-training, the obtained learned loss is used for meta-testing, which entails solving unseen tasks until convergence.   
A schematic illustration of the above alternating optimization procedure is provided in Fig.~\ref{fig:overview}.

As we demonstrate in the computational examples of the present paper, the proposed loss function meta-learning technique can improve PINN performance significantly even compared with the online adaptive loss proposed in \cite{barron2019general}, while also allocating the loss function training computational cost to the offline phase.
\begin{figure}[H]
	\centering
	\includegraphics[width=.7\linewidth]{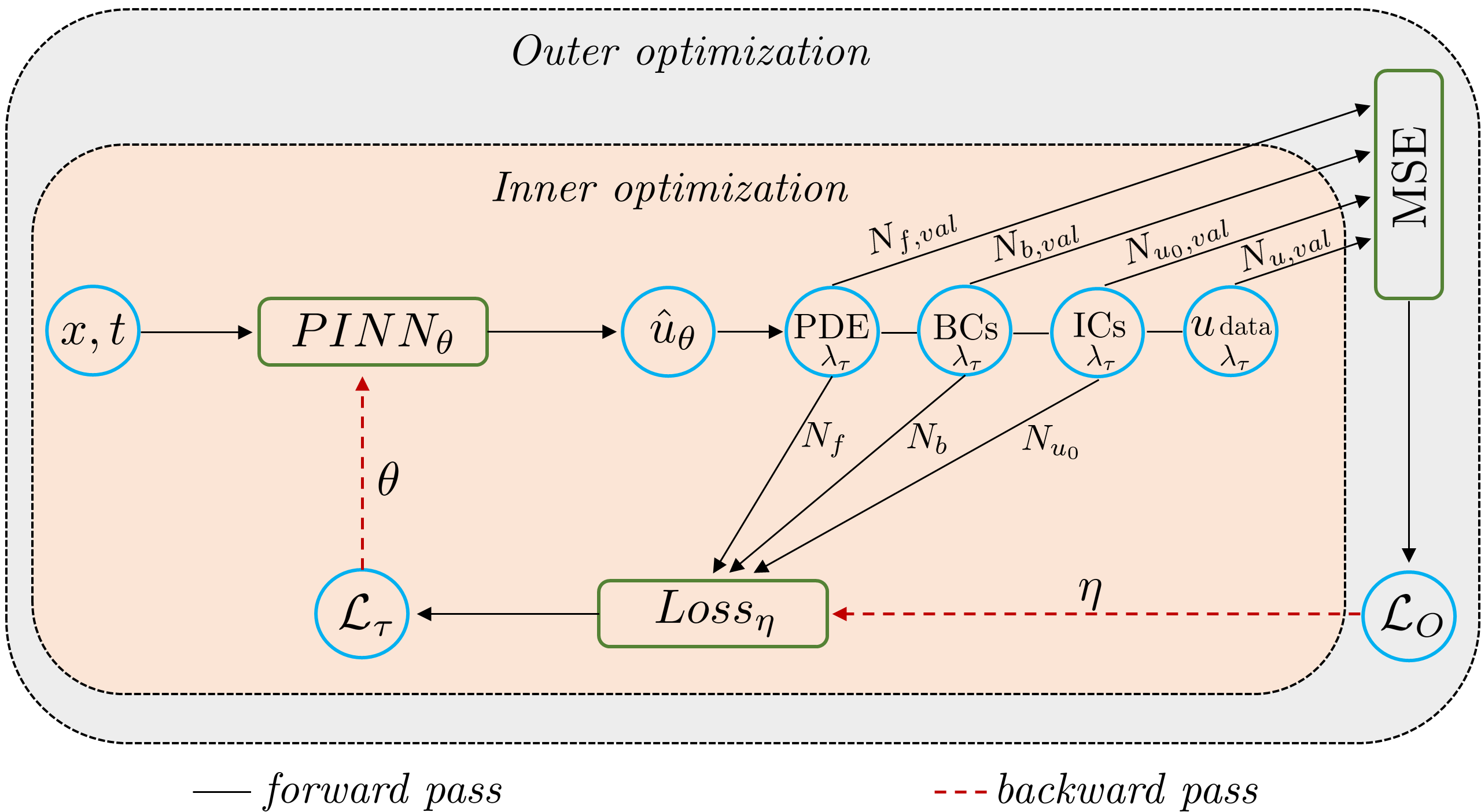}  
	\caption{Schematic illustration of the proposed meta-learning method for discovering PINN loss functions. 
		In the inner optimization part, the PINN parameters $\theta$ are updated based on $\pazocal{L}_{\tau}$, i.e., the current learned loss evaluated on $\{N_f, N_b, N_{u_0}\}$ datapoints corresponding to task $\lambda_{\tau}$; $N_f, N_b, N_{u_0}$ correspond to number of points for evaluating the residuals for PDE, boundary conditions (BCs), and initial conditions (ICs), respectively. 
		In the outer optimization part, the learned loss parameters $\eta$ are updated based on $\pazocal{L}_{O}$, i.e., the MSE of the optimized PINN parameters evaluated on $\{N_{f, val}, N_{b, val}, N_{u_0, val}, N_{u, val}\}$ datapoints; $N_{f, val}, N_{b, val}, N_{u_0, val}, N_{u, val}$ correspond to number of points for evaluating the PDE, BCs, ICs, and solution data (if available) residuals, respectively. 
	}
	\label{fig:overview}
\end{figure}

\subsection{Summary of innovative claims}

\begin{itemize}
	\item We propose a gradient-based meta-learning algorithm for offline discovery of PINN loss functions pertaining to diverse PDE task distributions.
	\item We extend the loss function meta-learning technique of \cite{bechtle2020metalearning} by considering alternative loss parametrizations and various algorithm design options.
	\item By proving two new theorems and a corollary, we identify two desirable properties of learned loss functions and propose a new regularization method to enforce them. 
	We also prove that the loss function parametrization proposed in \cite{barron2019general} is guaranteed to satisfy the two desirable properties.
	\item We define several representative benchmarks for demonstrating the performance of the considered algorithm design options as well as the applicability and the limitations of the proposed method.
\end{itemize}

\subsection{Organization of the paper}

We organize the paper as follows. 
In Section~\ref{sec:prelim} we provide a brief overview of PINNs for solving and discovering PDEs as well as of standard NN training.
In Section~\ref{sec:meta} we summarize meta-learning for PINNs, discuss our PINNs loss function meta-learning technique in detail, and present the theoretical results.
In Section~\ref{sec:examples}, we perform various computational experiments involving diverse PDE task distributions.
Finally, we summarize our findings in Section~\ref{sec:summary}, while the theorem proofs as well as additional design options and computational results are included in Appendices~\ref{app:meta:design:multi}-\ref{app:add:results}. 

\section{Preliminaries}\label{sec:prelim}
\subsection{PINN solution technique overview}
\label{sec:prelim:pinns}

Consider a general problem defined as 
\begin{subequations}\label{eq:param:pde}
	\begin{align}
		\pazocal{F}_{\lambda}[u](t, x) &= 0 \text{, } (t, x) \in [0, T] \times \Omega \label{eq:param:pde:a}\\
		\pazocal{B}_{\lambda}[u](t, x) & = 0 \text{, } (t, x) \in [0, T] \times \partial\Omega \label{eq:param:pde:b}\\
		u(0, x) & = u_{0, \lambda}(x) \text{, } x \in \Omega,   \label{eq:param:pde:c}
	\end{align}
\end{subequations}
where $\Omega \subset \mathbb{R}^{D_x}$ is a bounded domain with boundary $\partial \Omega$, $T > 0$, and $u(t, x) \in \mathbb{R}^{D_u}$ denotes the solution at $(t, x)$.
Eq.~\eqref{eq:param:pde:a} is a PDE expressed with a nonlinear operator $\pazocal{F}_{\lambda}[\cdot]$ that contains identity and differential operators as well as source terms; Eq.~\eqref{eq:param:pde:b} represents the boundary conditions (BCs) expressed with an operator $\pazocal{B}_{\lambda}[\cdot]$, and Eq.~\eqref{eq:param:pde:c} represents the initial conditions (ICs) expressed with a function $u_{0, \lambda}$.
In Eqs.~\eqref{eq:param:pde:a}-\eqref{eq:param:pde:c}, $\lambda$ represents the parametrization of the problem and is considered as shared among the operators $\pazocal{F}_{\lambda}$, $\pazocal{B}_{\lambda}$ and the function $u_{0, \lambda}$.
Furthermore, in practice Eqs.~\eqref{eq:param:pde:a}-\eqref{eq:param:pde:c} are often given in the form of collected data; i.e., they are only specified on discrete sets of locations that are subsets of $[0, T] \times \Omega$, $[0, T] \times \partial \Omega$ and $\Omega$, respectively.

In this regard, one problem scenario pertains to fixing the model parameters $\lambda$ and aiming to obtain the solution $u(t, x)$ for every $(t, x) \in [0, T] \times \Omega$. 
This problem is henceforth referred to as solving the PDE or as forward problem. 
Another problem scenario pertains to having observations from the solution $u(t, x)$ and aiming to obtain the parameters $\lambda$ that best describe the data. 
This problem is henceforth referred to as discovering the PDE or as inverse problem. 
PINNs were proposed in \cite{raissi2019physicsinformed} for addressing both problem scenarios.
In the case of a forward problem, the solution $u$ is represented with a NN $\hat{u}$, which is trained such that a composite objective comprised of the strong-form (PDE), boundary and initial residuals corresponding to Eqs.~\eqref{eq:param:pde:a}-\eqref{eq:param:pde:c}, respectively, on a discrete set of points is minimized.
For addressing the inverse problem with the PINNs solution technique of \cite{raissi2019physicsinformed}, an additional term corresponding to the solution $u$ data misfit is added to the composite objective and $\hat{u}$ is trained simultaneously with $\lambda$. 

Equivalently, we can view PINNs from a data-driven perspective. 
By defining $f(t, x)$ as the output of the left-hand side of Eq.~\eqref{eq:param:pde} for an arbitrary function $u(t, x)$, i.e., 
\begin{equation}\label{}
	f(t, x) := \pazocal{F}_{\lambda}[u](t, x),
\end{equation}
the operator $\pazocal{F}_{\lambda}: u \mapsto f$ can be construed as a map from $V_u$, the set of all admissible functions $u$ to $V_f$, the image of $V_u$ under $\pazocal{F}_{\lambda}$.
In this context, the function $u$ that satisfies Eq.~\eqref{eq:param:pde:a} corresponds to a mapped function $f \in V_f$ that is zero for every $(t, x) \in [0, T] \times \Omega$.
As a result, satisfying Eq.~\eqref{eq:param:pde:a} is equivalent to having observations from this zero-outputting $f$ at every $(t, x) \in [0, T] \times \Omega$ (or in a subset of it). 
Furthermore, satisfying Eq.~\eqref{eq:param:pde:b} is equivalent to having observations in $[0, T] \times \Omega$ from a zero-outputting $b$ function defined as 
\begin{equation}\label{}
	b(t, x) := \pazocal{B}_{\lambda}[u](t, x)
\end{equation}
and satisfying Eq.~\eqref{eq:param:pde:c} to having observations in $\{t = 0\} \times \Omega$ from the solution $u$.
If Eqs.~\eqref{eq:param:pde:a}-\eqref{eq:param:pde:c} are defined analytically for every point of the domains $[0, T] \times \Omega$, $[0, T] \times \partial \Omega$ and $\Omega$, respectively, a finite dataset is considered in practice.  

In this context, PINNs address the forward problem by (a) constructing three NN approximators $\hat{u}$, $\hat{f}$, and $\hat{b}$ that are connected via both the operators $\pazocal{F}_{\lambda}$ and $\pazocal{B}_{\lambda}$ and parameter sharing, i.e., $\hat{f}_{\theta, \lambda} = \pazocal{F}_{\lambda}[\hat{u}_{\theta}]$ and $\hat{b}_{\theta, \lambda} = \pazocal{B}_{\lambda}[\hat{u}_{\theta}]$, where $\theta$ are the shared parameters of the NN approximators, and (b) by training $\hat{u}$, $\hat{f}$, and $\hat{b}$ (simultaneously because of parameter sharing) to fit the complete dataset.
For the inverse problem, in which observations of the solution $u$ for times $t$ other than zero are also available, the same three connected approximators $\hat{u}$, $\hat{f}$, and $\hat{b}$ are constructed, but the additional constraint of fitting the $u$ observations is also included for learning $\lambda$ too. 

Overall, learning the approximators $\hat{u}$, $\hat{f}$, and $\hat{b}$ (and, potentially, $\lambda$ too) takes the form of a minimization problem expressed as
\begin{equation}\label{eq:pinns:loss:1}
	\min_{\theta \ (, \ \lambda)} \pazocal{L}_f(\theta, \lambda) +
	\pazocal{L}_b(\theta, \lambda) 
	+ \pazocal{L}_{u_0}(\theta, \lambda)
	+ \pazocal{L}_u(\theta),
\end{equation}
where $\pazocal{L}_f$ represents the loss related to the $f$ data, i.e., the physics of Eq.~\eqref{eq:param:pde:a}, $\pazocal{L}_b(\theta, \lambda)$ and $\pazocal{L}_{u_0}(\theta, \lambda)$ the losses related to the BCs and ICs, respectively, and $\pazocal{L}_u$ the loss related to the $u$ data. 
Next, considering $\{N_f, N_b, N_{u_0}, N_u\}$ datapoints, the terms $\{\pazocal{L}_f, \pazocal{L}_b, \pazocal{L}_{u_0}, \pazocal{L}_u\}$ in Eq.~\eqref{eq:pinns:loss:1} expressed as the weighted average dataset errors become 
\begin{subequations}\label{eq:pinns:loss:2}
	\begin{align}
		\pazocal{L}_f &= \frac{w_f}{N_f}\sum_{i=1}^{N_f}\ell(\hat{f}_{\theta, \lambda}(t_i, x_i), 0) \label{eq:pinns:loss:2:a}\\
		\pazocal{L}_b & = \frac{w_b}{N_b}\sum_{i=1}^{N_b}\ell(\hat{b}_{\theta, \lambda}(t_i, x_i), 0) \label{eq:pinns:loss:2:b}\\
		\pazocal{L}_{u_0} & = \frac{w_{u_0}}{N_{u_0}}\sum_{i=1}^{N_{u_0}}\ell(\hat{u}_{\theta}(0, x_i), u_{0, \lambda}(x_i))   \label{eq:pinns:loss:2:c} \\
		\pazocal{L}_u & = \frac{w_u}{N_u}\sum_{i=1}^{N_u}\ell(\hat{u}_{\theta}(t_i, x_i), u(t_i, x_i)). \label{eq:pinns:loss:2:d}
	\end{align}
\end{subequations}
In Eq.~\eqref{eq:pinns:loss:2}, $\ell(prediction, target)$, with $\ell: D_u \times D_u \to \mathbb{R}_{\geq 0}$, is a loss function that takes as input the NN prediction at a domain point and the target value, and outputs the corresponding loss; $\hat{u}_{\theta}(t_i, x_i)$, $\hat{f}_{\theta, \lambda}(t_i, x_i)$, $\hat{b}_{\theta, \lambda}(t_i, x_i)$ denote the NN predictions $\hat{u}_{\theta}$ and $\hat{f}_{\theta, \lambda} = \pazocal{F}_{\lambda}[\hat{u}_{\theta}]$, $\hat{b}_{\theta, \lambda} = \pazocal{B}_{\lambda}[\hat{u}_{\theta}]$ evaluated at the $i^{th}$ domain point $(t_i, x_i)$, respectively; $u_{0,\lambda}(x_i)$ denotes $i^{th}$ target value corresponding to the ICs function $u_{0,\lambda}$; and $u(t_i, x_i)$ denotes the $i^{th}$ target value corresponding to solution $u$ (if available). 
Clearly, the point sets $\{t_i, x_i\}_{i=1}^{N_{f}}$, $\{t_i, x_i\}_{i=1}^{N_{b}}$, $\{t_i, x_i\}_{i=1}^{N_{u_0}}$, and $\{t_i, x_i\}_{i=1}^{N_{u}}$ represent domain points at different locations, although the same symbol $(t_i, x_i)$ has been used for all of them, for notation simplicity.
Note that optimal weights $\{w_f, w_b, w_{u_0}, w_u\}$ to be used in Eq.~\eqref{eq:pinns:loss:1} are not known a priori and are often set in practice as equal to one or obtained via trial and error; see \cite{wang2020understanding} for discussion and an adaptive method for addressing this issue.

Finally, by considering as $\ell$ the squared $\ell_2$-norm of the discrepancy between predictions and targets, i.e., MSE if it is averaged over the dataset, Eq.~\eqref{eq:pinns:loss:2} reduces to 
\begin{subequations}\label{eq:pinns:loss:3}
	\begin{align}
		\pazocal{L}_f &= \frac{w_f}{N_f}\sum_{i=1}^{N_f}
		||\hat{f}_{\theta, \lambda}(t_i, x_i)||_2^2 \label{eq:pinns:loss:3:a}\\
		\pazocal{L}_b & = \frac{w_b}{N_b}\sum_{i=1}^{N_b}
		||\hat{b}_{\theta, \lambda}(t_i, x_i)||_2^2 \label{eq:pinns:loss:3:b}\\
		\pazocal{L}_{u_0} & = \frac{w_{u_0}}{N_{u_0}}\sum_{i=1}^{N_{u_0}}
		||\hat{u}_{\theta}(0, x_i) - u_{0, \lambda}(x_i)||_2^2   \label{eq:pinns:loss:3:c} \\
		\pazocal{L}_u & = \frac{w_u}{N_u}\sum_{i=1}^{N_u}
		||\hat{u}_{\theta}(t_i, x_i) - u(t_i, x_i)||_2^2. \label{eq:pinns:loss:3:d}
	\end{align}
\end{subequations}
We note that the terms MSE and squared $\ell_2$-norm of the discrepancy are used interchangeably in this paper depending on the context. 

\subsection{Loss functions in neural network training}\label{sec:prelim:loss}

This section serves as a brief summary of standard NN training and as an introduction to the meta-learning loss functions Section~\ref{sec:meta}.
As described in Section~\ref{sec:prelim:pinns}, addressing forward and inverse PDE problems with PINNs requires solving the minimization problem of Eq.~\eqref{eq:pinns:loss:1} with a composite objective function comprised of the loss terms of Eq.~\eqref{eq:pinns:loss:2}.
All of the functionals $\pazocal{L}_f(\theta, \lambda)$, $\pazocal{L}_b(\theta, \lambda)$, $\pazocal{L}_{u_0}(\theta, \lambda)$ and $\pazocal{L}_u(\theta)$ in Eq.~\eqref{eq:pinns:loss:2} are given as the average discrepancies over $f$, $b$, $u_0$ and $u$ data, respectively, and the total loss $\pazocal{L}$ is expressed as the weighted sum of the individual terms.
Therefore, we can study in this section, without loss of generality, each part of the sum separately by only considering the objective function
\begin{equation}\label{eq:gen:loss}
	\pazocal{L}(\theta) = \frac{1}{N}\sum_{i=1}^{N}\ell(\hat{u}_{\theta}(t_i, x_i), u(t_i, x_i)),
\end{equation}
where $\hat{u}_{\theta}(t_i, x_i)$ denotes the prediction value for each $(t_i, x_i)$ (i.e., either $\hat{f}_{\theta, \lambda}(t_i, x_i)$, $\hat{b}_{\theta, \lambda}(t_i, x_i)$, $\hat{u}_{\theta, \lambda}(0, x_i)$, or $\hat{u}_{\theta}(t_i, x_i)$ in Eq.~\eqref{eq:pinns:loss:2}) and $u(t_i, x_i)$ denotes the target (i.e., either $0$, $u_{0, \lambda}(x_i)$, or $u(t_i, x_i)$ in Eq.~\eqref{eq:pinns:loss:2}).   
For each $(t, x)$ in Eq.~\eqref{eq:param:pde}, $u(t, x)$ as well as $\pazocal{F}_\lambda[u](t,x)$ and $\pazocal{B}_\lambda[u](t,x)$ belong to $\mathbb{R}^{D_u}$; thus, $\hat{u}_{\theta}(t_i, x_i)$ and $u(t_i, x_i)$ in Eq.~\eqref{eq:gen:loss} are also considered $D_u$-dimensional.
Furthermore, $N$ represents the size of the corresponding dataset, i.e., $N_f$, $N_b$, $N_{u_0}$, or $N_u$. 

An optimization technique using only first-order information of the objective function is the (stochastic) gradient descent, which is typically denoted as SGD, whether or not stochastic gradients are used, by considering mini-batches of the data in Eq.~\eqref{eq:gen:loss}. 
Using SGD, the NN parameters $\theta$ are updated based on
\begin{equation}
	\theta \leftarrow \theta - \epsilon \nabla_{\theta}\pazocal{L},
\end{equation}
where $\epsilon$ is the learning rate,
\begin{equation}\label{eq:theta:grad}
	\nabla_{\theta}\pazocal{L} = \frac{\partial  \pazocal{L}}{\partial  \theta} = \left[\frac{\partial  \pazocal{L}}{\partial \theta_1}, \dots, \frac{\partial  \pazocal{L}}{\partial \theta_{D_{\theta}}}\right]
\end{equation}
is the gradient of $\pazocal{L}$ with respect to $\theta$, and $D_{\theta}$ is the number of NN parameters (weights and biases).
Employing the chain rule for each $\frac{\partial  \pazocal{L}}{\partial \theta_j}$, $j \in \{1,\dots,D_{\theta}\}$, Eq.~\eqref{eq:theta:grad} becomes
\begin{equation}\label{eq:theta:grad:2}
	\nabla_{\theta}\pazocal{L} =  \left[\left.\frac{1}{N}\sum_{i=1}^{N}\nabla_{q}\ell(q, u(t_i, x_i))\right\vert_{q=\hat{u}_{\theta}(t_i, x_i)}\right]J_{\hat{u}_{\theta}, \theta},
\end{equation}
where $\nabla_{q}\ell(q, u(t_i, x_i)) \in \mathbb{R}^{1 \times D_u}$ is the gradient of the scalar output $\ell(q, u(t_i, x_i))$ with respect to the prediction $q=\hat{u}_{\theta}(t_i, x_i)$.
Furthermore, $J_{\hat{u}_{\theta}, \theta} \in \mathbb{R}^{D_u \times D_{\theta}}$ is the Jacobian matrix 
\begin{equation}
	J_{\hat{u}_{\theta}, \theta} = \left[\nabla_{\theta}(\hat{u}_{\theta, 1}),\dots,\nabla_{\theta}(\hat{u}_{\theta, D_u})\right]\T
\end{equation} 
of the NN transformation $\theta \mapsto \hat{u}_{\theta}$.
If $\hat{u}_{\theta}$ is one-dimensional, Eq.~\eqref{eq:theta:grad:2} reduces to
\begin{equation}\label{eq:theta:grad:3}
	\nabla_{\theta}\pazocal{L} =  \left[\left.\frac{1}{N}\sum_{i=1}^{N}\frac{\partial \ell(q, u(t_i, x_i))}{\partial q}\right\vert_{q=\hat{u}_{\theta}(t_i, x_i)}\right]\nabla_{\theta}\hat{u}_{\theta}
\end{equation} 
and if, in addition, $\ell$ is the squared $\ell_2$-norm, to
\begin{equation}\label{eq:theta:grad:4}
	\nabla_{\theta}\pazocal{L} = \left[\frac{1}{N}\sum_{i=1}^{N}2(\hat{u}_{\theta}(t_i, x_i)-u(t_i, x_i))\right]\nabla_{\theta}\hat{u}_{\theta}.
\end{equation} 

The term enclosed in brackets in Eqs.~\eqref{eq:theta:grad:2}-\eqref{eq:theta:grad:4}, and more specifically the loss function $\ell$, controls how the objective function behaves for increasing discrepancies from the target.
If $\hat{u}_{\theta}$ is multi-dimensional, the loss function through $\left.\nabla_{q}\ell(q, u(t_i, x_i))\right\vert_{q=\hat{u}_{\theta}(t_i, x_i)}$ also controls how the discrepancy in each dimension affects the final gradient $\nabla_{\theta}\pazocal{L}$; note that each component of $\nabla_{\theta}\pazocal{L}$ is an inner product between the term inside brackets and $\frac{\partial  \hat{u}_{\theta}}{\partial \theta_j}$, $j \in \{1,\dots,D_{\theta}\}$.
For instance, by using the squared $\ell_2$-norm as $\ell$, which is given as the sum of squared discrepancies across dimensions, all dimensions are treated uniformly; see Appendix~\ref{app:meta:design:multi} regarding how a parametrized loss function for multi-dimensional inputs can be constructed. 

For other standard first-order optimization algorithms, such as AdaGrad and Adam, the parameter $\theta$ updates depend not only on the current iteration gradient $\nabla_{\theta}\pazocal{L}$ of Eqs.~\eqref{eq:theta:grad:2}-\eqref{eq:theta:grad:4}, but also on the $\theta$ updates history. 
For standard second-order algorithms, such as Newton's method and BFGS, $\theta$ updates depend not only on $\nabla_{\theta}\pazocal{L}$, but also on the Hessian or the approximate Hessian, respectively, of $\pazocal{L}$ with respect to $\theta$.   

\section{Meta-learning loss functions for PINNs}\label{sec:meta}

\subsection{Defining PDE task distributions}\label{sec:meta:dist}

In this section, we consider only the forward problem scenario as defined in Eq.~\eqref{eq:param:pde}.
As explained in Section~\ref{sec:prelim:pinns}, solving Eq.~\eqref{eq:param:pde} with PINNs for a value of $\lambda$ requires learning the PINN parameters $\theta$ by solving the minimization problem of Eq.~\eqref{eq:pinns:loss:1}.
Overall, a NN is constructed, an optimization strategy is selected and the optimization algorithm is run until some convergence criterion is met. 

However, all of the above steps, from defining the optimization problem to solving it, require the user to make certain design choices, which 
affect the overall training procedure and the final approximation accuracy. 
For example, they affect the convergence rate of the optimizer as well as the training and test error at the obtained minimum. 
For simplicity, all these aspects of training that depend on the selected hyperparameters are henceforth collectively called \textit{performance or efficiency} of the selected hyperparameters and of the training in general. 
For example, we may refer to a set of hyperparameters as being \textit{more efficient} than another set. 

Indicatively, the PINN architecture and activation function, the optimization algorithm, and the number of collocation points (points at which the PDE residual is evaluated) correspond to hyperparameters that must be tuned/selected a priori or be optimized in an online manner (see, e.g., \cite{jagtap2020adaptive} for adaptive activation functions).
In this regard, it is standard practice to experiment with many different hyperparameter settings by performing a few iterations of the optimization algorithm and by evaluating performance based on validation error; i.e., to perform hold-out validation.
In the context of PINNs, the validation error can be computed over collocation points not used in training or testing. 
After performing this trial-and-error procedure, the problem can then be fully solved; i.e., the optimization is run until convergence. 

For solving a novel, different PDE of the form of Eq.~\eqref{eq:param:pde}, hold-out validation is either repeated from scratch, or search is limited to a tight range of hyperparameter settings, depending on the experience of the user with the novel parameter $\lambda$ in Eq.~\eqref{eq:param:pde}.
Thus, it becomes clear that solving novel problems more efficiently requires:
\begin{enumerate}[(a)]
	\item To define families of related PDEs such that hyperparameters selected for solving one member of the family are expected to perform well also for other members.
	\item To effectively utilize information acquired from solving a few representative members of the family in order to solve other members efficiently.
\end{enumerate}

Families of related machine learning problems are typically called task distributions in the literature; see, e.g., \cite{hospedales2020metalearning}. 
For example, approximating the function $y = sin(x + \pi)$ is related to approximating $y = sin(x)$, in the sense that a hyperparameter setting found efficient for the former problem can be used as is, or with minimal modifications, for solving the latter one efficiently.
Therefore, a task distribution can, indicatively, be defined by functions of the form of $y = sin(x + \lambda)$, with $\lambda$ drawn uniformly from $[-\pi, \pi]$.
More generally, it is assumed that tasks are parametrized by $\lambda$, which is drawn from some distribution $p(\lambda)$.
In addition, each $\lambda$ defines a learning task, which can be shown experimentally or theoretically that is related to other learning tasks drawn from $p(\lambda)$.

Note that although the discussion up to this point has been motivated by the burden of repetitive hyperparameter tuning even for related problems, meta-learning has a rich range of applications that goes well beyond selecting NN architectures, learning rates and activation functions.
For example, optimizing the loss function or selecting a shared-among tasks NN initialization are not generally considered hyperparameter tuning. 
It is useful, however, to interpret all the different factors that affect the NN optimization procedure and the final performance as hyperparameters, some of which we try to optimize and some of which we fix.

\subsection{Meta-learning as a bi-level minimization problem}\label{sec:meta:bilevel}

As explained in Section~\ref{sec:meta:dist}, a task can be defined as a fixed PDE problem associated with a parameter $\lambda$ that is drawn from a pre-specified task distribution $p(\lambda)$.
Furthermore, in conjunction with task-specific training data, a task is solved with PINNs until convergence via Eqs.~\eqref{eq:pinns:loss:1}-\eqref{eq:pinns:loss:3}.
Final accuracy or overall performance of the training procedure can, indicatively, be evaluated based on final validation data error after training (e.g., PDE residual on a large set of points in the domain). 
Next, given a task distribution defined via $p(\lambda)$ and a set of optimizable hyperparameters $\eta$, meta-learning seeks for the optimal setting of $\eta$, where optimality is defined in terms of average performance across tasks drawn from $p(\lambda)$.
For example, if performance of PINN training is evaluated based on the $\ell_2$-norm of PDE residual on validation points, average performance of $\eta$ refers to average across $\lambda$ final $\ell_2$ error.

Following \cite{hospedales2020metalearning} and considering a finite set $\Lambda = \{\lambda_{\tau}\}_{\tau = 1}^{T}$ of $\lambda \sim p(\lambda)$ values, meta-learning is commonly formalized as a bi-level minimization problem expressed as
\begin{subequations}\label{eq:bilevel}
	\begin{align}
		\min_{\eta} \pazocal{L}_O(\theta^*(\eta), \eta) \label{eq:bilevel:a}\\
		\text{s.t. } \theta^*(\eta) = \{\theta^*_{\tau}(\eta)\}_{\tau = 1}^{T} \label{eq:bilevel:b}\\
		\text{with  } \theta^*_{\tau}(\eta) = \argmin_{\theta} \pazocal{L}_{\tau}(\theta, \eta), \label{eq:bilevel:c}
	\end{align}
\end{subequations} 
where Eq.~\eqref{eq:bilevel:a} is referred to as outer optimization, whereas Eqs.~\eqref{eq:bilevel:b}-\eqref{eq:bilevel:c} as inner optimization.
In Eq.~\eqref{eq:bilevel}, $\theta^*(\eta) = \{\theta^*_{\tau}(\eta)\}_{\tau = 1}^{T}$ consists of the final NN parameters $\theta^*$ obtained after solving all tasks in the set $\Lambda$ using hyperparameters $\eta$, and $\pazocal{L}_{\tau}(\theta, \eta)$, which is given by Eq.~\eqref{eq:pinns:loss:1}-\eqref{eq:pinns:loss:3}, is called inner- or base-objective.
The subscript $\tau$ and the arguments $\eta$, $\theta$, and $\theta^*$ in Eq.~\eqref{eq:bilevel} are used to signify task-dependent quantities as well as dependencies on (optimizable) $\eta$ and $\theta$, and (optimum) $\theta^*_{\tau}(\eta)$, respectively. 
Furthermore, $\pazocal{L}_O(\theta^*(\eta), \eta)$, which is called outer- or meta-objective and pertains to the ultimate goal of meta-learning, is commonly considered to be the average performance of $\eta$.
For example, if validation error is measured based on final $\ell_2$-norm of PDE residual on a validation set of size $N_{f, val}$, the outer-objective $\pazocal{L}_O(\theta^*(\eta), \eta)$ can be expressed as 
\begin{equation}\label{}
	\pazocal{L}_O(\theta^*(\eta), \eta) = \mathbb{E}_{\lambda_{\tau} \in \Lambda}\left[\frac{1}{N_{f, val}}\sum_{i=1}^{N_{f, val}}||\hat{f}_{\theta^*_{\tau}, \lambda_{\tau}}(t_i, x_i)||_2^2\right],
\end{equation}
where $N_{f, val}$ and domain points $\{(t_i, x_i)\}_{i=1}^{N_{f, val}}$ are considered the same for all tasks for notation simplicity and without loss of generality.
For more design choices regarding outer-objective $\pazocal{L}_O$ see Section~\ref{sec:meta:lossml} and \cite{hospedales2020metalearning}.

Next, for addressing the bi-level minimization problem of Eq.~\eqref{eq:bilevel}, an alternating approach between outer and inner optimization can be considered, which in practice can be achieved by performing one or only a few steps for each optimization. 
Three main families of techniques exist (\cite{hospedales2020metalearning}): gradient-based, reinforcement-learning-based, and evolutionary meta-learning. 
Discussions in this paper are limited to gradient-based approaches.
Although inner optimization corresponds to standard PINN training of Eqs.~\eqref{eq:pinns:loss:1}-\eqref{eq:pinns:loss:3}, gradient-based outer optimization requires computing the total derivative 
\begin{equation}\label{eq:many:steps:derivative}
	\boldsymbol{d}_{\eta}\pazocal{L}_O(\theta^*(\eta), \eta) = \nabla_{\eta} \pazocal{L}_O +\left[\nabla_{\theta^*}\pazocal{L}_O\right] J_{\theta^*(\eta), \eta},
\end{equation}
where $\boldsymbol{d}_{\eta}$ represents total derivative with respect to $\eta$, $\nabla_{\eta}$ and $\nabla_{\theta^*}$ here represent partial derivatives with respect to $\eta$ and $\theta^*(\eta)$, respectively, and $J_{\theta^*(\eta), \eta}$ is the Jacobian matrix of the transformation from $\eta$ to $\theta^*(\eta)$.
The first term on the right-hand side of Eq.~\eqref{eq:many:steps:derivative} refers to direct dependence of $\pazocal{L}_O$ on $\eta$ (e.g., as is the case for neural architecture search; see \cite{liu2019darts}), whereas the second term refers to dependence of $\pazocal{L}_O$ on $\eta$ through the obtained optimal $\theta^*(\eta)$. 
Because $\theta^*(\eta)$ is the result of a number of inner optimization steps, obtaining the Jacobian $J_{\theta^*(\eta), \eta}$ via chain rule is often referred to as \textit{differentiating over the optimization path}. 
In this regard, see Section~\ref{sec:meta:design:inner} for addressing the exploding gradients pathology arising because of path differentiation in loss function meta-learning as well as \cite{nichol2018firstorder,rajeswaran2019metalearning,lorraine2020optimizing} on approximate ways of computing Eq.~\eqref{eq:many:steps:derivative}.
Overall, the computation of Eq.~\eqref{eq:many:steps:derivative} can be performed by designing for this purpose automatic differentiation algorithms; see \cite{grefenstette2019generalized} for more information and open-source code. 
Of course, for utilizing Eq.~\eqref{eq:many:steps:derivative} the outer-objective $\pazocal{L}_O$ and the inner optimization steps must be differentiable.
For instance, a single SGD step given as $\theta \leftarrow \theta - \epsilon \nabla_{\theta}\pazocal{L}(\theta,\eta)$, where $\epsilon$ is the learning rate, is differentiable with respect to $\eta$ if $\pazocal{L}$ is also differentiable with respect to $\eta$, meaning that the Jacobian in Eq.~\eqref{eq:many:steps:derivative} can be computed. 
The same holds for multiple SGD steps as well as for other optimization algorithms, such as AdaGrad and Adam.

Algorithm~\ref{algo:general} is a general gradient-based meta-learning algorithm for arbitrary, admissible $\eta$. 
The input to the meta-learning algorithm includes the number of outer and inner iterations, $I$ and $J$, respectively, and the outer and inner learning rates, $\epsilon_1$, $\epsilon_2$, respectively; see Appendix~\ref{app:meta:design:other} for stopping criteria.
Furthermore, the algorithm input includes the task distribution $p(\lambda)$ to be used for resampling during training and the number of tasks $T$.
As shown in Algorithm~\ref{algo:general}, the task set $\Lambda$ is not required to remain the same during optimization.
Optimizing $\eta$ using a meta-learning algorithm such as Algorithm~\ref{algo:general} is typically called \textit{meta-training}, and using the obtained $\eta$ for solving unseen tasks from the task distribution is called \textit{meta-testing}. 

\begin{algorithm}[H]
	\label{algo:general}
	\SetAlgoLined
	\textbf{input:} $\epsilon_1$, $\epsilon_2$, $I$, $T$, $J$, and task distribution $p(\lambda)$
	
	initialize $\eta$ with $\eta^{(0)}$ 
	
	\For{$i \in \{1,\dots, I\}$}{
		sample set $\Lambda$ of $T$ tasks from $p(\lambda)$
		
		\For{$\tau \in \{1,\dots, T\}$}{
			initialize $\theta_{\tau}$ with $\theta_{\tau}^{(0)}$
			
			\For{$j \in \{1,\dots, J\}$}{
				\Comment*[f]{Inner step for each task}
				
				$\theta_{\tau}^{(j)} = \theta_{\tau}^{(j-1)} - \epsilon_2 \nabla_{\theta}\pazocal{L}_{\tau}(\theta,\eta)\bigg\rvert_{\theta = \theta_{\tau}^{(j-1)}, \ \eta = \eta^{(i-1)}}$	
				\Comment*[f]{PINN step}
				
				\Comment*[f]{$\pazocal{L}_{\tau}$ given by Eq.~\eqref{eq:lossml:inner:loss}}	
			}
			set $\theta^{*(i)}_{\tau} = \theta_{\tau}^{(J)}$ 
		}	
		$\eta^{(i)} = \eta^{(i-1)} - \epsilon_1 \boldsymbol{d}_{\eta}\pazocal{L}_O(\theta^*(\eta),\eta)\bigg\rvert_{\theta^* = \{\theta^{*(i)}_{\tau}\}_{\tau = 1}^{T}, \ \eta = \eta^{(i-1)}}$
		\Comment*[f]{Outer step}	
		
		\Comment*[f]{$\boldsymbol{d}_{\eta}\pazocal{L}_O$ given by Eqs.~\eqref{eq:many:steps:derivative} and \eqref{eq:lossml:outer:loss}}
	}
	
	\textbf{return} $\eta^{(I)}$
	\caption{General gradient-based meta-learning algorithm for PINNs}
\end{algorithm}

\subsection{An algorithm for meta-learning PINN loss functions}\label{sec:meta:lossml}

Meta-learning PINN loss functions by utilizing the concepts of Section~\ref{sec:meta:bilevel} requires defining an admissible hyperparameter $\eta$ that can be used in conjunction with Algorithm~\ref{algo:general}.
In this regard, a parametrization $\eta$ can be used for the loss function $\ell$ of Eq.~\eqref{eq:pinns:loss:2} for which $\pazocal{L}_O$ and the inner optimization steps are differentiable.
Indicatively, $\ell$ can be represented with a feed-forward NN (FFN) and thus $\eta$ represents the weights and biases of the FFN.
Such a parametrization in conjunction with Algorithm~\ref{algo:general} has been proposed in \cite{bechtle2020metalearning} for supervised learning and reinforcement learning problems.
Alternatively, $\ell$ can be the adaptive loss function proposed in \cite{barron2019general} (see Eqs.~\eqref{eq:barron:loss:1} and \eqref{eq:barron:loss:2}) and thus $\eta$ can be the robustness and scale parameters. 
More discussion and options regarding loss function parametrization can be found in Section~\ref{sec:meta:design}.

Following parametrization, the PINN objective function for each task $\tau$ is the same as in standard PINNs training and given as
\begin{equation}\label{eq:lossml:inner:loss}
	\pazocal{L}_{\tau}(\theta, \eta) =  \pazocal{L}_f(\theta, \lambda_{\tau}) +
	\pazocal{L}_b(\theta, \lambda_{\tau}) 
	+ \pazocal{L}_{u_0}(\theta, \lambda_{\tau}),
\end{equation}
with the terms $\{\pazocal{L}_f, \pazocal{L}_b, \pazocal{L}_{u_0}\}$ given from Eq.~\eqref{eq:pinns:loss:2} and evaluated on the training datasets of sizes $\{N_{f}, N_{b}, N_{u_0}\}$ by using the parametrized loss function $\ell_{\eta}$ instead of a fixed $\ell$.
The objective function $\pazocal{L}_{\tau}(\theta, \eta)$ of Eq.~\eqref{eq:lossml:inner:loss} is optimized with respect to $\theta$ in the inner optimization step of Algorithm~\ref{algo:general}, while also tracking the optimization path dependence on $\eta$ (see \cite{grefenstette2019generalized}).
Next, the outer-objective $\pazocal{L}_O$, which is used for optimizing the learned loss $\ell_{\eta}$, can be defined as the MSE on validation data, i.e.,  
\begin{equation}\label{eq:lossml:outer:loss}
	\pazocal{L}_O(\theta^*(\eta), \eta) = \mathbb{E}_{\lambda_{\tau} \in \Lambda}\left[
	\pazocal{L}_f(\theta^*_{\tau}, \lambda_{\tau}) +
	\pazocal{L}_b(\theta^*_{\tau}, \lambda_{\tau}) 
	+ \pazocal{L}_{u_0}(\theta^*_{\tau}, \lambda_{\tau})
	\right],
\end{equation}
with the terms $\{\pazocal{L}_f, \pazocal{L}_b, \pazocal{L}_{u_0}\}$ given from Eq.~\eqref{eq:pinns:loss:2} and evaluated on the validation datasets of sizes $\{N_{f, val}, N_{b, val}, N_{u_0, val}, N_{u, val}\}$ by using the task parameters $\lambda_{\tau}$ and the optimal task-specific parameters $\theta^*_{\tau}$ for every $\tau$. 
Clearly, Eq.~\eqref{eq:lossml:outer:loss} has the same form as Eq.~\eqref{eq:lossml:inner:loss}, except for the fact that (a) the number of validation points $\{N_{f, val}, N_{b, val}, N_{u_0, val}, N_{u, val}\}$ can be different from $\{N_{f}, N_{b}, N_{u_0}, N_{u}\}$, (b) MSE is used in Eq.~\eqref{eq:lossml:outer:loss} as a loss function, whereas $\ell_{\eta}$ is used in Eq.~\eqref{eq:lossml:inner:loss}, and (c) $\pazocal{L}_O$ is an average loss across tasks, whereas $\pazocal{L}_{\tau}$ is task-specific. 
Optimizing $\ell_{\eta}$ utilizing Algorithm~\ref{algo:general} with $\pazocal{L}_O$ defined based on Eq.~\eqref{eq:lossml:outer:loss} aims to answer the following question: \textit{Is there a loss function $\ell_{\eta}$ which if used for PINN training (Eq.~\eqref{eq:lossml:inner:loss}) will perform better in terms of average across tasks MSE error (Eq.~\eqref{eq:lossml:outer:loss})?}

Note that in Eq.~\eqref{eq:lossml:outer:loss} additional data not used for inner training can also be utilized. 
For example, we may have solution data $u_{\tau}$ corresponding to $\lambda_{\tau}$ values sampled from the task distribution $p(\lambda)$ (e.g., produced by traditional numerical solvers or measurements), which can be used in the outer optimization step. 
By utilizing such additional data, the aforementioned question that loss function meta-learning aims to answer is augmented with the following: \textit{Is there a loss function $\ell_{\eta}$ that, in addition, leads to better solution $u$ performance error (Eq.~\eqref{eq:lossml:outer:loss}), although training has been performed based only on PDE residual and BCs/ICs data (Eq.~\eqref{eq:lossml:inner:loss})?} 
Finally, see Section~\ref{sec:meta:design:properties} for more information regarding imposing additional properties to the loss function $\ell_{\eta}$ through penalties in Eq.~\eqref{eq:lossml:outer:loss}.  

\subsection{Algorithm design and theory}\label{sec:meta:design}

\subsubsection{Loss function parametrization and initialization}
\label{sec:meta:design:param}

\paragraph{Adaptive loss function.}\label{sec:meta:design:param:LAL}
A loss function parametrization that can be used in conjunction with Algorithm~\ref{algo:general} has been proposed in \cite{barron2019general}.
In this regard, the one-dimensional loss function is parametrized by the shape parameter $\alpha \in \mathbb{R}$, which controls robustness to outliers (see Section~\ref{sec:prelim:loss}) and the scale parameter $c>0$ that controls the size of the quadratic bowl near zero.
Specifically, the loss function is expressed as 
\begin{equation}\label{eq:barron:loss:1}
	\rho_{\alpha,c}(d) = \frac{|\alpha - 2|}{\alpha}\left(\left(\frac{(d/c)^2}{|\alpha - 2|} + 1\right)^{\alpha/2} - 1\right), 
\end{equation}
where $d$ denotes the discrepancy between each dimension of the prediction and the target; e.g., $d=\hat{u}_{\theta, j}(t,x)-u(t,x)$ or $d = \hat{f}_{\theta, \lambda, j}(t,x)$ for each $j\in\{1,\dots,D_u\}$ in PINNs.
For fixed values of $\alpha$, Eq.~\eqref{eq:barron:loss:1} yields known losses; see \cite{barron2019general} and Table~\ref{tab:losses}.
An extension to multi-dimensional inputs can be achieved via Eqs~\eqref{eq:g:sum:loss}-\eqref{eq:g:sum:loss:2}. 

Nevertheless, the loss function of Eq.~\eqref{eq:barron:loss:1} cannot be used directly as an adaptive loss function to be optimized in the online manner (i.e., simultaneously with NN parameters) proposed in \cite{barron2019general}. 
Specifically, $\rho_{\alpha,c}(d)$ in Eq.~\eqref{eq:barron:loss:1} is monotonic with respect to $\alpha$, and thus attempting to optimize $\alpha$ by minimizing Eq.~\eqref{eq:barron:loss:1} trivially sets $\alpha$ to be as small as possible.
To address this issue, \cite{barron2019general} defined also the corresponding probability density function, i.e., 
\begin{equation}\label{eq:barron:pdf}
	p_{\alpha,c}(d) = \frac{1}{c Z(\alpha)}\exp(-\rho_{\alpha,c}(d)),
\end{equation}
which is valid only for $\alpha \geq 0 $ as $Z(\alpha)$ is divergent for $\alpha <0$.
Furthermore, \cite{barron2019general} defined a loss function based on the negative log likelihood of $p_{\alpha,c}(d)$, i.e., 
\begin{equation}\label{eq:barron:loss:2}
	\hat{\ell}_{\alpha,c}(d) = \log(c) + \log(Z(\alpha)) + \rho_{\alpha,c}(d),
\end{equation}
which is simply a shifted version of Eq.~\eqref{eq:barron:loss:1}.
Because the partition function $Z(\alpha)$ is difficult to evaluate and differentiate, $\log(Z(\alpha))$ is approximated with a cubic Hermite spline, which induces an added computational cost.

The loss function of Eq.~\eqref{eq:barron:loss:2} has been used in \cite{barron2019general} in conjunction with Eqs.~\eqref{eq:g:sum:loss}-\eqref{eq:g:sum:loss:2} as an adaptive loss function that is optimized online.
Specifically, either the same pair $(\alpha, c)$ is used for each dimension, i.e., Eq.~\eqref{eq:g:sum:loss} is employed with unit weights and $\eta = \{\alpha, c\}$ or a different pair is used, i.e., Eq.~\eqref{eq:g:sum:loss} with $\eta = \{\alpha_1, c_1,\dots, \alpha_{D_u}, c_{D_u}\}$ is employed.
Regarding implementation of the constraints $\alpha \geq 0$ and $c>0$, the parameters $\alpha$ and $c$ for every dimension (if applicable) can be expressed as
\begin{equation}\label{eq:barron:imple}
	\begin{aligned}
		\alpha &= (\alpha_{max}-\alpha_{min})\ sigmoid(\hat{\alpha}) + \alpha_{min} \\
		c &= softplus(\hat{c}) + c_{min} 
	\end{aligned}
\end{equation}
and $\hat{\alpha}$, $\hat{c}$ are optimized simultaneously with the NN parameters. 
In Eq.~\eqref{eq:barron:imple}, the sigmoid function limits $\alpha$ in the range $[\alpha_{min}, \alpha_{max}]$, whereas the softplus function constrains $c$ to being greater than $c_{min}$; e.g., $c_{min} = 10^{-8}$ for avoiding degenerate optima.

Note that Algorithm~\ref{algo:general} uses the outer objective $\pazocal{L}_O$ of Eq.~\eqref{eq:lossml:outer:loss} for optimizing the loss function, which is different from the inner objectives $\pazocal{L}_{\tau}$ of Eq.~\eqref{eq:lossml:inner:loss}.
As a result, using Eq.~\eqref{eq:barron:loss:1} as a loss function parametrization in Algorithm~\ref{algo:general} does not lead to trivial $\alpha$ solutions as in \cite{barron2019general}. 
Thus, either Eq.~\eqref{eq:barron:loss:1} or Eq.~\eqref{eq:barron:loss:2} can be considered as loss function parametrizations for Algorithm~\ref{algo:general}.

In Section~\ref{sec:meta:design:properties}, we prove that the loss function of Eq.~\eqref{eq:barron:loss:1} automatically satisfies the specific conditions required for successful training according to our new theorems, without adding any regularization terms in meta-training.  
In the present paper, the loss function parametrization of this section is referred to as \textit{LAL} when used as a meta-learning parametrization, and as \textit{OAL} when used as an online adaptive loss. 
For initialization, we consider the values $\alpha = 2.01$ and $c = 1/\sqrt{2}$, which approximate the squared $\ell_2$-norm.

\paragraph{NN-parametrized loss function.}\label{sec:meta:design:param:FFN}
An alternative parametrization based on FFN has been proposed in \cite{bechtle2020metalearning}.
Specifically, for the one-dimensional supervised learning regression problem considered in \cite{bechtle2020metalearning}, the most expressive representation $\ell_{\eta} = \hat{\ell}_{\eta}(\hat{u}_{\theta}, u)$ of Fig.~\ref{fig:gNN:higher} has been utilized.
In terms of parametrizing $\hat{\ell}_{\eta}$, $2$ hidden layers with $40$ neurons each without biases and with ReLU activations functions have been used, while the output is also passed through a softplus activation function for producing the final loss output.
Finally, the NN parameters are initialized using the Xavier uniform initializer. 

In this regard, we  note that ensuring positivity of the loss function does not affect NN parameter optimization; i.e., the softplus output activation function affects the results of \cite{bechtle2020metalearning} only through its nonlinearity and not by dictating positive loss outputs. 
Furthermore, instead of randomly initializing NN parameters $\eta$, one can alternatively initialize them so that $\ell_{\eta}$ approximates a known loss function such as the squared $\ell_2$-norm.
For obtaining such an initialization, it suffices to perform even a few Adam iterations with synthetic data obtained by computing the $\ell_2$-norm of randomly sampled values in the considered domain; see computational examples in Section~\ref{sec:examples} for more information.  

\paragraph{Meta-learning the composite objective function weights.}\label{sec:meta:design:param:weights}

Finally, the composite objective function weights $\{w_f, w_b, w_{u_0}\}$, corresponding to the PDE residual, BCs and ICs loss terms, respectively, in Eqs.~\eqref{eq:pinns:loss:2}-\eqref{eq:pinns:loss:3}, can also be included in the meta-learned parameters $\eta$. 
As a result, three different loss functions are learned that are equivalent up to a scaling factor; see computational examples in Section~\ref{sec:examples} for experiments. 
For restricting $\{w_f, w_b, w_{u_0}\}$ to values greater than zero or a minimum value, the softplus activation function can be used similarly to the $c$ parameter in Eq.~\eqref{eq:barron:imple}.

\subsubsection{Inner optimization steps}\label{sec:meta:design:inner}

As discussed in Sections~\ref{sec:meta:bilevel}-\ref{sec:meta:lossml}, the gradient $\nabla_{\eta}\pazocal{L}_O$ required to update $\eta$ in Algorithm~\ref{algo:general} is obtained through inner optimization path differentiation, i.e., via Eq.~\eqref{eq:many:steps:derivative}.
For each outer iteration $i$, the obtained NN parameters $\theta^{*(i)}_{\tau}$ for each task $\tau$ following $J$ inner SGD steps with hyperparameters $\eta^{(i)}$ are given as
\begin{equation}\label{eq:theta:many:steps}
	\theta^{*(i)}_{\tau} = \theta_{\tau}^{(J)} = \theta_{\tau}^{(0)} - \epsilon_2 \sum_{j=1}^{J}\nabla_{\theta}\pazocal{L}_{\tau}(\theta,\eta)\bigg\rvert_{\theta = \theta_{\tau}^{(j-1)}, \ \eta = \eta^{(i-1)}},
\end{equation} 
where $\pazocal{L}_{\tau}$ is given by Eq.~\eqref{eq:lossml:inner:loss}.
As a result, the Jacobian $J_{\theta^*(\eta), \eta}$ in Eq.~\eqref{eq:many:steps:derivative}, which is the average over tasks of the derivative of Eq.~\eqref{eq:theta:many:steps} with respect to $\eta$, is given as a sum of $J$ gradient terms, i.e.,  
\begin{equation}\label{eq:jacobian:theta:star}
	J_{\theta^*(\eta), \eta} = \mathbb{E}_{\lambda_{\tau} \in \Lambda}\left[- \epsilon_2 \sum_{j=1}^{J}\nabla_{\eta}\left(\nabla_{\theta}\pazocal{L}_{\tau}(\theta,\eta)\bigg\rvert_{\theta = \theta_{\tau}^{(j-1)}, \ \eta = \eta^{(i-1)}}\right)\T\right].
\end{equation}

Clearly, using a large number of inner steps $J$ enables $\eta$ optimization to take into account larger parts of the $\theta$ optimization history, which can be construed as enriching the meta-learning dataset.  
In this regard, the increased computational cost associated with large $J$ values can be addressed by considering approximations of Eq.~\eqref{eq:jacobian:theta:star}; see, e.g., \cite{nichol2018firstorder} for an introduction.
However, if the $\eta$ gradients for subsequent $j$ values point to similar directions, the summation of Eq.~\eqref{eq:jacobian:theta:star} can lead to large Jacobian values.
In turn, large Jacobian values in Eq.~\eqref{eq:jacobian:theta:star} lead to large $\boldsymbol{d}_{\eta}\pazocal{L}_O(\theta^*(\eta), \eta)$ values in Eq.~\eqref{eq:many:steps:derivative}, which make the optimization of $\eta$ unstable.
For the computational example of Section~\ref{sec:examples:regression} we demonstrate in Appendix~\ref{app:add:results} the effect of the number of inner steps as well as the exploding gradients pathology with an experiment.

Finally, to address this exploding $\eta$ gradients issue, we have tested (a) dividing $\boldsymbol{d}_{\eta}\pazocal{L}_O(\theta^*(\eta), \eta)$ values by the number of inner iterations $J$, (b) normalizing the gradient by its norm, and (c) performing gradient clipping, i.e., setting a cap for the gradient norm.
The fact that the norm of the $\eta$ gradient does not explode in every outer iteration makes dividing by $J$ too strict, while normalizing by the gradient norm deprives the $\eta$ gradient of its capability to provide also an update magnitude apart from a direction. 
For these reasons, gradient clipping is expected to be a better option.
This is corroborated by the experiments of Section~\ref{sec:examples}.

\subsubsection{Theoretical derivation of desirable loss function properties}\label{sec:meta:design:properties}

In general, meta-learning $\ell_{\eta}$ via Algorithm~\ref{algo:general} in conjunction with Eqs.~\eqref{eq:lossml:inner:loss}-\eqref{eq:lossml:outer:loss} aims to maximize meta-testing performance by considering an expressive $\ell_{\eta}$, which is learned during meta-training.
However, because the loss function plays a central role in optimization as explained in Section~\ref{sec:prelim:loss}, it would be important for the learned loss $\ell_{\eta}$ to satisfy certain conditions to allow efficient gradient-based training.
In this section, we theoretically identify \textit{\nc}\ and \textit{the \msec}\ as the two desirable conditions that enable efficient training in our problems. Moreover, we propose a novel regularization method to impose the conditions. The LAL parametrization of Section~\ref{sec:meta:design:param:LAL} is then proven to satisfy the two conditions without any regularization.

The results presented in this section are general and pertain to regression problems as well.
For this reason, we consider the more general than a PINN scenario of having a training dataset $\{(x_i,u_i)\}_{i=1}^N$ of $N$ samples, where the pair $(x_i, u_i)$ with $x_i \in \pazocal{X} \subseteq  \RR^{D_x}$ and $u_i \in\pazocal{U}\subseteq  \RR^{D_u}$, for $i \in \{1,\dots,N\}$, corresponds to the $i$-th (input, target) pair.
For learning a NN approximator $\hat{u}_{\theta}$, we minimize the objective 
\begin{align}
	\pazocal{L}(\theta)= \frac{1}{N} \sum_{i=1}^N \ell(\hat{u}_{\theta}(x_{i}),u_{i}), 
\end{align} 
over $\theta \in \RR^{D_{\theta}}$, where $\ell: \RR^{D_{u}} \times\pazocal{U} \rightarrow \RR_{\ge 0}$ is the selected or learned loss function.
Note that in the following, the function $\hat{u}_{\theta}$ is allowed to represent a wide range of network architectures, including ones with batch normalization, convolutions, and skip connections. In this section, we assume that the map $\hat{u}_{i}:\theta \mapsto \hat{u}_{\theta}(x_i)$ is differentiable for every $i \in \{1,\dots, N\}$.

We define the output vector for all $N$ data points by
\begin{equation}
	\hat{u}_{X}(\theta)= \vect((\hat{u}_{\theta}(x_{1}),\dots,\hat{u}_{\theta}(x_{N}))) \in \RR^{N D_u},
\end{equation}
and let $\{\theta^{(r)}\}_{r=0}^\infty$ be the optimization sequence defined by
\begin{equation}
	\theta^{(r+1)} = \theta^{(r)}-\epsilon^{(r)}\bg^{(r)},
\end{equation}
with an initial parameter vector $\theta^{(0)}$, a learning rate $\epsilon^{(r)}$, and an update vector $\bg^{(r)}$. One of the theorems in this section relies on the following assumption on the update vector $\bg^{(r)}$: 
\begin{assumption} \label{assump:5}
	There exist $\bc,\uc>0$ such that 
	$
	\uc \|\nabla_{\theta}\pazocal{L}(\theta_{}^{(r)})\|^2  \allowbreak \le\nabla_{\theta}\pazocal{L}(\theta_{}^{(r)})\T \bg^{(r)} 
	$
	and
	$
	\|\bg^{(r)}\|^2 \le \bc \|\nabla_{\theta}\pazocal{L}(\theta_{}^{(r)})\|^2
	$ for any $r \ge 0$.
\end{assumption}
\noindent
It is noted that Assumption \ref{assump:5} is satisfied by using  
$
\bg^{(r)} =  D^{(r)}\nabla_{\theta}\pazocal{L}(\theta^{(r)}), 
$
where $D^{(r)}$ is any positive definite symmetric matrix with eigenvalues in the interval  $[\uc, \sqrt{\bc}]$. 
Setting $D^{(r)}=I$ corresponds to SGD and Assumption \ref{assump:5} is satisfied with $\uc=\bc = 1$. 
Next, we define \textit{\nc} as well as \textit{the \msec} and provide the main Theorems~\ref{thm:1}-\ref{thm:2} and Corollary \ref{corollary:1} of this section; corresponding proofs can be found in Appendix~\ref{app:proofs}.
\begin{definition}\label{def:1} 
	The learned loss $\ell$ is said to satisfy \textit{\nc} if the following holds: for all $u \in \pazocal{U}$ and $q \in \RR^{D_u}$, $\nabla_{q} \ell(q,u)$ exists and  $\nabla_{q} \ell(q,u)=0$ implies that $\ell(q,u) \le \ell(q',u)$ for all $q' \in \RR^{D_u}$. 
\end{definition}    

\begin{theorem} \label{thm:1}
	If  the learned loss $\ell$ satisfies \nc,  then any stationary point $\theta$ of $\pazocal{L}$  is a global minimum of $\pazocal{L}$ when $\rank(\frac{\partial \hat{u}_X (\theta)}{\partial \theta})=ND_u$. If   the learned loss $\ell$ does not satisfy \nc\ by having a  point $q \in \RR^{D_u}$ such that $\nabla_{q} \ell(q,u)=0$ and   $\ell(q,u) > \ell(q',u)$ for some $q'\in \RR^{D_u}$, then there exists a stationary point $\theta$ of $\pazocal{L}$  that is not a global minimum of $\pazocal{L}$ when $\{\hat{u}_X (\theta)\in \RR^{N D_u}:\theta\in \RR^{D_{\theta}}\}=\RR^{N D_u}$.  
\end{theorem}

\begin{theorem} \label{thm:2}
	Suppose Assumption \ref{assump:5} holds. Assume that the learned loss $\ell$ satisfies \nc,    $\|\nabla_{\theta}\pazocal{L}(\theta)-\nabla_{\theta}\pazocal{L}(\theta')\|\le L \|\theta-\theta'\|$ for all $\theta,\theta'$ in the domain of $\pazocal{L}$ for some $L \ge 0$, and the learning rate sequence $\{\epsilon^{(r)}\}_{r}$ satisfies either (i) $\zeta \le \epsilon^{(r)} \le \frac{\uc (2-\zeta)}{L\bc}$ for some $\zeta>0$, 
	or (ii) $\lim_{r \rightarrow \infty}\epsilon^{(r)} =0$ and $\sum_{r=0}^\infty \epsilon^{(r)} = \infty$. 
	Then, for any limit point $\theta$ of the sequence $\{\theta^{(r)}\}_{r}$, the limit point $\theta$ is a global minimum of $\pazocal{L}$ if $\rank(\frac{\partial \hat{u}_X (\theta)}{\partial \theta})=ND_u$. 
\end{theorem}

\begin{definition}\label{def:2} 
	The learned loss $\ell$ is said to satisfy \textit{the \msec} if the following holds: for all $u \in \pazocal{U}$ and $q \in \RR^{D_u}$, $\nabla_{q} \ell(q,u)=0$ if and only if $q=u$. 
\end{definition} 

\begin{corollary}  \label{corollary:1}
	If the learned loss $\ell$ satisfies \nc\ and the \msec,  then any stationary point $\theta$ of $\pazocal{L}$  is a global minimum of the MSE loss $\pazocal{L}_{\mathrm{MSE}}$ when $\rank(\frac{\partial \hat{u}_X (\theta)}{\partial \theta})=ND_u$, where $\pazocal{L}_{\mathrm{MSE}}(\theta)=\frac{1}{N}\sum_{i=1}^N\|\hat u_\theta(x_{i})-u_{i}\|_2^2$. 
\end{corollary}  

For example, the rank condition $\rank(\frac{\partial \hat{u}_X (\theta)}{\partial \theta})=ND_u$ (as well as the expressivity condition of $\{\hat{u}_X (\theta)\in \RR^{N D_u}:\theta\in \RR^{D_\theta}\}=\RR^{N D_u}$) is guaranteed to be satisfied by using wide NNs (e.g., see \citealp{kawaguchi2019gradient,huang2020dynamics,kawaguchi2021recipe}).
Nevertheless, the rank condition $\rank(\frac{\partial \hat{u}_X (\theta)}{\partial \theta})=ND_u$ is more general than the condition of using wide NNs, in the sense that the latter implies the former but not vice versa. Moreover, the standard loss functions used in PINNs, such as squared $\ell_2$-norm, satisfy the differentiability condition of Definition~\ref{def:1}, and are convex with respect to $q$: i.e., $\ell_{u} :q\mapsto \ell(q,u)$ is convex for all $u \in \pazocal{U}$. It is known that for any differentiable convex function $\ell_{u}:\RR^{D_u} \rightarrow \RR$, we have $\ell_u(q')\ge\ell_u(q)+\nabla_q\ell_u(q)\T(q'-q)$ for all $q,q' \in \RR^{D_u}$. Because the latter implies \nc, we conclude that standard loss functions typically used in PINNs satisfy \nc.

Unlike standard loss functions, the flexibility provided by meta-learning the loss function allows the learned loss $\ell_{u} :q\mapsto \ell(q,u)$ to be non-convex in $q$. This flexibility allows the learned loss to be well tailored to the task distribution considered, while we can  still check or impose \nc; see also Section~\ref{sec:examples}.
Corollary \ref{corollary:1} shows that for our PINN problems with the MSE measure, it is desirable for the learned loss to satisfy \nc~and the \msec.
The \msec~imposes the existence  of the desired stationary point related to MSE, whereas \nc~ ensures the global optimality. Since \nc~does not guarantee the existence of a stationary point, it is possible that there is no stationary point without the \msec\ or this type of an additional condition. As a pathological example, we may have $\ell_{\eta}(q, u)=q-u$, for which there is no stationary point and Theorems \ref{thm:1}--\ref{thm:2} vacuously hold true.
Therefore, depending on the measures used in applications (i.e., MSE for our case), it is beneficial to impose this type of an additional condition along with \nc~in order to impose an existence of a desirable stationary point.

Using this theoretical result, we now propose a  novel penalty term $\pazocal{L}_{O, add}$   to be added to the outer objective $\pazocal{L}_O$ of Eq.~\eqref{eq:lossml:outer:loss:penalty}  in order to penalize the deviation of the learned loss from the conditions in Corollary \ref{corollary:1}.
More concretely, the novel penalty term $\pazocal{L}_{O, add}$ based on Corollary \ref{corollary:1} is expressed as
\begin{equation}\label{eq:lossml:outer:loss:penalty}
	\pazocal{L}_{O,add}(\eta) = \mathbb{E}_{q}[\norm{\nabla_{q} \ell_{\eta}(q, q)}_2^2] + \mathbb{E}_{q \neq q'}[\max(0, c-\norm{\nabla_{q} \ell_{\eta}(q, q')}_2^2)],
\end{equation}
where $q, q' \in \RR^{D_u}$ are arbitrary inputs to the loss function, and $c$ is a hyperparameter that can be set equal to a small value (e.g., $10^{-2}$).
The first term on the right-hand side of Eq.~\eqref{eq:lossml:outer:loss:penalty} promotes the first-order derivative of the learned loss to be zero for zero discrepancy, for obtaining a learned loss that  satisfies the \msec. 
Furthermore, in the second term on the right-hand side of Eq.~\eqref{eq:lossml:outer:loss:penalty},  the additional penalty term maximizes the derivative away from zero up to a constant $c$, for obtaining a learned loss that satisfies \nc. 
The terms $\mathbb{E}_{q}[\norm{\nabla_{q} \ell_{\eta}(q, q)}_2^2]$ and $\mathbb{E}_{q \neq q'}[\max(0, c-\norm{\nabla_{q} \ell_{\eta}(q, q')}_2^2)]$ can, in practice, be computed by drawing  some $q$ and a random pair $(q,q')$ such that $q \neq q'$, in each outer iteration, and by computing $\norm{\nabla_{q} \ell_{\eta}(q, q)}_2^2$ and $\max(0, c-\norm{\nabla_{q} \ell_{\eta}(q, q')}_2^2)$, respectively. 
Alternatively, we can define an empirical distribution on $q$ and on $(q,q')$, and replace the expectations by summations over finite points.  
By following the same rationale and by augmenting the outer objective $\pazocal{L}_O$ of Eq.~\eqref{eq:lossml:outer:loss:penalty}, other problem-specific constraints can be imposed to the loss function as well.

Finally, we prove that a learned loss with the LAL parametrization of Section~\ref{sec:meta:design:param:LAL} automatically satisfies \nc~and the \msec\ without adding the regularization term: 

\begin{proposition} \label{prop:1}
	Any  LAL loss of the form $\ell(q,u)=\rho_{\alpha,c}(q-u) = \frac{|\alpha - 2|}{\alpha}((\frac{((q-u)/c)^2}{|\alpha - 2|} + 1)^{\alpha/2} - 1)$ satisfies \nc~and the \msec\ if $c>0$, $\alpha\neq0$, and  $\alpha\neq 2$.     
\end{proposition}

\section{Computational examples}\label{sec:examples}

We consider four computational examples in order to demonstrate the applicability and the performance of Algorithm~\ref{algo:general} for meta-learning PINN loss functions.
In Section~\ref{sec:examples:regression}, we address the problem of discontinuous function approximation with varying frequencies. Because the function approximation problem is conceptually simpler than solving PDEs and computationally cheaper, Section~\ref{sec:examples:regression} serves not only as a pedagogical example but also as a guide for understanding the behavior of Algorithm~\ref{algo:general} when different loss function parametrizations and initializations, inner optimizers and other design options are used; see Appendices~\ref{app:meta:design:other} and \ref{app:add:results}.
In Section~\ref{sec:examples:ad} we address the problem of solving the advection equation with varying initial conditions and discontinuous solutions, in Section~\ref{sec:examples:rd} we solve a steady-state version of the reaction-diffusion equation with varying source term, and in Section~\ref{sec:examples:burgers} we solve the Burgers equation with varying viscosity in two regimes. 

For each computational example, both FFN and LAL parametrizations from Section~\ref{sec:meta:design:param} are studied. For the FFN parametrization, we perform meta-training with and without the theoretically-driven regularization terms of Eq.~\eqref{eq:lossml:outer:loss:penalty} as developed in Section~\ref{sec:meta:design:properties}. We present the regularization results explicitly only when the terms of Eq.~\eqref{eq:lossml:outer:loss:penalty} are not zero, otherwise the respective results are identical. For the LAL parametrization, the regularization is not required because the desirable conditions of Section~\ref{sec:meta:design:properties} are automatically satisfied according to Proposition~\ref{prop:1}.

In all the examples, meta-training is performed using Adam as the outer optimizer for $10{,}000$ iterations. During meta-training, 6 snapshots of the learned loss are captured, with 0 corresponding to initialization. 
Furthermore, only one task is used in each outer step throughout this section ($T = 1$ in Algorithm~\ref{algo:general}); increasing this number up to 5 does not provide any significant performance increase in the considered cases. In meta-testing, we compare the performance of the snapshots of the learned loss captured during meta-training with standard loss functions from Table~\ref{tab:losses}. 
Specifically, we compare $12$ learned losses ($6$ snapshots of the FFN and $6$ of the LAL parametrization; see Sections~\ref{sec:meta:design:param:LAL}-\ref{sec:meta:design:param:FFN}) with the squared $\ell_2$-norm (MSE), the absolute error (L1), the Cauchy, and the Geman-McClure (GMC) loss functions. 
In addition, we compare with the OAL of \cite{barron2019general} (see Section~\ref{sec:meta:design:param:LAL}) with $2$ learning rates ($0.01$ and $0.1$, denoted as OAL 1 and OAL 2) for its trainable parameters (robustness and scale); only the robustness parameter is trained in the computational examples because this setting yielded better performance. For evaluating the performance of the considered loss functions we use them for meta-testing on either $5$ or $10$ unseen tasks, either in-distribution (ID) or out-of-distribution (OOD), and record the relative $\ell_2$ test error (rl2) on exact solution datapoints averaged over tasks.

\subsection{Discontinuous function approximation with varying frequencies and heteroscedastic noise}\label{sec:examples:regression}

We first consider a distribution of functions in $[0, 4\pi ]$ defined as
\begin{equation}\label{eq:regex:udef}
	u(x)= 
	\begin{cases}
		\sin(\omega_1 x) + \epsilon,& \text{if } 0 \leq x \leq 2\pi\\
		k (1 + \sin(\omega_2 (x - 2\pi))),              & \text{if } 2\pi < x \leq 4\pi,
	\end{cases}
\end{equation}
where $k$ denotes the magnitude of discontinuity and $\epsilon$ represents a zero-mean Gaussian noise term defined only in $[0, 2\pi]$ with standard deviation $\sigma_{\epsilon}$.
In this regard, a task distribution $p(\lambda)$ can be defined by drawing randomly frequency values $\lambda = \{\omega_1, \omega_2\}$ from $\mathcal{U}_{[\omega_{1, min}, \omega_{1, max}]}$ and $\mathcal{U}_{[\omega_{2, min}, \omega_{2, max}]}$, respectively, where $\mathcal{U}$ denotes a uniform distribution.
The defined task distribution is used in this section for function approximation.
The values of the fixed parameters used in this example can be found in Table~\ref{tab:reg:params}.
\begin{table}[H]
	\centering
	\caption{Function approximation: Task distribution values pertaining to Eq.~\eqref{eq:regex:udef}, used in meta-training and OOD meta-testing. 
	}
	\begin{tabular}{ccccccc}
		\toprule
		&$k$     & $\omega_{1,min}$     & $\omega_{1,max}$ & $\omega_{2,min}$ & $\omega_{2,max}$ & $\sigma_{\epsilon}$\\
		\midrule
		meta-training&$1$ & $1$ & $3$ & $5$ & $6$ & $0.2$\\
		OOD meta-testing&$1$ & $0.5$ & $4$ & $6$ & $7$ & $0.2$\\
		\bottomrule
	\end{tabular}
	\label{tab:reg:params}
\end{table}

\paragraph*{Meta-training.}
Following the design options experiment in Appendix~\ref{app:add:results}, we fix the design options of Algorithm~\ref{algo:general} to $J=20$ inner steps and to resampling and re-initializing every outer iteration.
The approximator NN architecture consists of $3$ hidden layers with $40$ neurons each and $tanh$ activation function.
In the inner objective we consider $N_u = 100$ noisy datapoints with $\sigma_{\epsilon} = 0.2$, whereas in the outer objective $N_{u, val} = 1{,}000$ and $\sigma_{\epsilon} = 0$. 
This can be interpreted as leveraging in the offline phase clean historical data that we synthetically corrupt with noise in order to meta-learn a loss function that can work well at test time also for the case of noisy data.
Moreover, both parametrizations (FFN and LAL) are used for comparison, and Adam is used as both inner and outer optimizer, with learning rates $10^{-3}$ and $10^{-4}$, respectively. 
In Fig.~\ref{fig:reg:snaps} we show the learned loss snapshots and their first-order derivatives and compared with MSE for the FFN and LAL parametrizations. 
Both FFN and LAL parametrizations with MSE initialization yield highly different learned losses as compared to MSE.
Being more flexible than LAL, FFN leads to more complex learned losses as depicted especially in the first-order derivative plots (Figs.~\ref{fig:reg:snaps:ffn:lossder} and \ref{fig:reg:snaps:lal:lossder}).
\begin{figure}[H]
	\centering
	\begin{subfigure}[t]{0.48\textwidth}
		\centering
		\includegraphics[width=\textwidth]{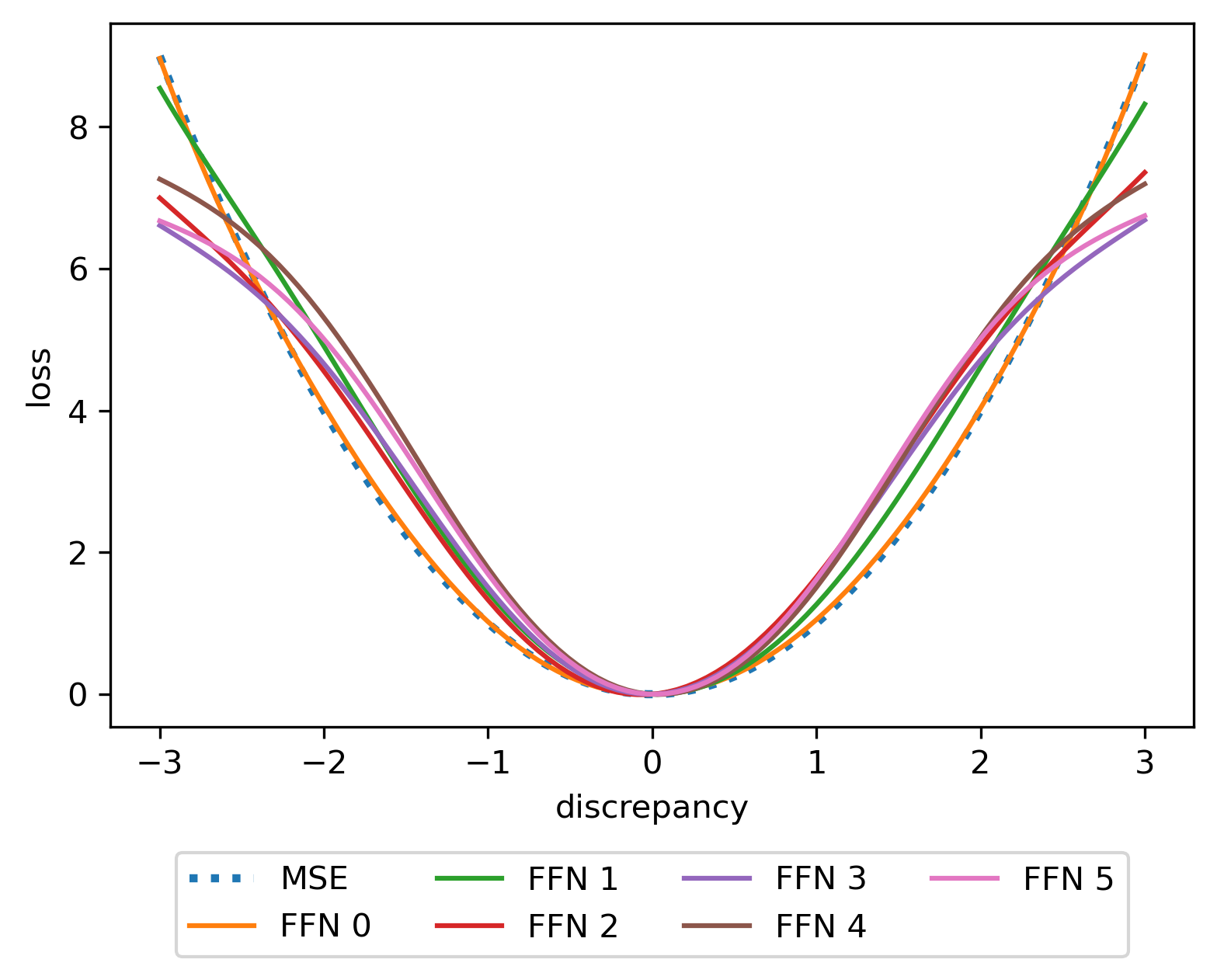}
		\caption{}
		\label{fig:reg:snaps:ffn:loss}
	\end{subfigure}
	\hfill
	\begin{subfigure}[t]{0.48\textwidth}
		\centering
		\includegraphics[width=\textwidth]{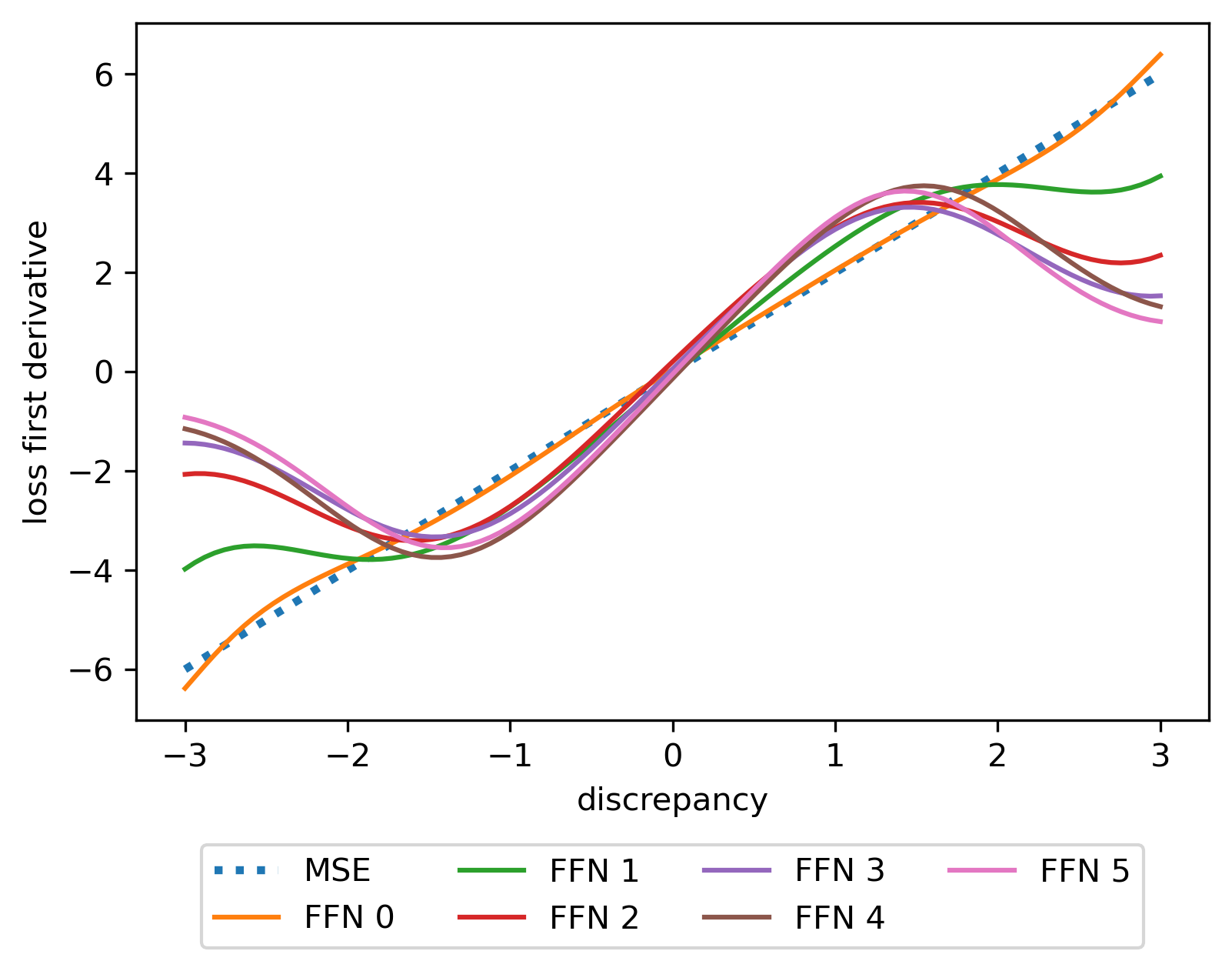}
		\caption{}
		\label{fig:reg:snaps:ffn:lossder}
	\end{subfigure}
	\hfill
	\begin{subfigure}[t]{0.48\textwidth}
		\centering
		\includegraphics[width=\textwidth]{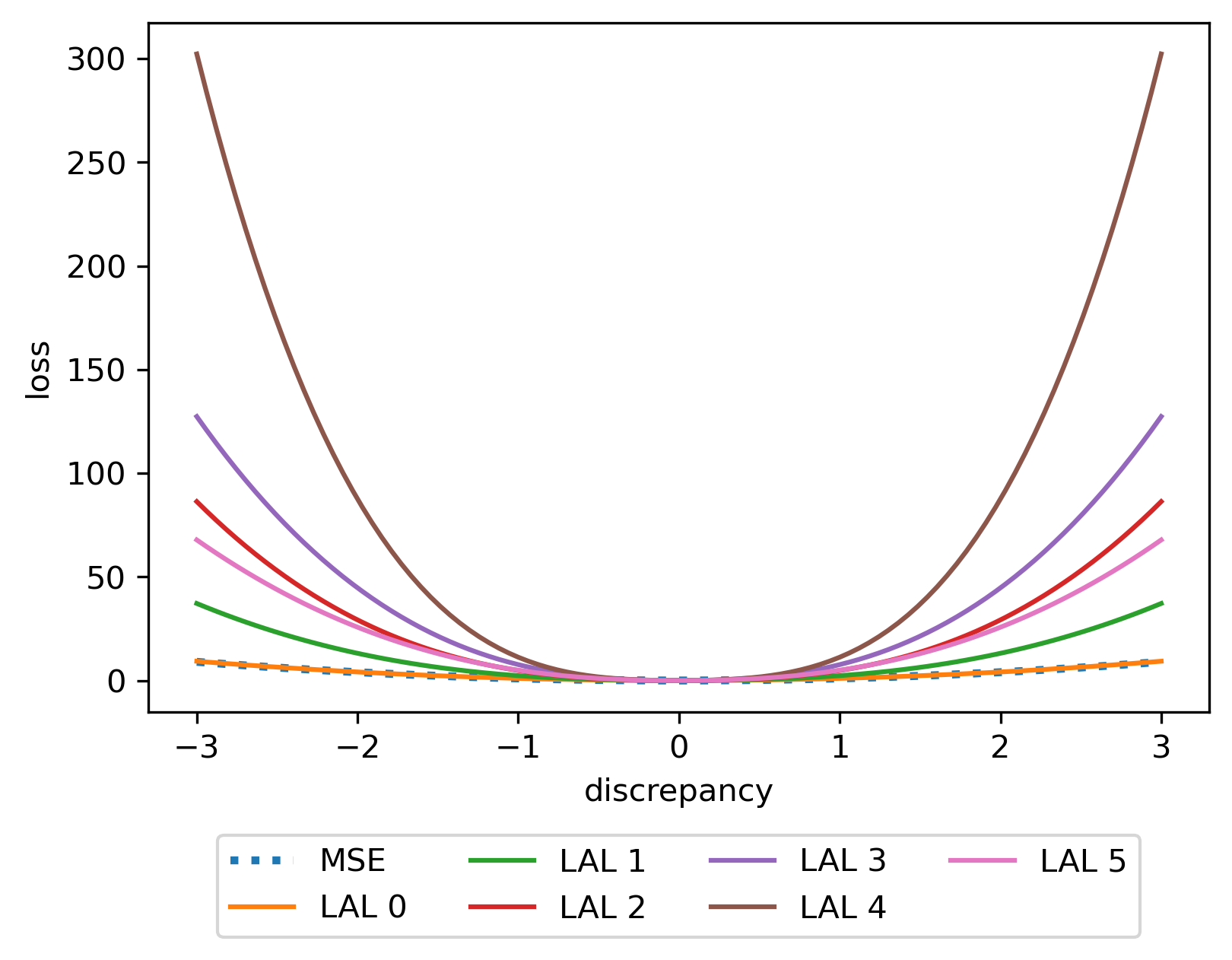}
		\caption{}
		\label{fig:reg:snaps:lal:loss}
	\end{subfigure}
	\hfill
	\begin{subfigure}[t]{0.48\textwidth}
		\centering
		\includegraphics[width=\textwidth]{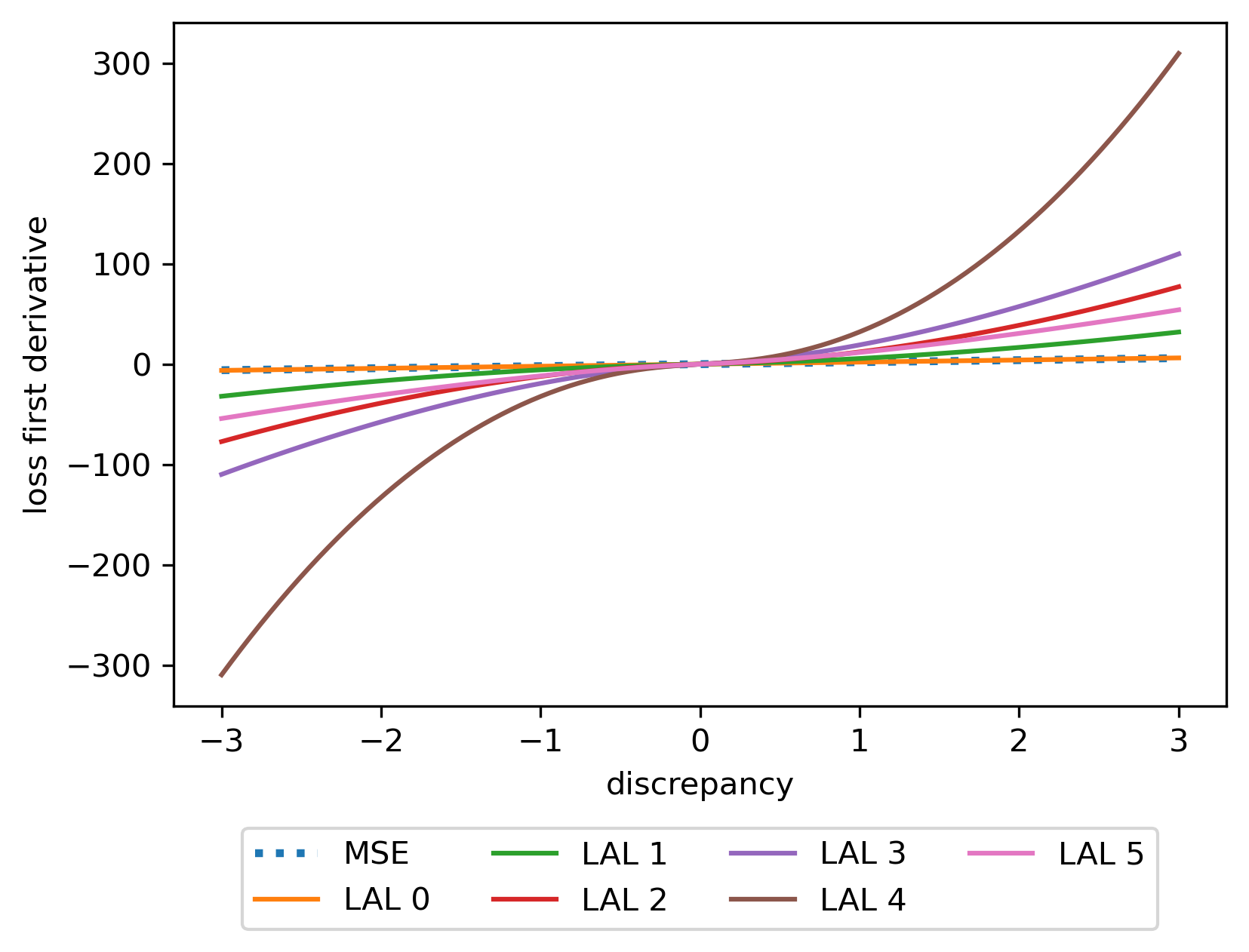}
		\caption{}
		\label{fig:reg:snaps:lal:lossder}
	\end{subfigure}
	\caption{Function approximation: Learned loss snapshots (a, c) and corresponding first-order derivatives (b, d), as captured during meta-training (distributed evenly in $10{,}000$ outer iterations with 0 referring to initialization).
		Results obtained with FFN (a, b) and LAL (c, d) parametrizations.
		Both FFN and LAL parametrizations with MSE initialization yield highly different learned losses as compared to MSE.
		Being more flexible than LAL, FFN leads to more complex learned losses as depicted especially in the first-order derivative plots (b, d).}
	\label{fig:reg:snaps}
\end{figure}

\paragraph*{Meta-testing.}
For evaluating the performance of the captured learned loss snapshots, we use them for meta-testing on $10$ OOD tasks and compare with standard loss functions from Table~\ref{tab:losses}.
Specifically, we train with Adam for $50{,}000$ iterations $10$ tasks using $18$ different loss functions ($6$ FFN, $6$ LAL and $6$ standard) and record the rl2 error on $1{,}000$ exact solution datapoints. 
The test tasks are sampled from a distribution defined by combining Eq.~\eqref{eq:regex:udef} with $k=1$ and $\sigma_{\epsilon} = 0.2$, and with the uniform distributions $\mathcal{U}_{[\omega_{1, min}, \omega_{1, max}]}$ and $\mathcal{U}_{[\omega_{2, min}, \omega_{2, max}]}$, where $\{\omega_{1, min}, \omega_{1, max}, \omega_{2, min}, \omega_{2, max}\}$ are shown in Table~\ref{tab:reg:params}. 
The minimum rl2 error results are shown in Fig.~\ref{fig:reg:test:min:stats}. 
We see that the loss functions learned with the FFN parametrization do not generalize well, whereas the ones learned with the LAL parametrization achieve an average minimum rl2 error that is smaller than the error corresponding to all the other considered loss functions by at least $15 \%$. 
For example, the average minimum rl2 error for LAL 4 is approximately $17 \%$ and is obtained (on average) close to iteration $20{,}000$, whereas the corresponding error for MSE is approximately $20 \%$ and is obtained (on average) close to iteration $50{,}000$.
Despite the improvement demonstrated in this example, in Fig.~\ref{fig:ad:traintest:Adam:experiment} we show for the advection equation example that the performance of the learned loss depends on the exponential decay parameters of Adam that control the dependence of the updates on the gradient history.  
\begin{figure}[H]
	\centering
	\includegraphics[width=.7\linewidth]{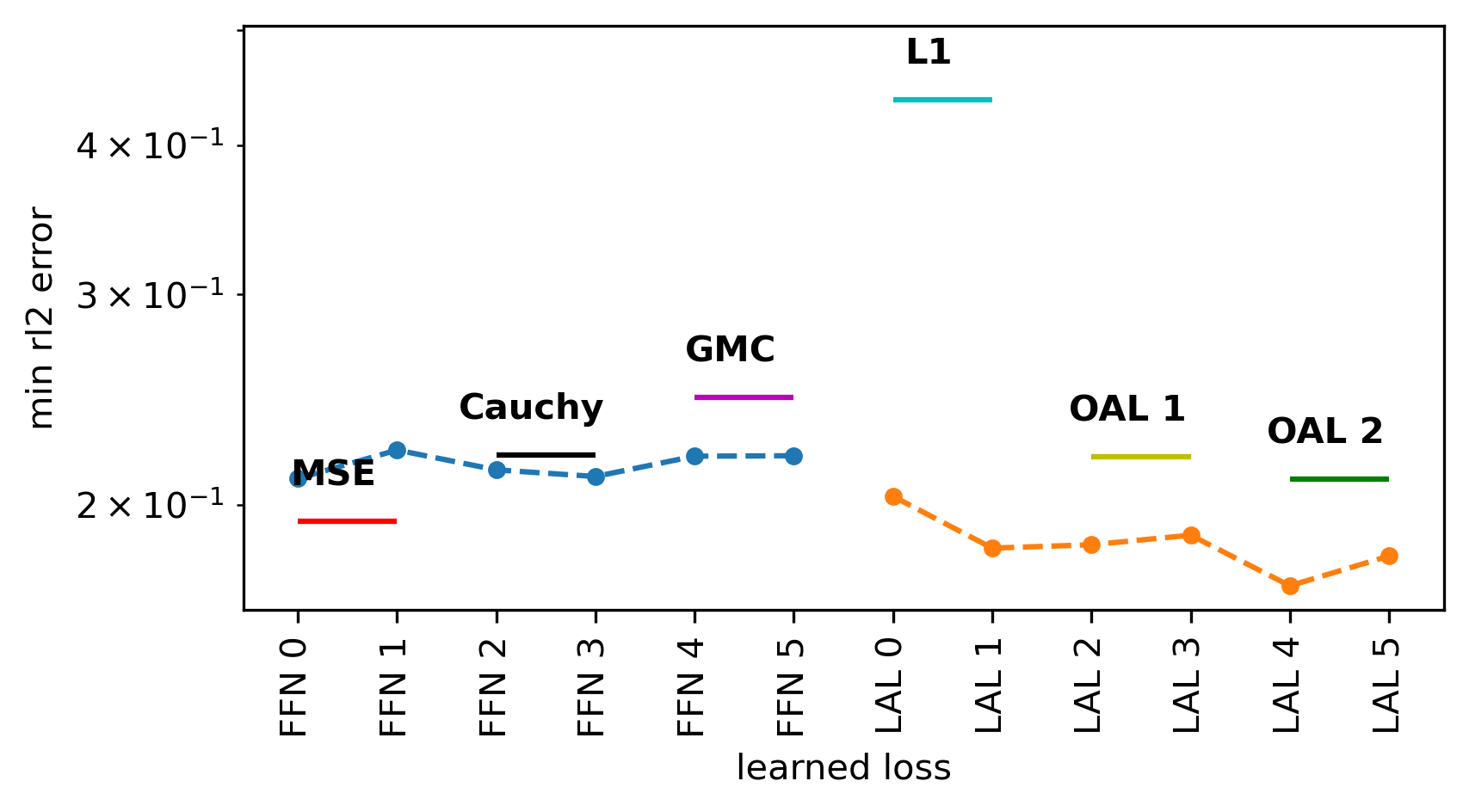}  
	\caption{Function approximation: Minimum relative test $\ell_2$ error (rl2) averaged over $10$ OOD tasks during meta-testing with Adam for $50{,}000$ iterations.
		Learned loss snapshots (FFN 0-6 and LAL 0-6) are compared with standard loss functions of Table~\ref{tab:losses} and with online adaptive loss functions OAL 1 and OAL 2 (2 loss-specific learning rates).
		The loss functions learned with the FFN parametrization do not generalize well, whereas the ones learned with the LAL parametrization achieve an average minimum rl2 error that is smaller than the error corresponding to all the other considered loss functions by at least $15 \%$.}
	\label{fig:reg:test:min:stats}
\end{figure}

\subsection{Task distributions defined based on advection equation with varying initial conditions and discontinuous solutions}\label{sec:examples:ad}

Next, we consider the $(1 + 1)$-dimensional advection equation given as
\begin{equation}\label{eq:advection}
	\partial_t u + V \partial_x u = 0,
\end{equation}
where $V$ is the constant advection velocity, $x \in [-1, 1]$ and $t \in [0, 1]$.
The considered Dirichlet BCs are given as 
\begin{equation}\label{eq:advection:bcs}
	u(-1, t) = u(1, t) = 0
\end{equation}
and the ICs as
\begin{equation}\label{eq:advection:ics}
	u(x, 0) = u_{0, \lambda}(x) =
	\begin{cases}
		\frac{1}{\lambda},& \text{if } -1 \leq x \leq -1 + \lambda\\
		0,              & \text{if } -1 + \lambda < x \leq 1,
	\end{cases}
\end{equation}
i.e., $u_{0, \lambda}(x)$ is a normalized box function of length $\lambda$. 
The exact solution for this problem is given as $u_{0, \lambda}(x - V t)$, which is also a box function that advects in time. 
In this regard, we can define a PDE task distribution comprised of problems of the form of Eq.~\eqref{eq:advection} with ICs given by Eq.~\eqref{eq:advection:ics} with varying $\lambda$. 
A task distribution $p(\lambda)$ can be defined by drawing randomly $\lambda$ values from $\mathcal{U}_{[\lambda_{min}, \lambda_{max}]}$, which correspond to different initial conditions $u_{0, \lambda}(x)$ in Eq.~\eqref{eq:advection:ics}.
In this example, $V = 1$, $\lambda_{min} = 0.5$, and the maximum value of $\lambda$ that can be considered so that the BCs are not violated is $1$; thus we consider $\lambda_{max} = 1$.

\paragraph*{Meta-training.}
During meta-training, Algorithm~\ref{algo:general} is employed with either SGD or Adam as inner optimizer with learning rate $10^{-2}$ or $10^{-3}$, respectively. 
Both FFN and LAL are initialized as MSE approximations, the number of inner iterations is $20$, and tasks are resampled and approximator NNs (PINNs in this case) are randomly re-initialized in every outer iteration.
The number of datapoints $\{N_{f}, N_{b}, N_{u_0}\}$ and $\{N_{f, val}, N_{b, val}, N_{u_0, val}\}$ used for evaluating both the inner objective of Eq.~\eqref{eq:lossml:inner:loss} and the outer objective of Eq.~\eqref{eq:lossml:outer:loss}, respectively, are the same and equal to $\{1{,}000, 100, 200\}$, and $N_{u, val} = 0$. 
Furthermore, the PINN architecture consists of $4$ hidden layers with $20$ neurons each and $tanh$ activation function.

We show in Fig.~\ref{fig:ad:comp}, the final loss functions (FFN and LAL parametrizations) as obtained with SGD as inner optimizer, with and without meta-learning the objective function weights (Section~\ref{sec:meta:design:param:weights}), and as obtained with Adam as inner optimizer.
For SGD as inner optimizer, both FFN and LAL parametrizations with MSE initialization yield highly different learned losses as compared to MSE, with FFN yielding more complex learned losses.
Furthermore, objective function weights meta-learning leads to an asymmetric final learned loss and is found in meta-testing to deteriorate performance.
For Adam as inner optimizer, the final learned losses are close to MSE.
The corresponding objective function weights trajectories are shown in Fig.~\ref{fig:ad:params}.
All objective function weights increase for both parametrizations, which translates into learning rate increase, while FFN and LAL disagree on how they balance ICs. 
The meta-testing results ($100$ test iterations) obtained while meta-training for the case of SGD as inner optimizer are shown in Fig.~\ref{fig:ad:traintest}.
Although initially objective function weights meta-learning improves performance, the corresponding final learned losses in conjunction with the final learned weights eventually deteriorate performance. 
\begin{figure}[H]
	\centering
	\begin{subfigure}[t]{0.48\textwidth}
		\centering
		\includegraphics[width=\textwidth]{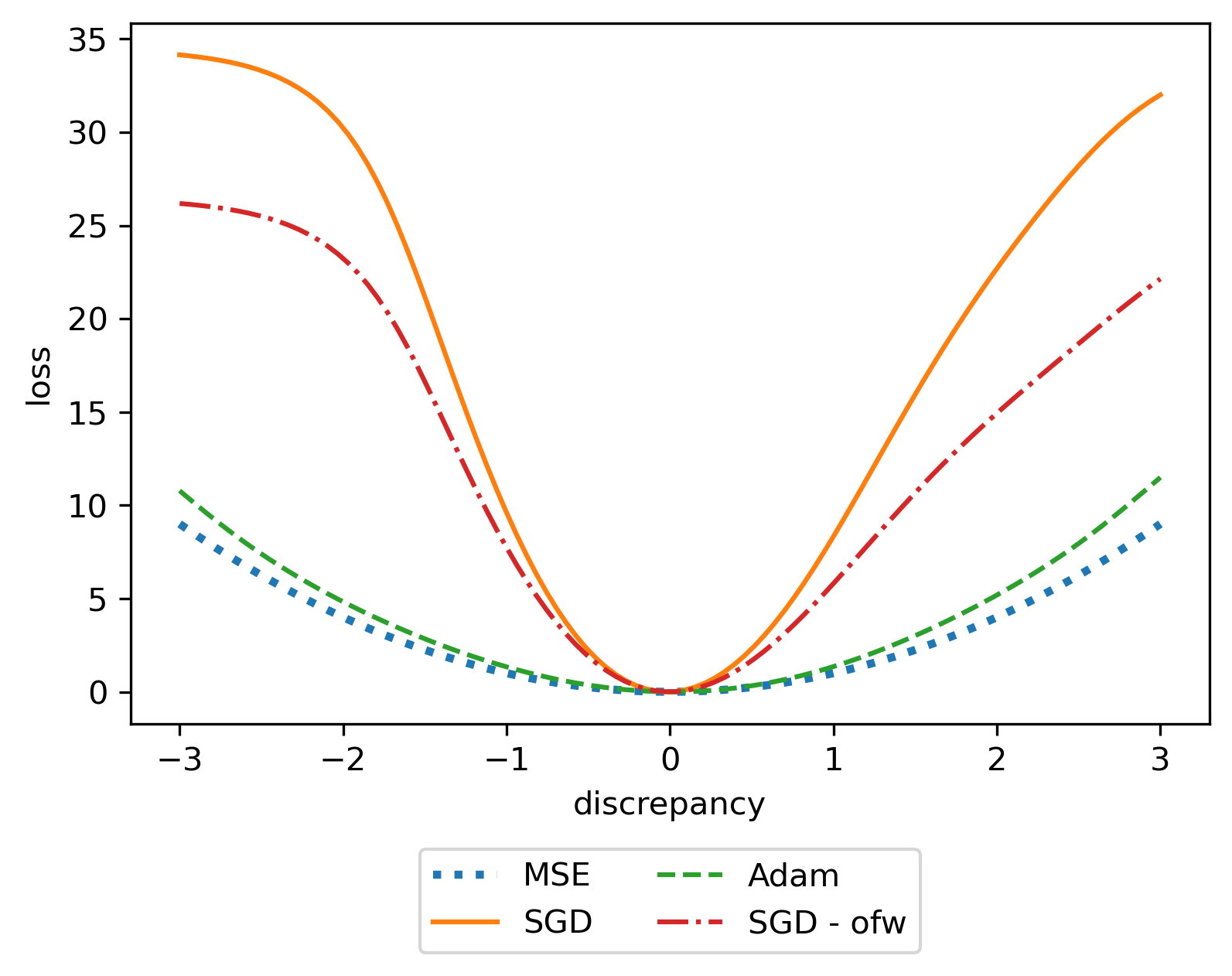}
		\caption{}
		\label{fig:ad:comp:ffn:loss}
	\end{subfigure}
	\hfill
	\begin{subfigure}[t]{0.48\textwidth}
		\centering
		\includegraphics[width=\textwidth]{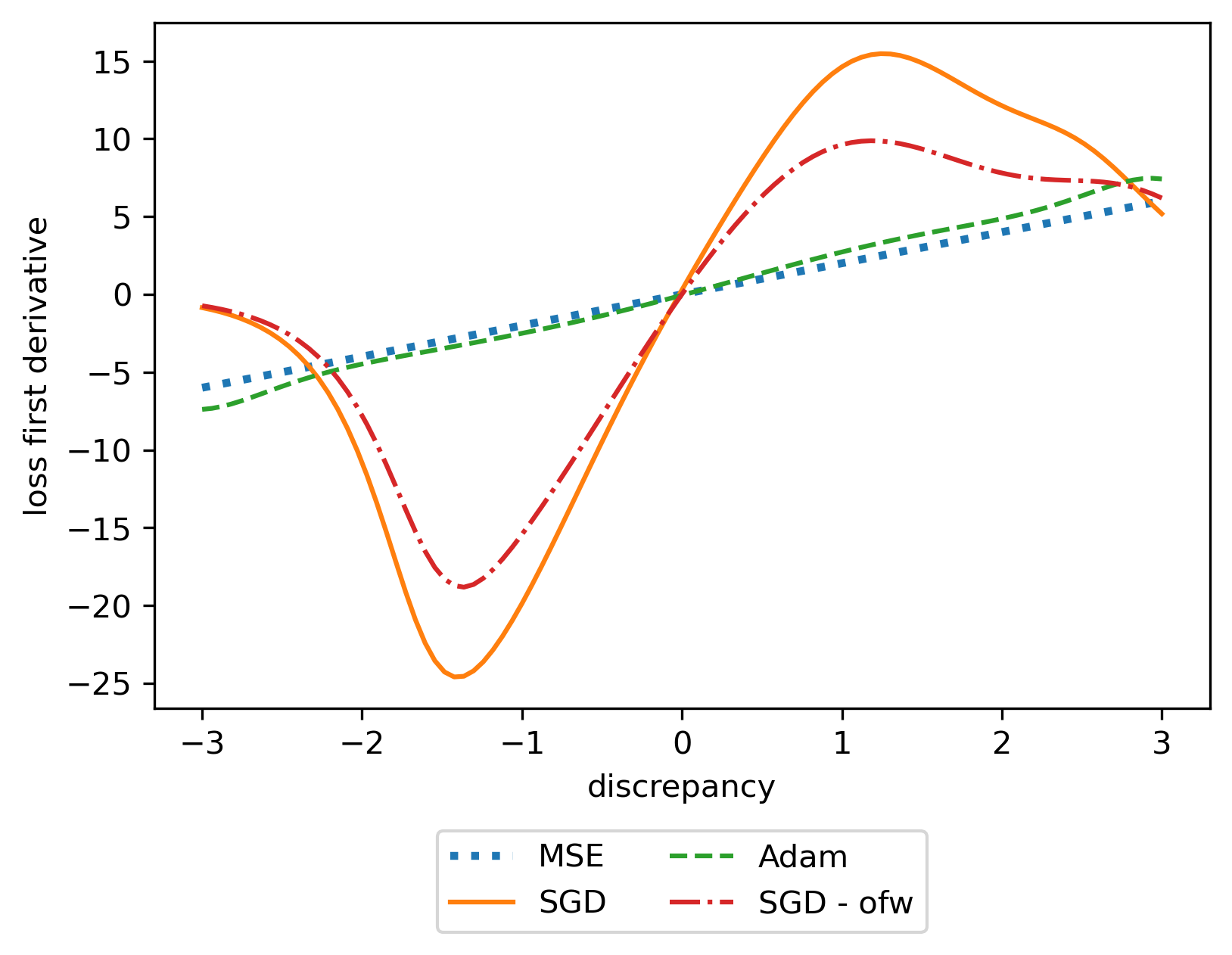}
		\caption{}
		\label{fig:ad:comp:ffn:lossder}
	\end{subfigure}
	\hfill
	\begin{subfigure}[t]{0.48\textwidth}
		\centering
		\includegraphics[width=\textwidth]{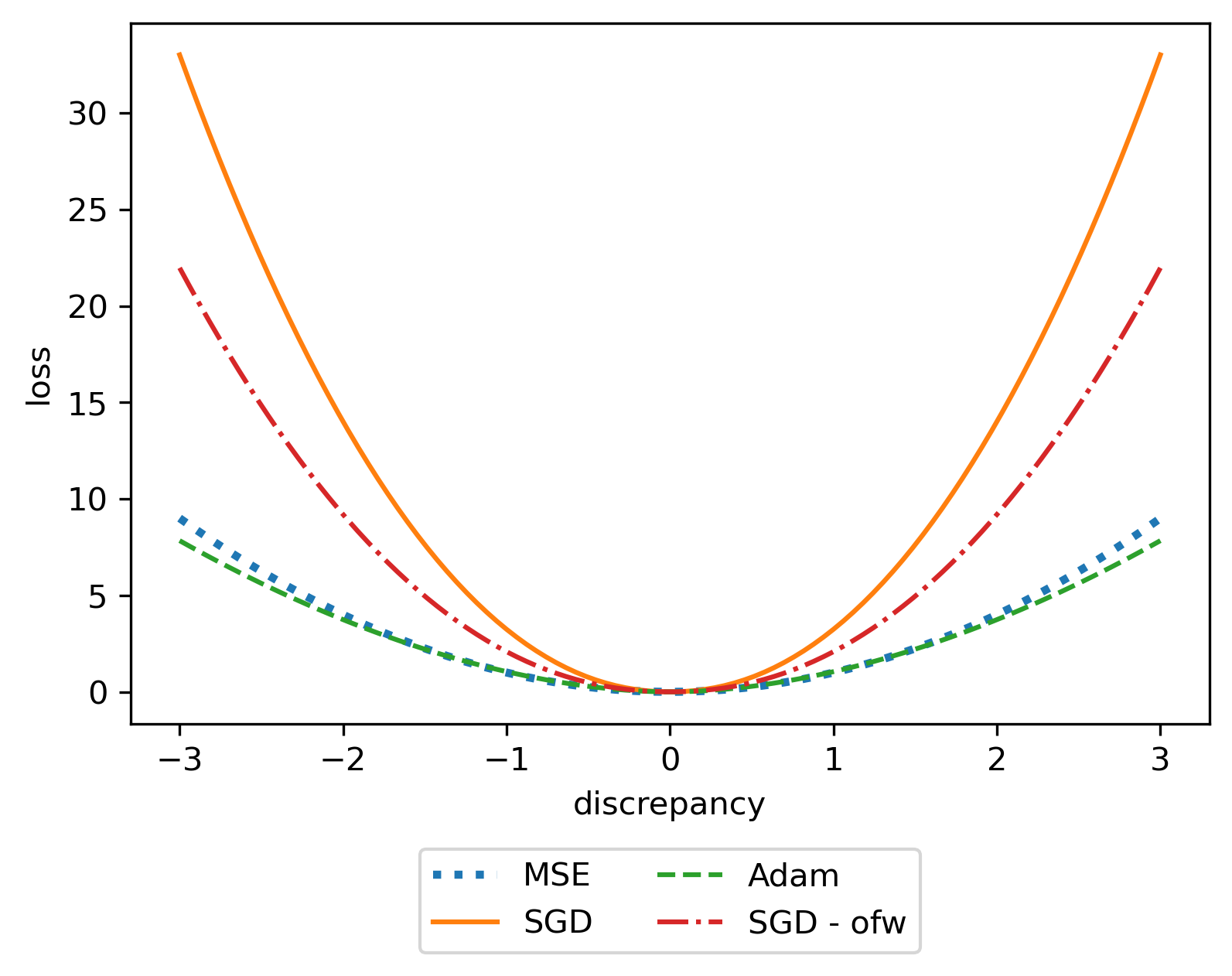}
		\caption{}
		\label{fig:ad:comp:lal:loss}
	\end{subfigure}
	\hfill
	\begin{subfigure}[t]{0.48\textwidth}
		\centering
		\includegraphics[width=\textwidth]{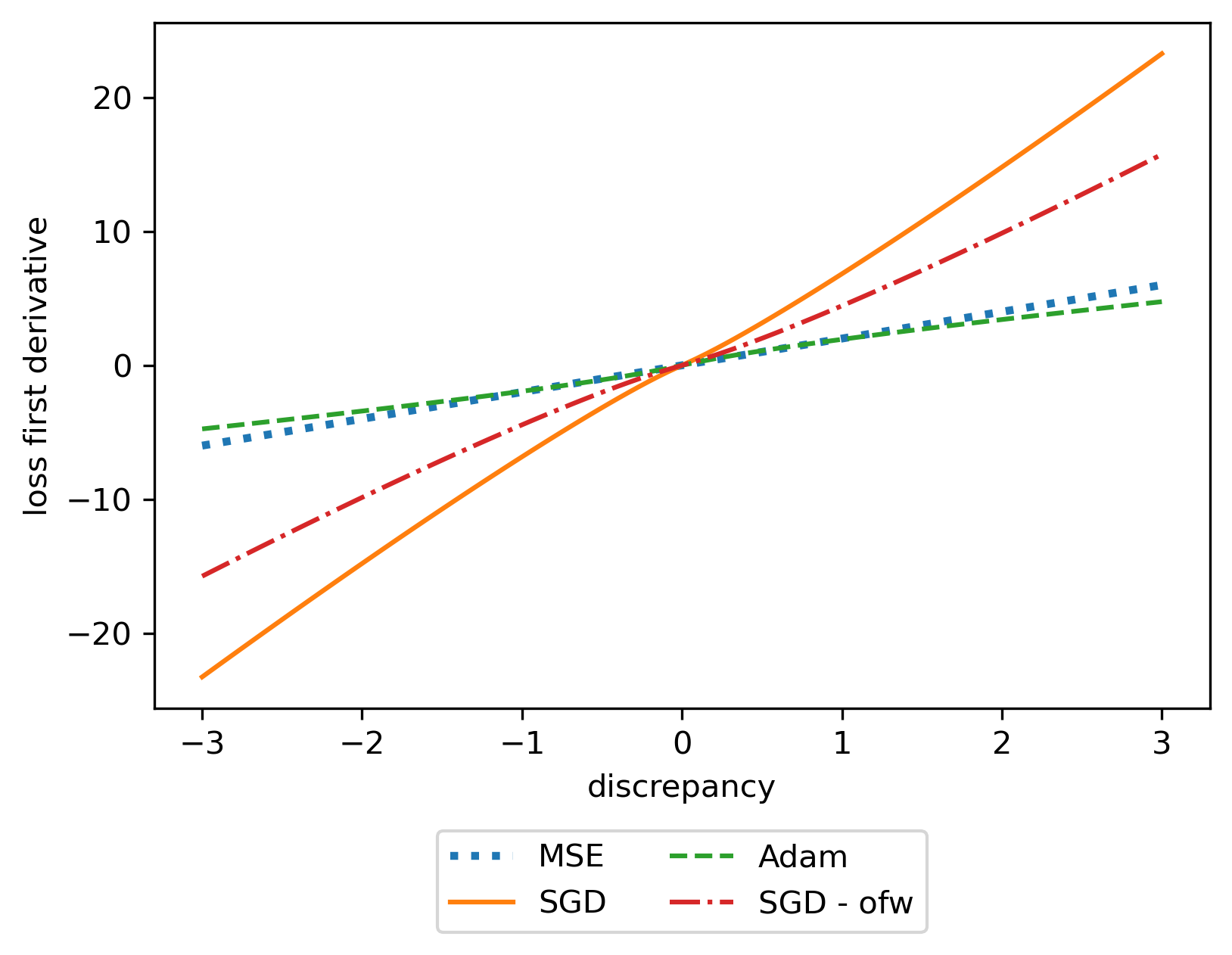}
		\caption{}
		\label{fig:ad:comp:lal:lossder}
	\end{subfigure}
	\caption{Advection equation: Final learned losses (a, c) and corresponding first-order derivatives (b, d), with FFN (a, b) and LAL (c, d) parametrizations. 
		Results obtained via meta-training with SGD as inner optimizer (without and with meta-learning the objective function weights; SGD and SGD - ofw) and with Adam optimizer; comparisons with MSE are also included.
		For SGD as inner optimizer, both FFN and LAL parametrizations with MSE initialization yield highly different learned losses as compared to MSE, with FFN yielding more complex learned losses.
		SGD - ofw leads to an asymmetric final learned loss and is found in meta-testing to deteriorate performance.
		For Adam as inner optimizer, final learned losses are close to MSE.
	}
	\label{fig:ad:comp}
\end{figure}
\begin{figure}[H]
	\centering
	\begin{subfigure}[t]{0.48\textwidth}
		\centering
		\includegraphics[width=\textwidth]{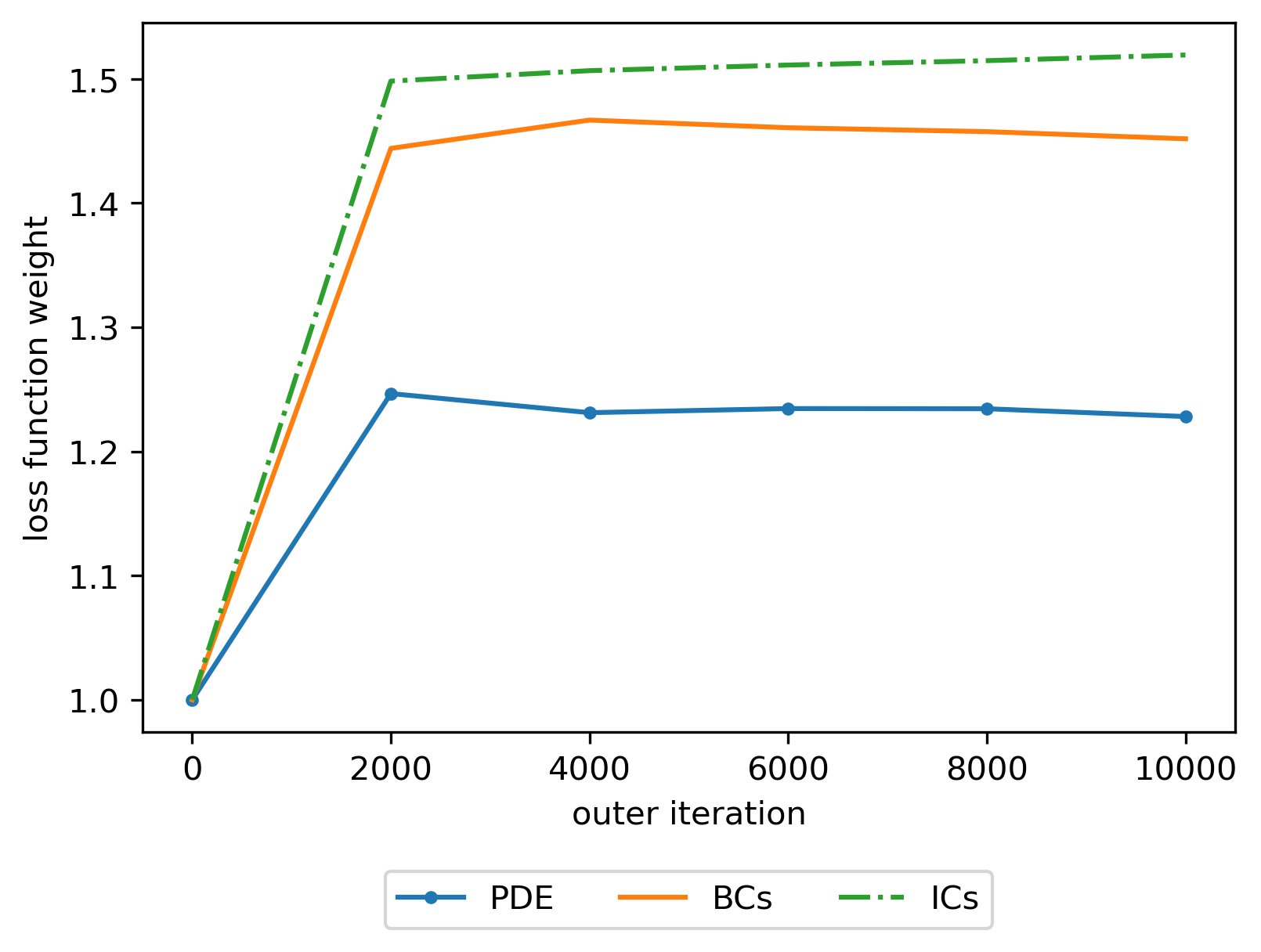}
		\caption{}
		\label{fig:ad:params:ffn}
	\end{subfigure}
	\hfill
	\begin{subfigure}[t]{0.48\textwidth}
		\centering
		\includegraphics[width=\textwidth]{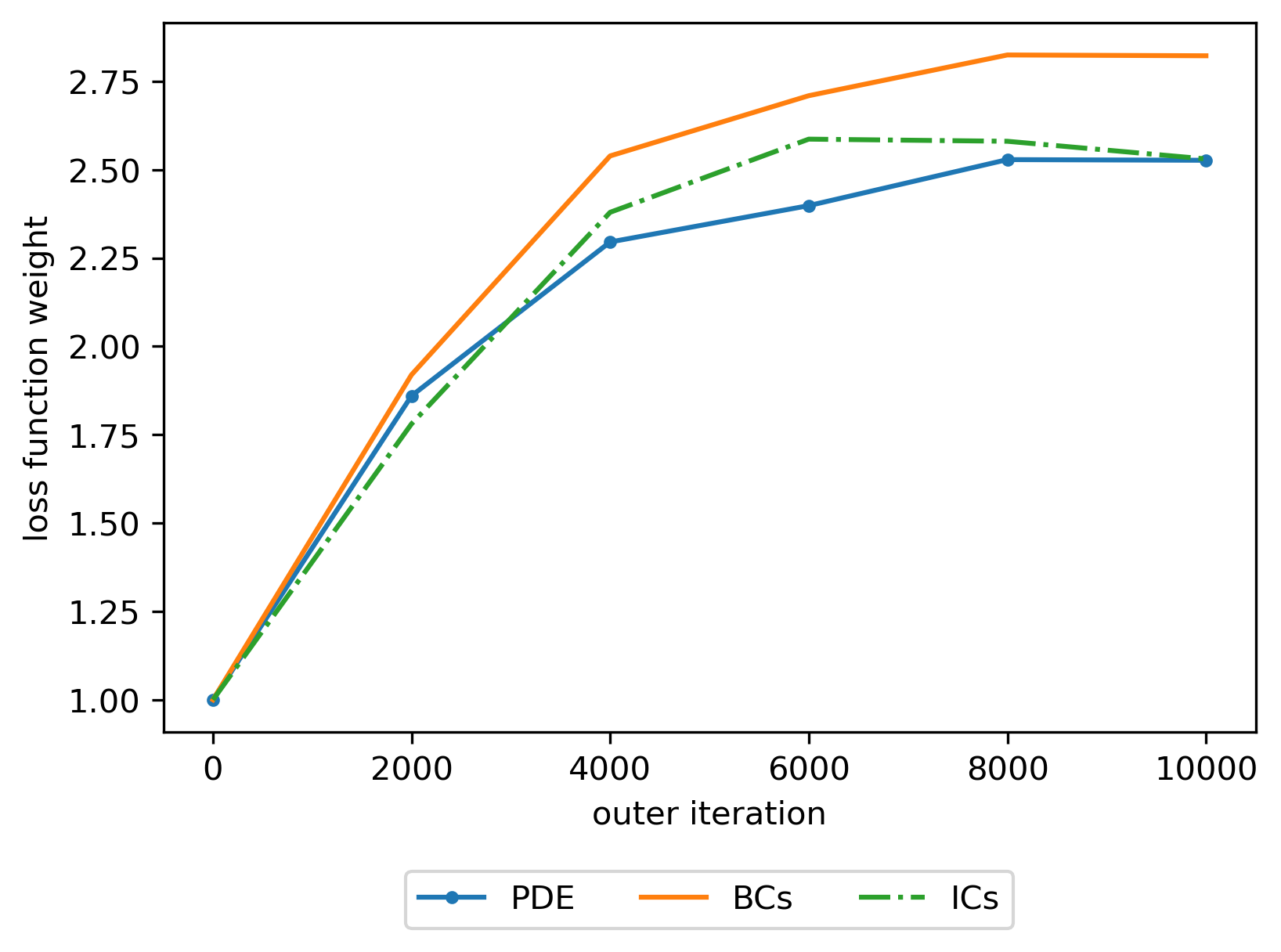}
		\caption{}
		\label{fig:ad:params:lal}
	\end{subfigure}
	\caption{Advection equation: Learned objective function weights, pertaining to PDE, BCs, and ICs residuals, as a function of outer iteration in meta-training; see Section~\ref{sec:meta:design:param:weights}.
		Results obtained with FFN (a) and LAL (b) parametrizations.
		All objective function weights increase for both FFN and LAL parametrizations, which translates into learning rate increase, while FFN and LAL disagree on how they balance ICs. 	
	}
	\label{fig:ad:params}
\end{figure}
\begin{figure}[H]
	\centering
	\includegraphics[width=.6\linewidth]{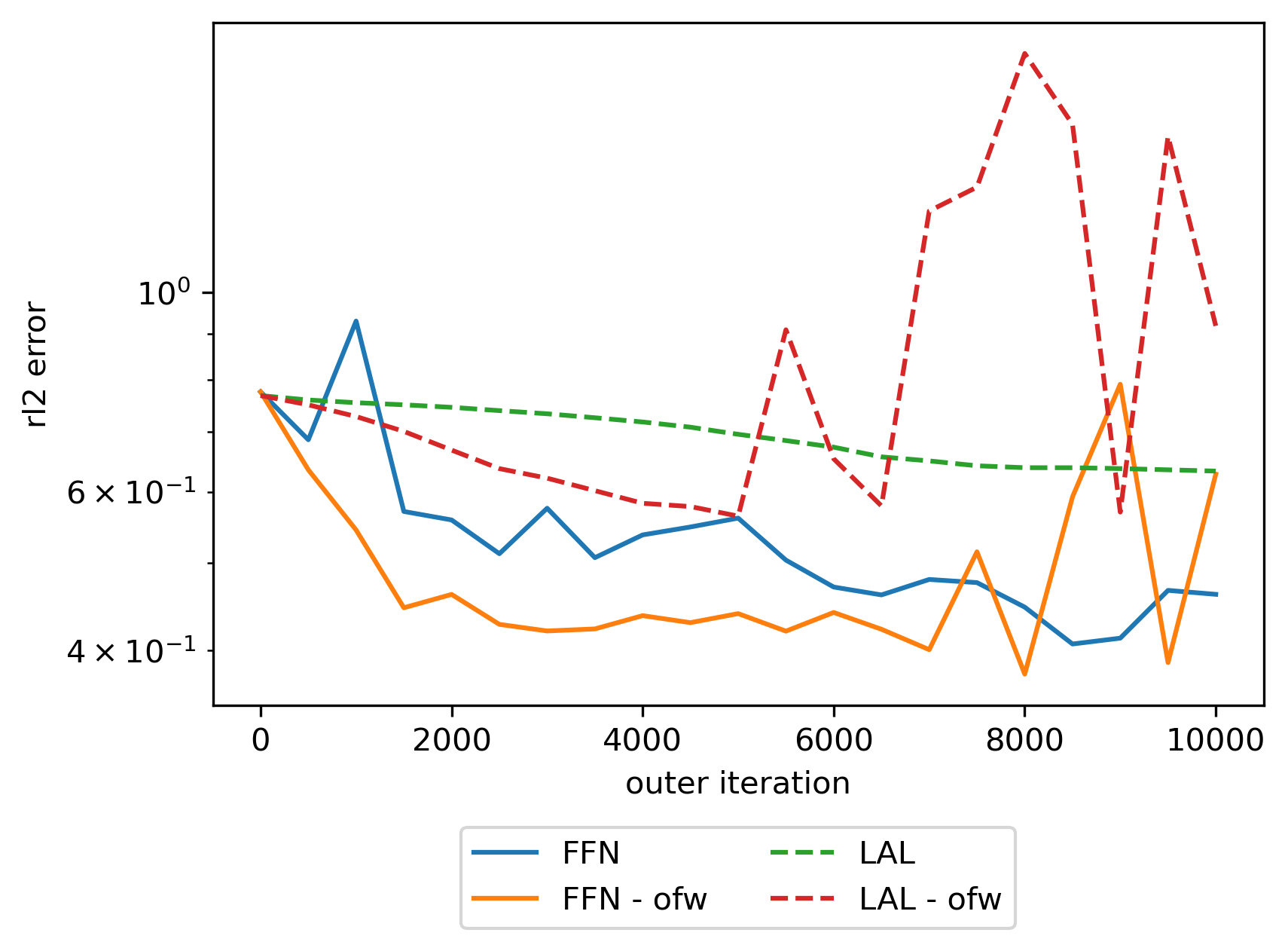}  
	\caption{Advection equation: Meta-testing results (relative $\ell_2$ test error on $1$ unseen task after $100$ iterations) obtained using learned loss snapshots and performed during meta-training (every $500$ outer iterations); these can be construed as \textit{meta-validation error} trajectories.
		Results related to FFN (without and with objective function weights meta-learning; FFN and FFN - ofw) and to LAL (without and with objective function weights meta-learning; LAL and LAL - ofw) are included.
		In this experiment, although initially objective function weights meta-learning improves performance, the corresponding final learned losses in conjunction with the final learned weights eventually deteriorate performance. 
	}
	\label{fig:ad:traintest}
\end{figure}

\paragraph*{Meta-testing with SGD.}
For evaluating the performance of the captured learned loss snapshots, we employ them for meta-testing on $5$ ID tasks and compare with standard loss functions from Table~\ref{tab:losses}.
Specifically, we train with SGD for $10{,}000$ iterations with learning rate $0.01$ (same as in meta-training), $5$ tasks using learned and standard loss functions, and record the rl2 error on $10{,}000$ exact solution datapoints. 
As learned losses, we use the ones obtained with SGD as inner optimizer and without objective function weights meta-learning.
The test error histories as well as the OAL parameters during training are shown in Fig.~\ref{fig:ad:test:misc}, and the minimum rl2 error results are shown in Fig.~\ref{fig:ad:test:min:stats}. 
As shown in Fig.~\ref{fig:ad:test:oal:params}, the final learned loss LAL 5 is much different than both OAL 1 and 2, which converge to a robustness parameter value close to 3 for all tasks. 
In Figs.~\ref{fig:ad:test:train:hists}-\ref{fig:ad:test:min:stats}, we see that the loss functions learned with both parametrizations achieve an average minimum rl2 error during $10{,}000$ iterations that is significantly smaller than all the considered losses (even the online adaptive ones) although they have been meta-trained with only $20$ iterations. 
\begin{figure}[H]
	\centering
	\begin{subfigure}[t]{0.48\textwidth}
		\centering
		\includegraphics[width=\textwidth]{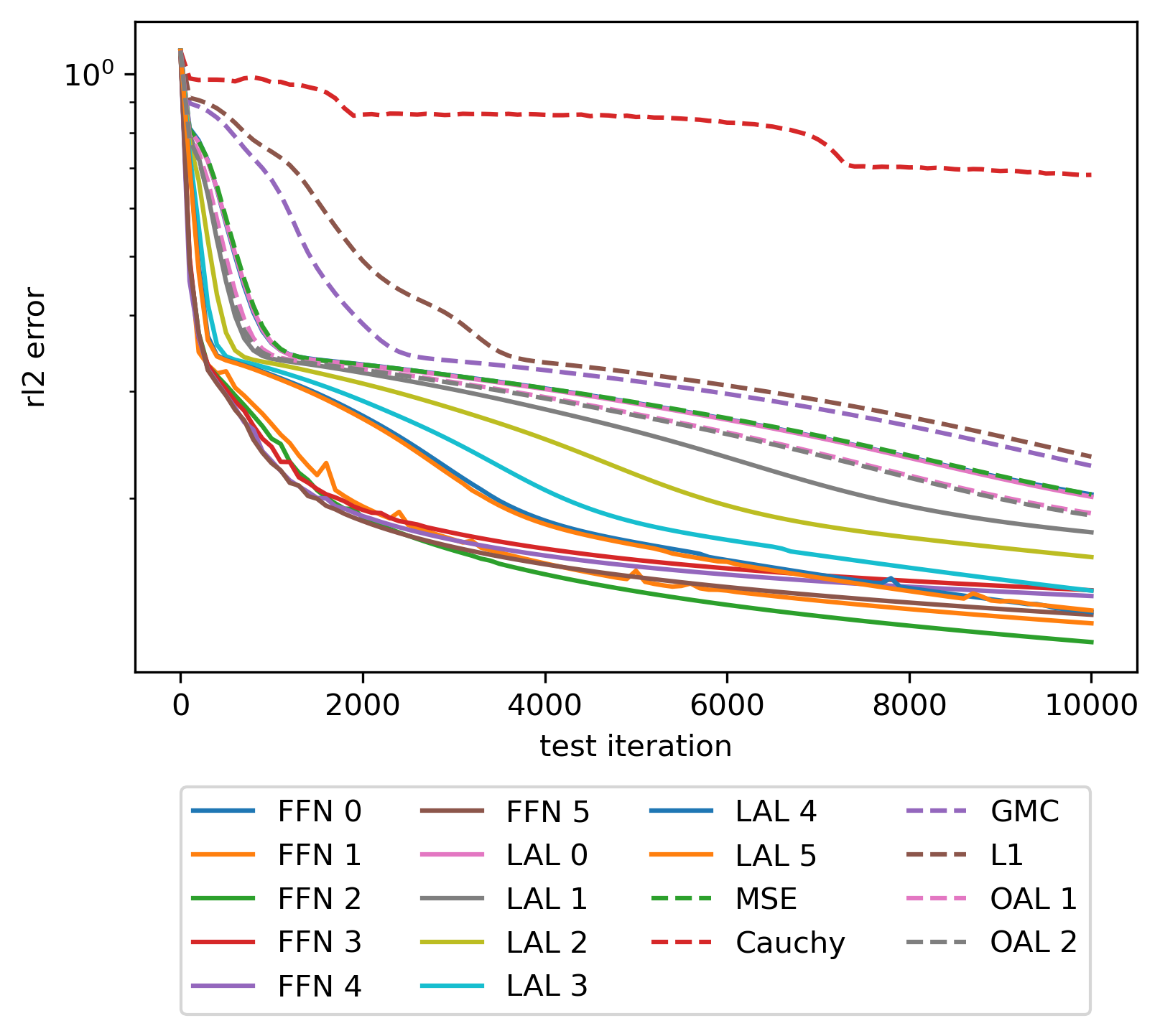}
		\caption{}
		\label{fig:ad:test:train:hists}
	\end{subfigure}
	\hfill
	\begin{subfigure}[t]{0.48\textwidth}
		\centering
		\includegraphics[width=\textwidth]{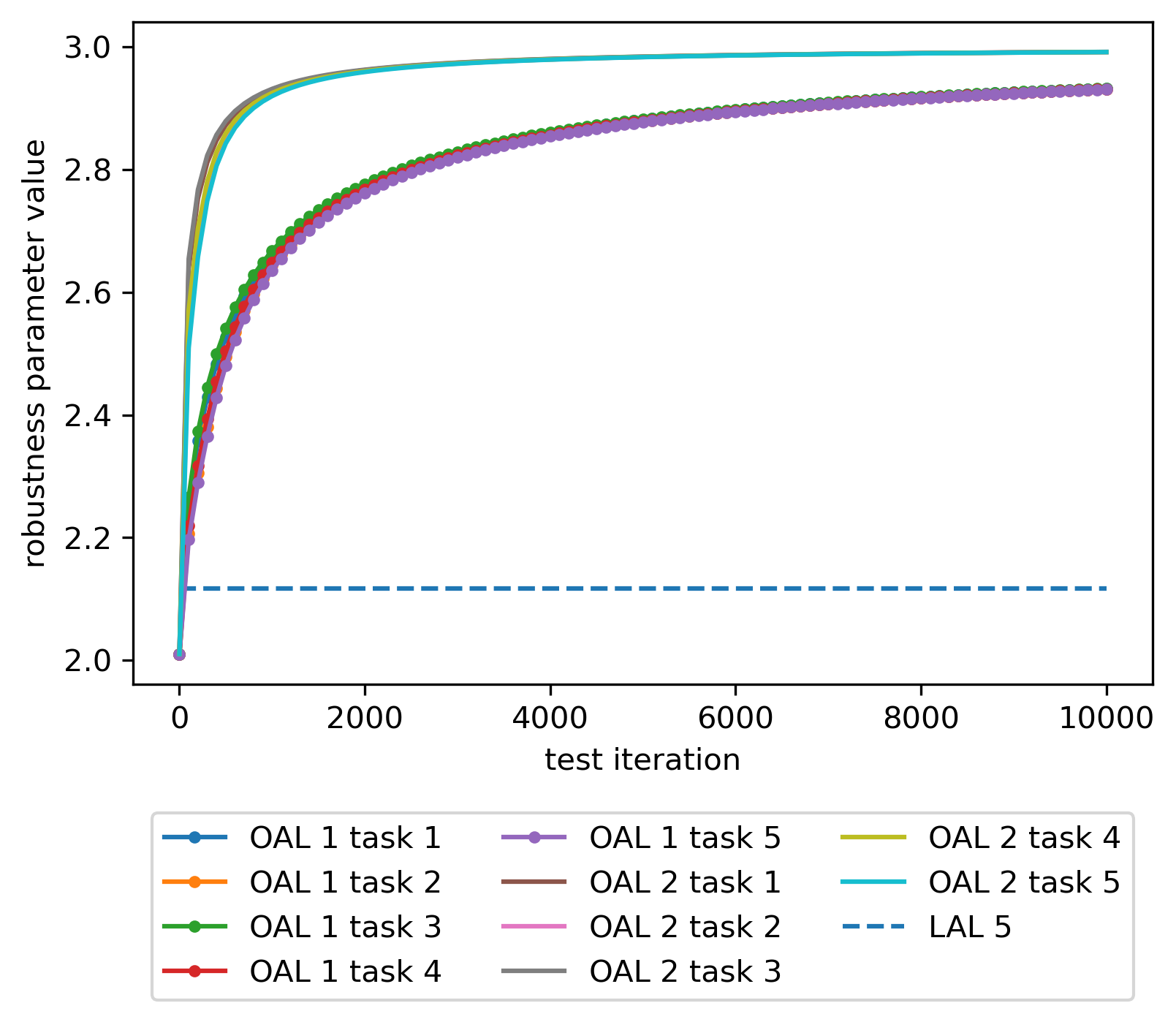}
		\caption{}
		\label{fig:ad:test:oal:params}
	\end{subfigure}
	\caption{Advection equation: In-distribution (ID) meta-testing results.
		Results obtained using SGD with learning rate $0.01$ for $10{,}000$ iterations.
		(a) shows the meta-testing relative $\ell_2$ test error (rl2) trajectories for all loss functions considered.
		(b) shows the robustness parameter trajectories for online adaptive loss functions OAL 1 and OAL 2, with loss-specific learning rates $0.01$ and $0.1$, respectively (see Section~\ref{sec:meta:design:param:LAL}); comparison with final learned loss obtained via meta-training with LAL parametrization also included.
		As shown in (a), learned losses perform better than considered standard and adaptive losses.
		Furthermore, as shown in (b), final learned loss LAL 5 is much different than both OAL 1 and 2, which converge to a robustness parameter value close to 3 for all tasks. 
	}	
	\label{fig:ad:test:misc}
\end{figure}
\begin{figure}[H]
	\centering
	\includegraphics[width=.7\linewidth]{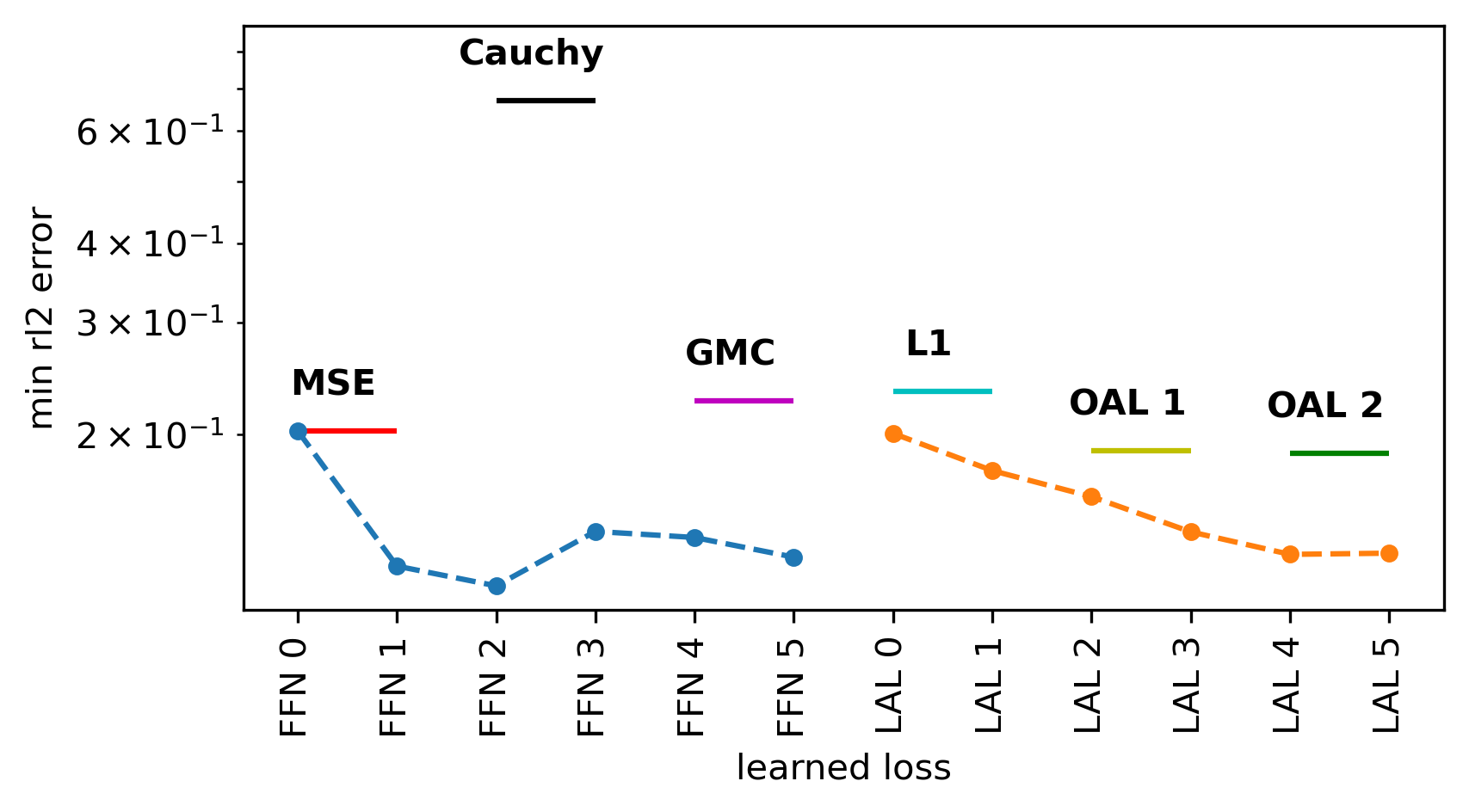}  
	\caption{Advection equation: Minimum relative test $\ell_2$ error (rl2) averaged over $5$ ID tasks during meta-testing with SGD for $10{,}000$ iterations.
		Learned loss snapshots (FFN 0-6 and LAL 0-6, obtained with SGD as inner optimizer) are compared with standard loss functions of Table~\ref{tab:losses} and with online adaptive loss functions OAL 1 and OAL 2 (2 loss-specific learning rates).
		Loss functions learned with both FFN and LAL parametrizations achieve an average minimum rl2 error during $10{,}000$ iterations that is significantly smaller than all the considered losses (even the online adaptive ones), although they have been meta-trained with only $20$ iterations. 
	}
	\label{fig:ad:test:min:stats}
\end{figure}

\paragraph*{Meta-testing with Adam.}
Next, we employ the learned losses obtained using Adam as inner optimizer in the same meta-testing experiment as the one of Figs.~\ref{fig:ad:test:misc}-\ref{fig:ad:test:min:stats}; except for the fact that Adam with learning rate $10^{-3}$ is used in meta-testing instead of SGD.
The minimum rl2 error results are shown in Fig.~\ref{fig:ad:adam:test:min:stats}, where we see that the learned losses do not improve performance as compared to MSE; same results have been observed for the examples of Sections~\ref{sec:examples:rd}-\ref{sec:examples:burgers} but these results are not included in this paper.
One reason for this result is the fact that Adam depends on the whole history of gradients during optimization through an exponentially decaying average that only discards far in the past gradients.
However, our learned losses have been meta-trained with only $20$ inner iterations, and thus it is unlikely that they could have learned this memory property of Adam. 
To illustrate the validity of the above explanation, we perform meta-training with Adam as inner optimizer with varying exponential decay parameters and subsequently meta-testing with the obtained learned loss snapshots (see Fig.~\ref{fig:ad:traintest:Adam:experiment}). 
Specifically, we use values for the pair ($\beta_1, \beta_2$), corresponding to the decay parameters for the first and second moment estimates in Adam (see \cite{kingma2014adam}) in the set $\{(0.5, 0.5), (0.8, 0.8), (0.9, 0.999) = \text{default}, (0.99, 0.9999)\}$ with higher numbers corresponding to higher dependency on the far past.
Note that the decay factors multiplying the $21st$ gradient in the past (i.e., 1 gradient beyond the history used in meta-training) are approximately $10^{-6}$ and $10^{-2}$ for the pairs $(0.5, 0.5)$ and $(0.8, 0.8)$, respectively.   
As expected and shown in Fig.~\ref{fig:ad:traintest:Adam:experiment}, higher $\beta$ values corresponding to higher dependency on the far past yield deteriorating performance of the learned losses.  
In this regard, we use only SGD as inner optimizer in the rest of the computational examples and leave the task of improving the performance of the technique for addressing inner optimizers with memory, such as SGD with momentum, AdaGrad, RMSProp and Adam, as future work.
\begin{figure}[H]
	\centering
	\includegraphics[width=.7\linewidth]{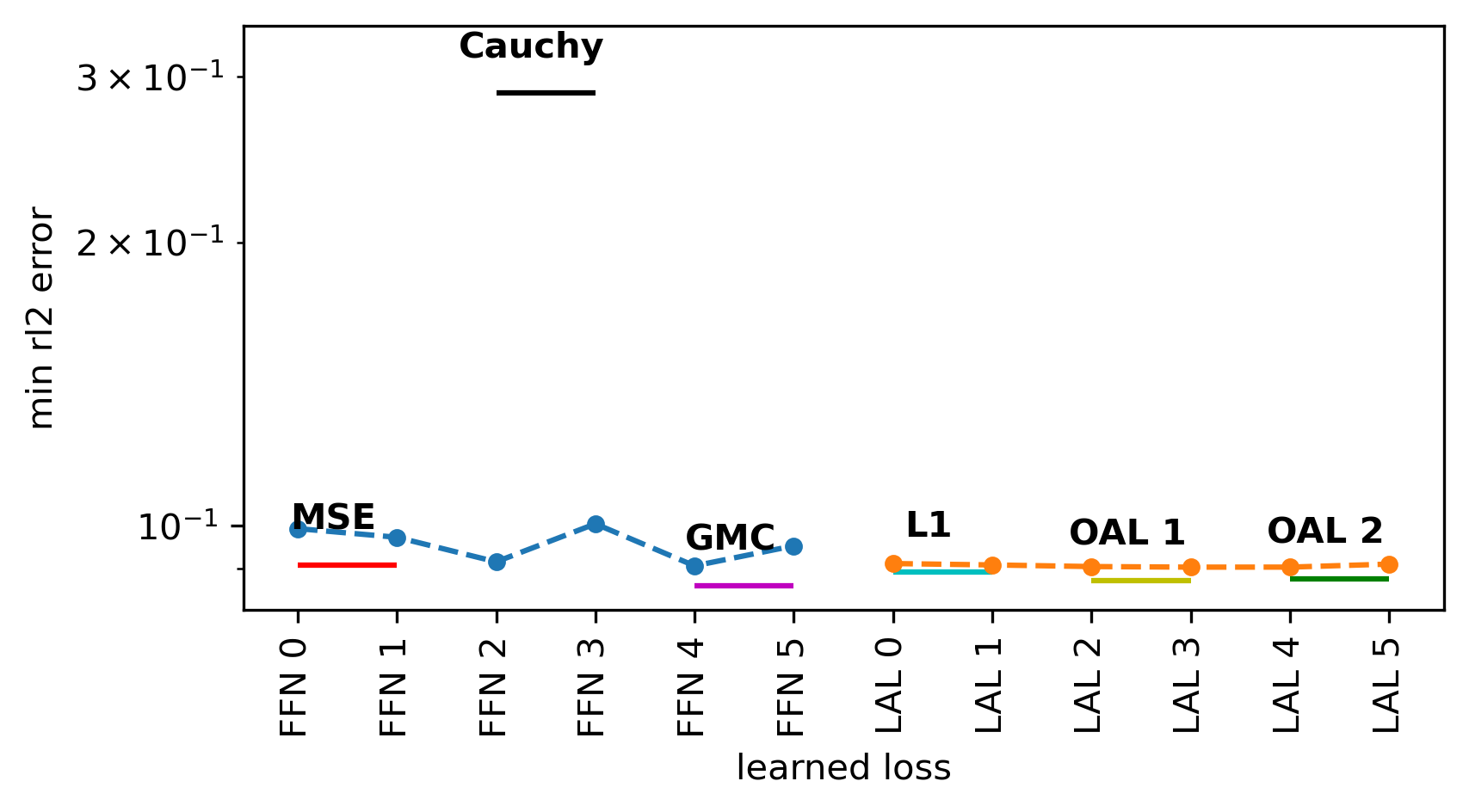}  
	\caption{Advection equation: Minimum relative test $\ell_2$ error (rl2) averaged over $5$ ID tasks during meta-testing with Adam for $10{,}000$ iterations.
		Learned loss snapshots (FFN 0-6 and LAL 0-6, obtained with Adam as inner optimizer) are compared with standard loss functions of Table~\ref{tab:losses} and with online adaptive loss functions OAL 1 and OAL 2 (2 loss-specific learning rates).
		Loss functions learned with both FFN and LAL parametrizations do not improve performance as compared to MSE.
		This is attributed to the fact that our learned losses have been meta-trained with only $20$ inner iterations, whereas Adam depends on the whole history of gradients during optimization through an exponentially decaying average; see Fig.~\ref{fig:ad:traintest:Adam:experiment} for an experiment with various exponential decay parameters. 
	}
	\label{fig:ad:adam:test:min:stats}
\end{figure}
\begin{figure}[H]
	\centering
	\includegraphics[width=.7\linewidth]{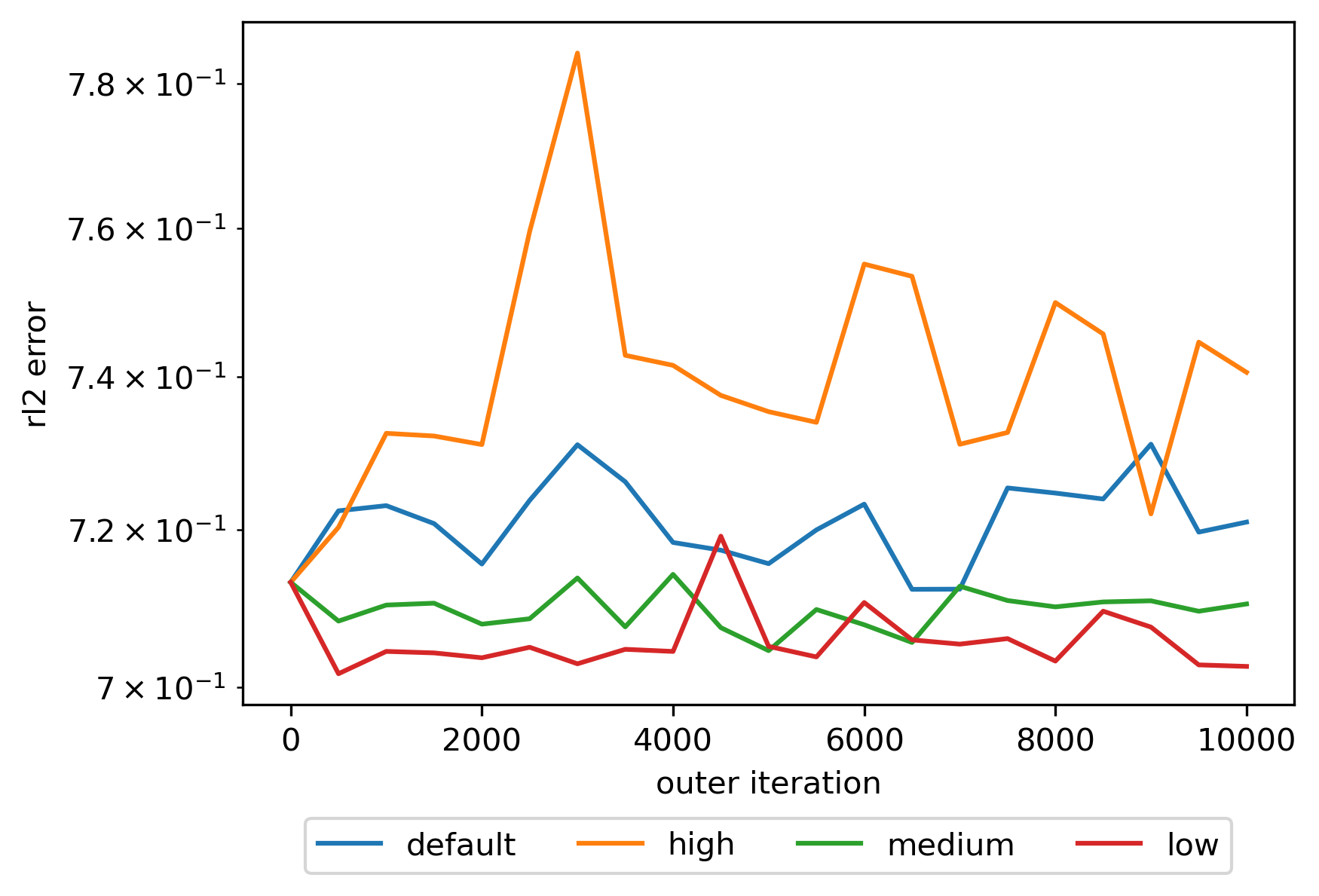}  
	\caption{Advection equation: Meta-testing results (relative $\ell_2$ test error on $10$ unseen task after $100$ iterations) obtained using learned loss snapshots and performed during meta-training (every $500$ outer iterations).
		Results are related to meta-training with FFN parametrization and Adam as inner optimizer with 4 different $\beta$ pair values: 
		$\text{low} = (0.5, 0.5)$, $\text{medium} = (0.8, 0.8)$, $\text{default} = (0.9, 0.999)$, and $\text{high} = (0.99, 0.9999)$; levels low, medium, high indicate degree of optimizer dependence on gradient history older than $20$ iterations, where $20$ is the number of inner iterations used in meta-training.
		Higher $\beta$ values corresponding to higher dependency on the far past yield deteriorating performance of the learned losses.  
	}
	\label{fig:ad:traintest:Adam:experiment}
\end{figure}

\subsection{Task distributions defined based on reaction-diffusion equation with varying source term}
\label{sec:examples:rd}

Reaction-diffusion equations are used to describe diverse systems ranging from population dynamics to chemical reactions and have the general form
\begin{equation}\label{eq:react}
	\partial_tu = D\Delta u + z(x, u, \nabla u),
\end{equation}
where $\Delta$ denotes the Laplace operator and $D$ is called diffusion coefficient.
In Eq.~\eqref{eq:react}, $D\Delta u$ represents the diffusion term whereas $z(x, u, \nabla u)$ the reaction term. 
In the following, without loss of generality, we consider a two-dimensional nonlinear, steady-state version of Eq.~\eqref{eq:react} given as
\begin{equation}\label{eq:react:ex}
	k (\partial^2_{x_1} u + \partial^2_{x_2} u) +  u(1 - u^2) = z,
\end{equation}
where $x_1, x_2 \in [-1, 1]$ refer to space dimensions, $z$ can be interpreted as a source term, and $u$ is considered as known at all boundaries. 

To demonstrate the role of task distributions in the context of varying excitation terms, we consider a family of fabricated solutions $u$ that after differentiation produce a family of $z$ terms in Eq.~\eqref{eq:react:ex}.
As an illustrative example, the task distribution $p(z_{\lambda})$ can be defined by drawing $\lambda = \{\alpha_1, \alpha_2, \omega_1, \omega_2, \omega_3, \omega_4\}$, with $\lambda \sim p(\lambda)$, constructing an analytical solution $u = \alpha_1 \tanh(\omega_1 x_1) \tanh(\omega_2 x_2) + \alpha_2 \sin(\omega_3 x_1) \sin(\omega_4 x_2)$, and constructing $z_{\lambda}$ via Eq.~\eqref{eq:react:ex}.
Obviously, in practice the opposite is true; different excitation terms $z$ pose a novel problem of the form of Eq.~\eqref{eq:react:ex} to be solved. 
Nevertheless, defining $z$ by using fabricated $u$ solutions that are by construction related to each other helps to demonstrate the concepts of this work in a more straightforward manner. 

\paragraph*{Meta-training.}
The task distribution parameters used in meta-training are shown in Table~\ref{tab:rd:params}.
The meta-training design options are the same as in Section~\ref{sec:examples:ad} except for the fact that objective function weights are not meta-learned.
Furthermore, the PINN architecture consists of $3$ hidden layers with $20$ neurons each and $tanh$ activation function.

In Fig.~\ref{fig:rd:comp}, we show the final loss functions (FFN and LAL parametrizations) as obtained with SGD and Adam as inner optimizers.
In addition, in Fig.~\ref{fig:rd:comp} we include the loss functions obtained using the exact solution data in the outer objective instead of the composite PINNs loss of Eq.~\eqref{eq:lossml:outer:loss}. 
In the considered example, this data is available because a fabricated solution is used, whereas in practice this data may be originating from a numerical solver or from measurements.   
This is referred to as \textit{double data (DD)} in the plots because different data is used for the inner and for the outer objective; single data is denoted as SD.
The number of datapoints used for evaluating the outer objective of Eq.~\eqref{eq:lossml:outer:loss} is shown in Table~\ref{tab:rd:outer:data}, whereas the corresponding number for the inner objective of Eq.~\eqref{eq:lossml:inner:loss} is the same as the single-data case in Table~\ref{tab:rd:outer:data}.

\paragraph*{Connection with the theory of Section~\ref{sec:meta:design:properties}.}

Whereas regularization is not required for LAL (see Proposition~\ref{prop:1}), as shown in Fig.~\ref{fig:rd:comp} the loss function corresponding to SGD with FFN and no regularization is shifted to the right; i.e., its first-order derivative at $0$ discrepancy is not zero.
As a result, the \msec\ of Corollary~\ref{corollary:1} is not satisfied and the learned loss leads to divergence in optimization when used for meta-testing.
Thus, we also include in Fig.~\ref{fig:rd:comp} the regularized loss function obtained via meta-training with the theoretically-driven gradient penalty of Eq.~\eqref{eq:lossml:outer:loss:penalty}, which solves this issue; see also Fig.~\ref{fig:rd:snaps:loss:ffn} for the loss function snapshots captured during meta-training for the non-regularized and regularized cases with the FFN parametrization.

\begin{table}[H]
	\centering
	\caption{Reaction-diffusion equation: Task distribution values used in meta-training and OOD meta-testing. 
		Parameters $\alpha_1$, $\alpha_2$ and parameters $\omega_1$, $\omega_2$, $\omega_3$, $\omega_4$ share the same limits $\alpha_{min}$, $\alpha_{max}$ and $\omega_{min}$, $\omega_{max}$, respectively.
	}
	\begin{tabular}{ccccc}
		\toprule
		&$\alpha_{min}$& $\alpha_{max}$& $\omega_{min}$& $\omega_{max}$\\
		\midrule
		meta-training&$0.1$& $1$& $1$& $5$\\
		OOD meta-testing&$0.1$& $2$& $0.5$& $7$\\
		\bottomrule
	\end{tabular}
	\label{tab:rd:params}
\end{table}
\begin{table}[H]
	\centering
	\caption{Reaction-diffusion equation: PDE, BCs, ICs, and solution data considered in the outer objective of Eq.~\eqref{eq:lossml:outer:loss} for single data (outer objective data same as inner objective) and for double data (solution data considered available in the outer objective) in meta-training ($\{N_{f, val}, N_{b, val}, N_{u_0, val}, N_{u, val}\}$), as well as for OOD meta-testing ($\{N_{f}, N_{b}, N_{u_0}, N_{u}\}$).
	}
	\begin{tabular}{c|cccc}
		\toprule
		& $N_{f(, val)}$     & $N_{b(, val)}$     & $N_{u_0(, val)}$ & $N_{u(, val)}$ \\
		single data (SD) meta-training & $1{,}600$ & $160$ & NA & $0$ \\
		double data (DD) meta-training & $0$ & $0$ & NA & $1{,}600$ \\
		OOD meta-testing & $2{,}500$ & $200$ & NA & $0$ \\
		\bottomrule
	\end{tabular}
	\label{tab:rd:outer:data}
\end{table}
\begin{figure}[H]
	\centering
	\begin{subfigure}[t]{0.48\textwidth}
		\centering
		\includegraphics[width=\textwidth]{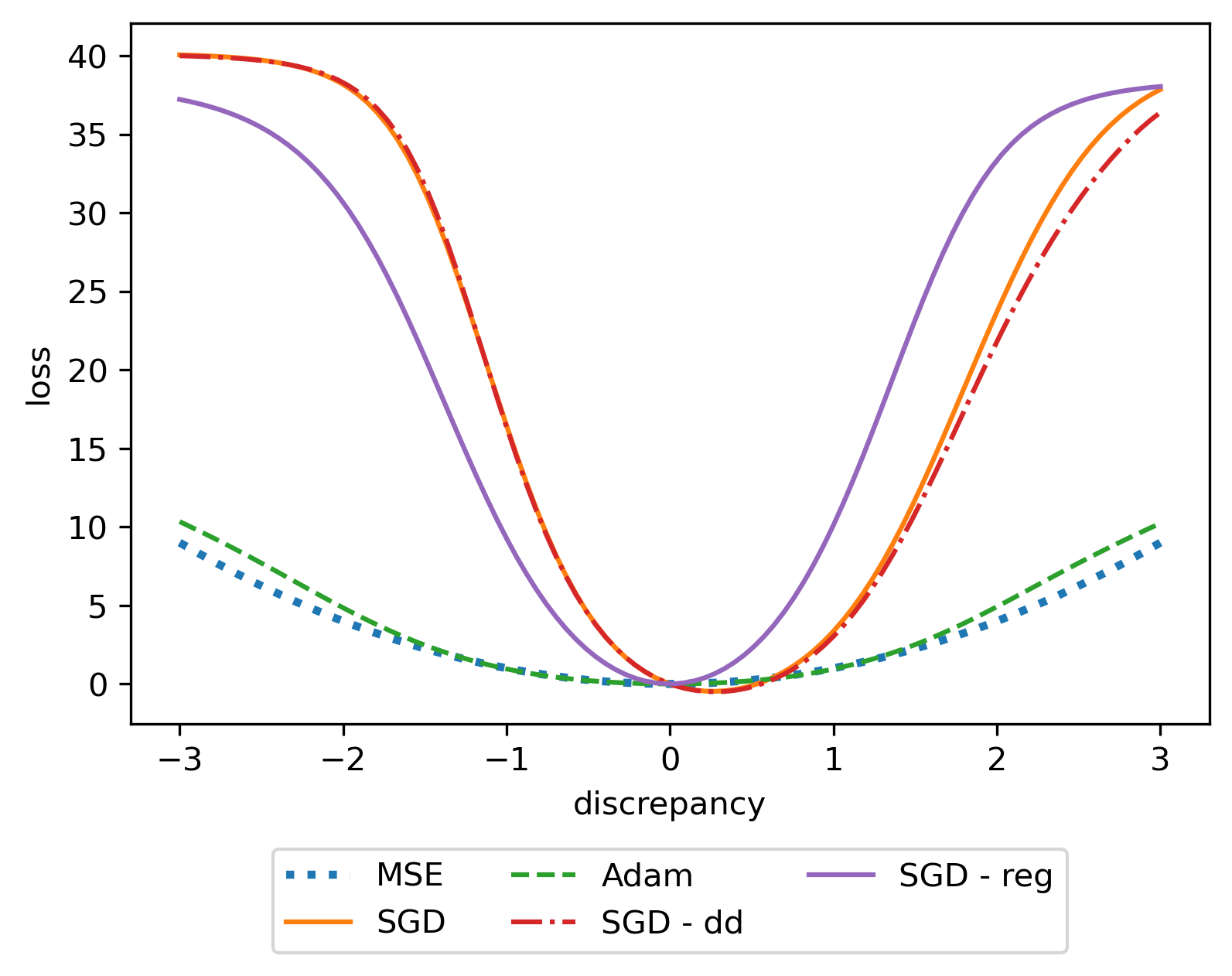}
		\caption{}
		\label{fig:rd:comp:ffn:loss}
	\end{subfigure}
	\hfill
	\begin{subfigure}[t]{0.48\textwidth}
		\centering
		\includegraphics[width=\textwidth]{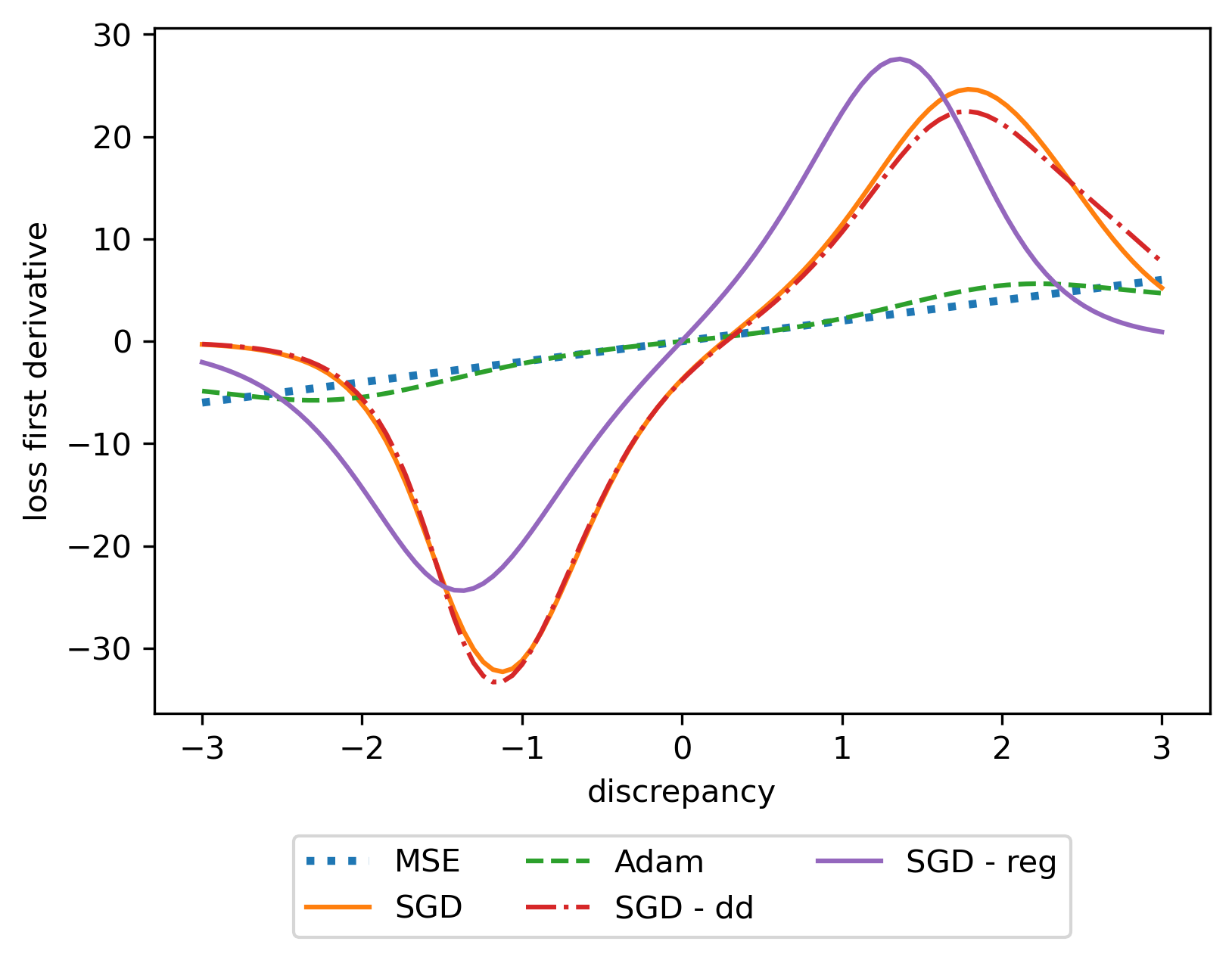}
		\caption{}
		\label{fig:rd:comp:ffn:lossder}
	\end{subfigure}
	\hfill
	\begin{subfigure}[t]{0.48\textwidth}
		\centering
		\includegraphics[width=\textwidth]{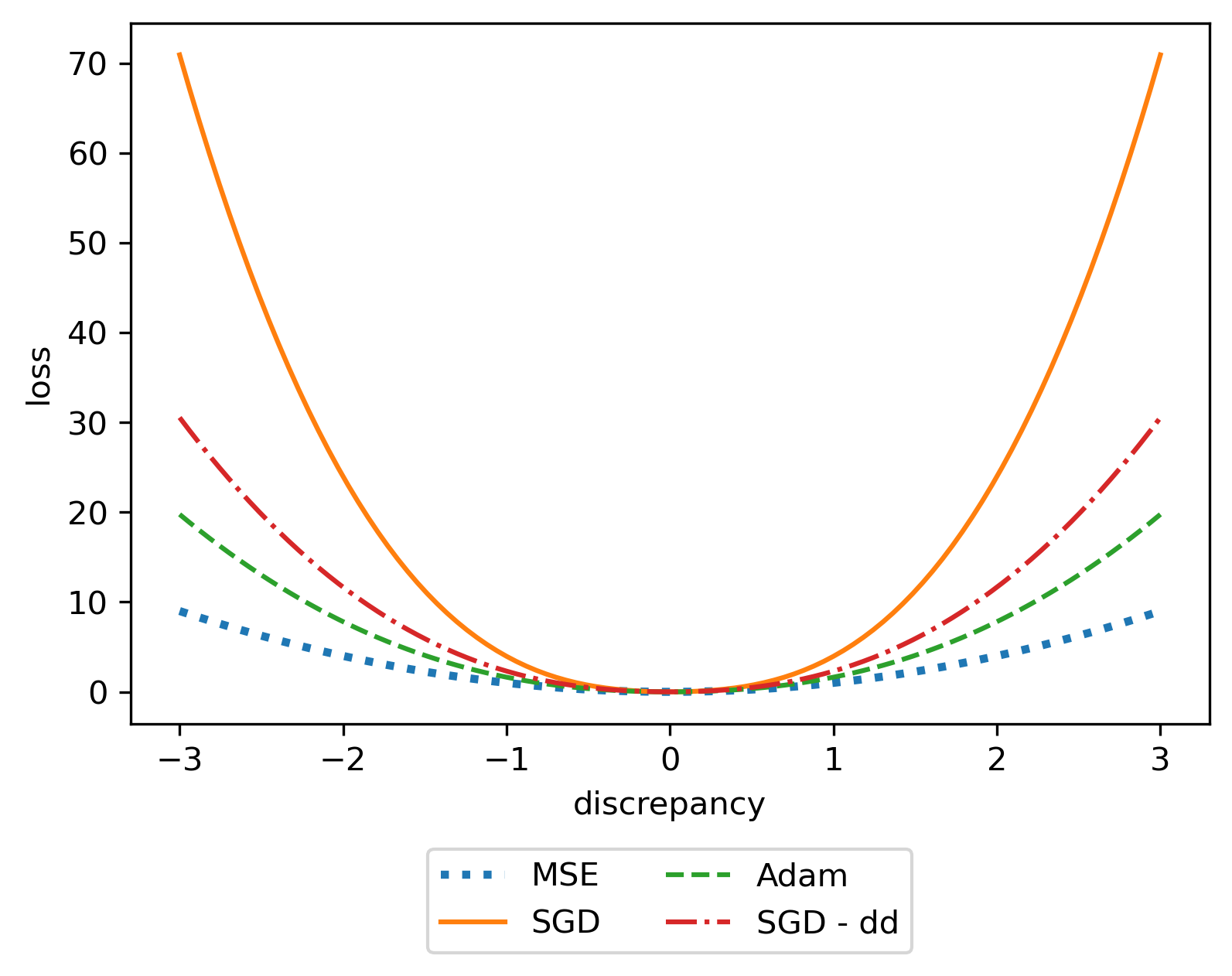}
		\caption{}
		\label{fig:rd:comp:lal:loss}
	\end{subfigure}
	\hfill
	\begin{subfigure}[t]{0.48\textwidth}
		\centering
		\includegraphics[width=\textwidth]{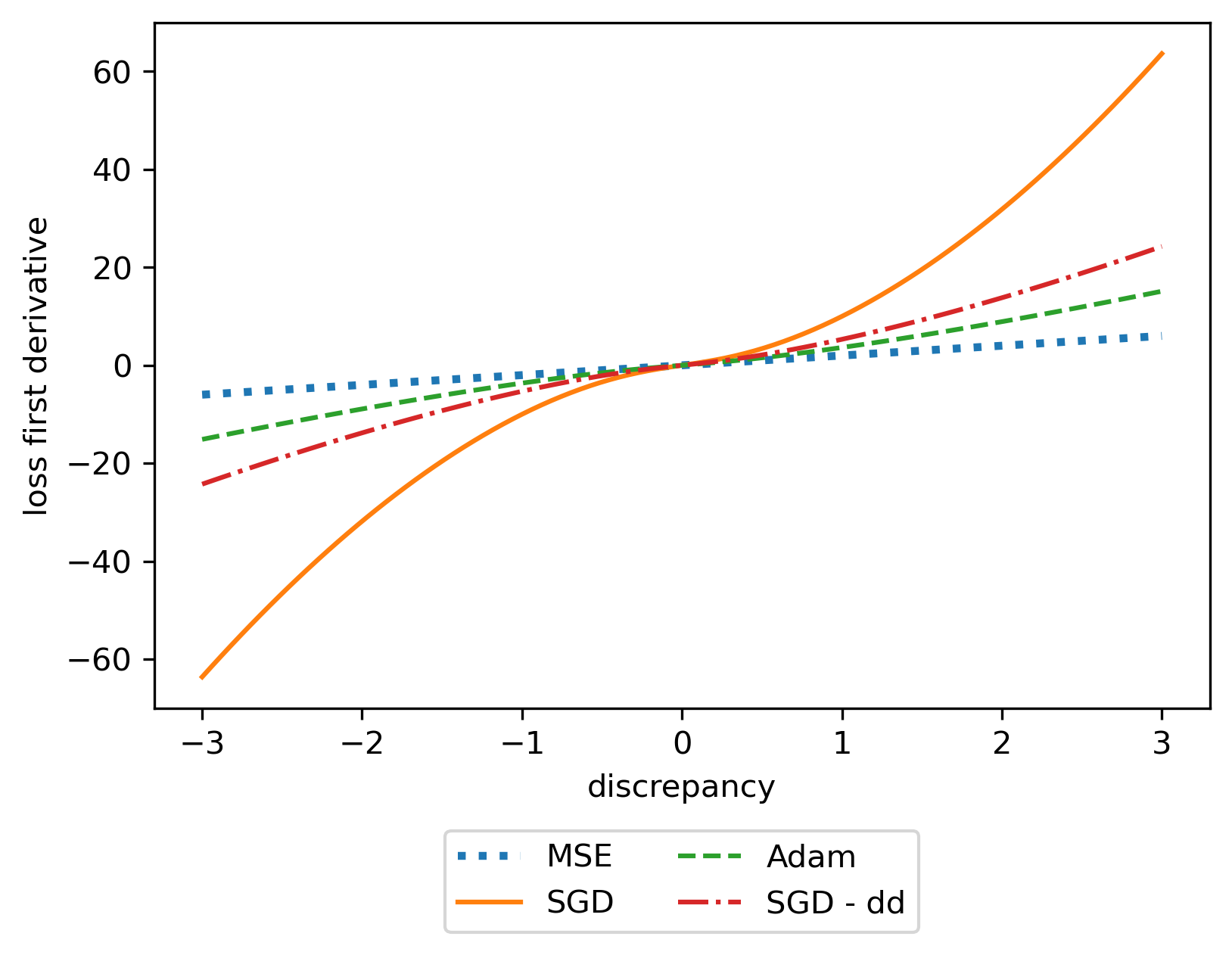}
		\caption{}
		\label{fig:rd:comp:lal:lossder}
	\end{subfigure}
	\caption{Reaction-diffusion equation: Final learned losses (a, c) and corresponding first-order derivatives (b, d), with FFN (a, b) and LAL (c, d) parametrizations. 
		Results obtained via meta-training with SGD as inner optimizer (with single data, without and with regularization, and with double data; SGD, SGD - reg, SGD - dd) and with Adam optimizer; comparisons with MSE are also included.
		Whereas regularization is not required for LAL (Proposition~\ref{prop:1}), the loss function corresponding to SGD with FFN and no regularization is shifted to the right, i.e., its first-order derivative at $0$ discrepancy is not zero.
		Theory-driven regularization as discussed in Section~\ref{sec:meta:design:properties} fixes this issue.
	}
	\label{fig:rd:comp}
\end{figure}
\begin{figure}[H]
	\centering
	\begin{subfigure}[t]{0.48\textwidth}
		\centering
		\includegraphics[width=\textwidth]{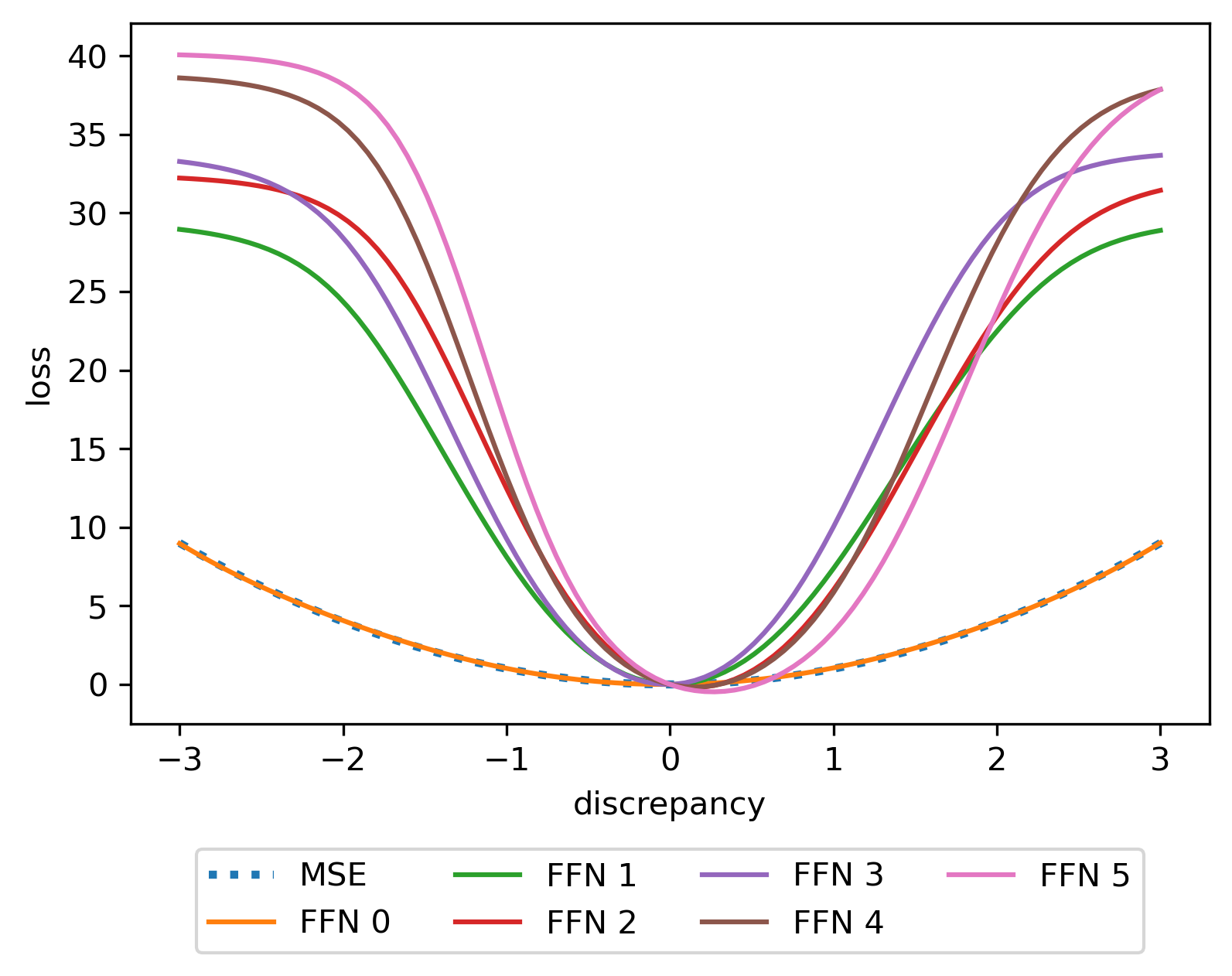}
		\caption{}
		\label{fig:rd:snaps:loss:ffn:unreg}
	\end{subfigure}
	\hfill
	\begin{subfigure}[t]{0.48\textwidth}
		\centering
		\includegraphics[width=\textwidth]{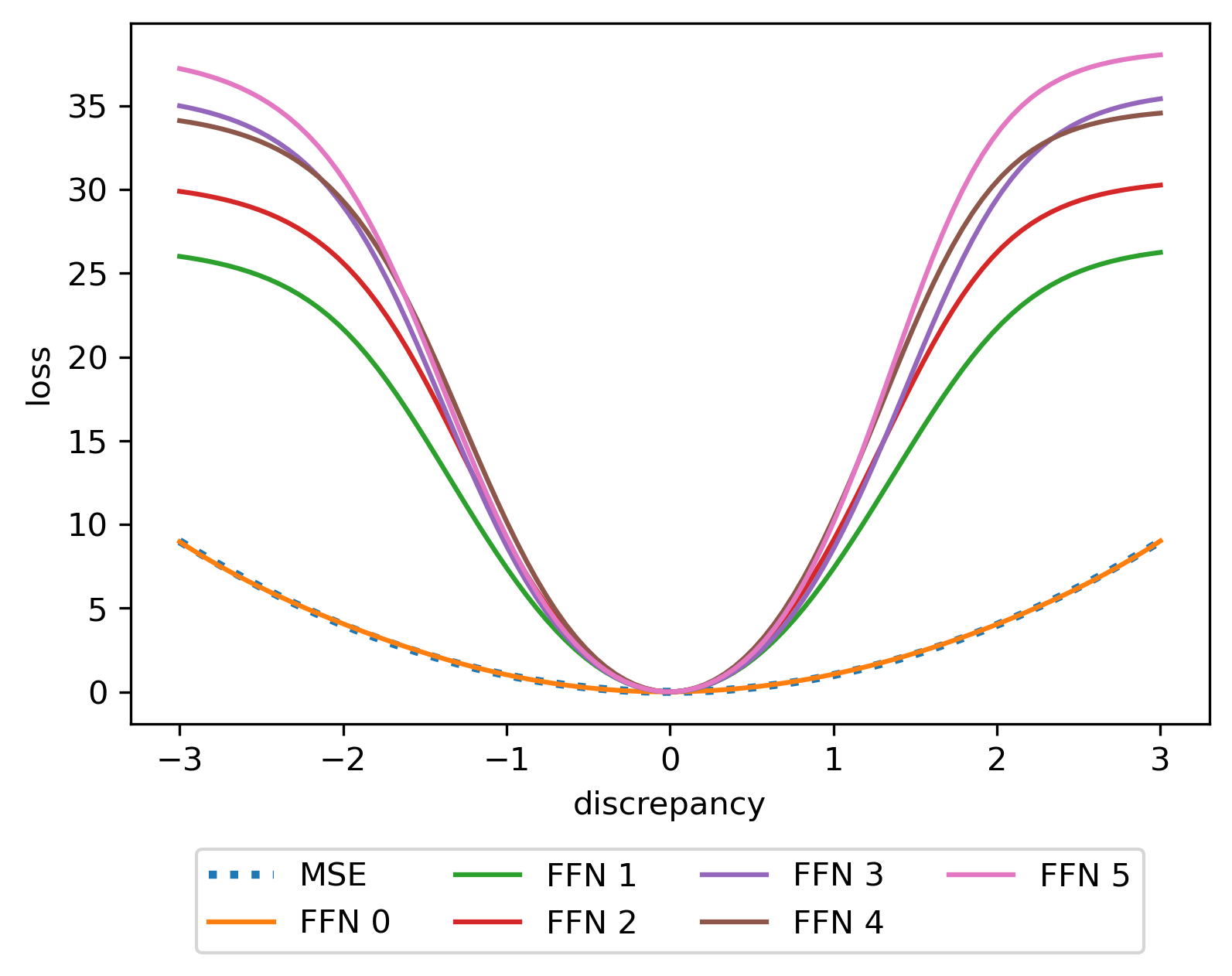}
		\caption{}
		\label{fig:rd:snaps:loss:ffn:reg}
	\end{subfigure}
	\caption{Reaction-diffusion equation: Learned loss snapshots captured during meta-training (distributed evenly in $10{,}000$ outer iterations and 0 corresponds to initialization), without (a) and with regularization (b) via the penalty of Eq.~\eqref{eq:lossml:outer:loss:penalty}, with FFN parametrization.
		Learned losses corresponding to SGD with FFN and no regularization are shifted to the right, i.e., their first-order derivative at $0$ discrepancy is not zero.
		Theory-driven regularization as discussed in Section~\ref{sec:meta:design:properties} fixes this issue.}
	\label{fig:rd:snaps:loss:ffn}
\end{figure}

\paragraph*{Meta-testing.}
For evaluating the performance of the captured learned loss snapshots, we employ them for meta-testing on $5$ OOD tasks and compare with standard loss functions from Table~\ref{tab:losses}.
Specifically, we train with SGD for $20{,}000$ iterations with learning rate $0.01$ (same as in meta-training) $5$ tasks using learned and standard loss functions and record the rl2 error on $2{,}500$ exact solution datapoints. 
The OOD test tasks are drawn based on the parameter limits shown in Table~\ref{tab:rd:params}. 
In addition, to increase the difficulty of OOD meta-testing, we also draw random architectures to be used in meta-testing that are different than the meta-training architecture; the number of hidden layers is drawn from $\mathcal{U}_{[2, 5]}$ and the number of neurons in each layer from $\mathcal{U}_{[15, 55]}$.
A learned loss that performs equally well for architectures not used in meta-training is desirable if, for example, we seek to optimize the PINN architecture with a fixed learned loss.

The test error histories as well as the OAL parameters during meta-training are shown in Fig.~\ref{fig:rd:test:misc}, and the minimum rl2 error results are shown in Fig.~\ref{fig:rd:test:min:stats}. 
As shown in Fig.~\ref{fig:rd:test:oal:params}, the online adaptive losses OAL converge to different loss functions for each task, whereas LAL provides a shared learned loss for all tasks. 
In Fig.~\ref{fig:rd:test:min:stats} we see that the loss functions learned with the FFN parametrization achieve an average minimum rl2 error during $20{,}000$ iterations that is significantly smaller than most considered standard losses although they have been meta-trained with only $20$ iterations, with a different PINN architecture and on a different task distribution. 
On the other hand, LAL does not generalize well.
This is attributed to the fact that, for this example, potentially a task-specific loss would be more appropriate (as suggested by Fig.~\ref{fig:rd:test:oal:params}), and the LAL parametrization is not flexible enough to provide a shared loss function that performs well across tasks (as opposed to FFN).
\begin{figure}[H]
	\centering
	\begin{subfigure}[t]{0.48\textwidth}
		\centering
		\includegraphics[width=\textwidth]{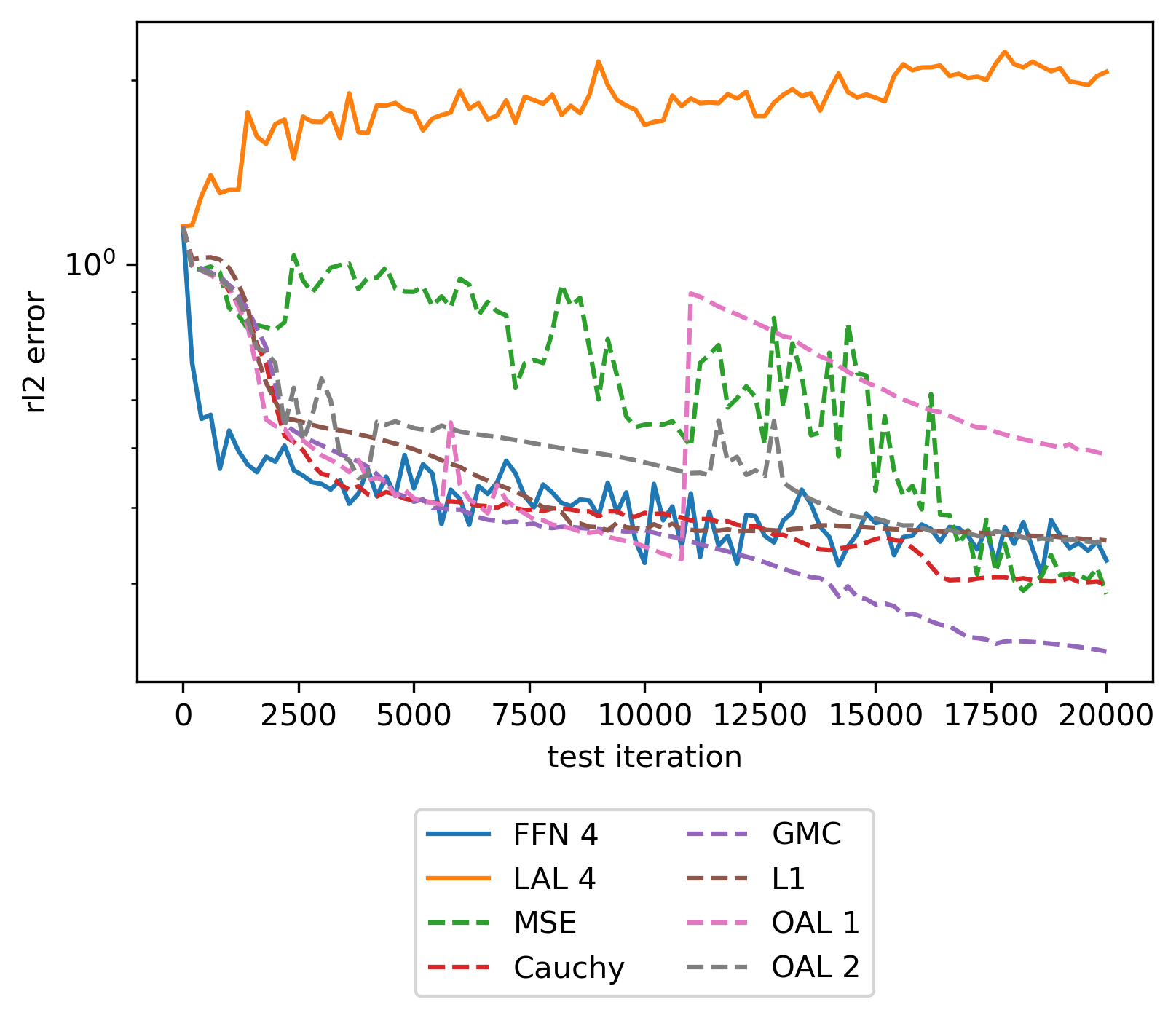}
		\caption{}
		\label{fig:rd:test:train:hists}
	\end{subfigure}
	\hfill
	\begin{subfigure}[t]{0.48\textwidth}
		\centering
		\includegraphics[width=\textwidth]{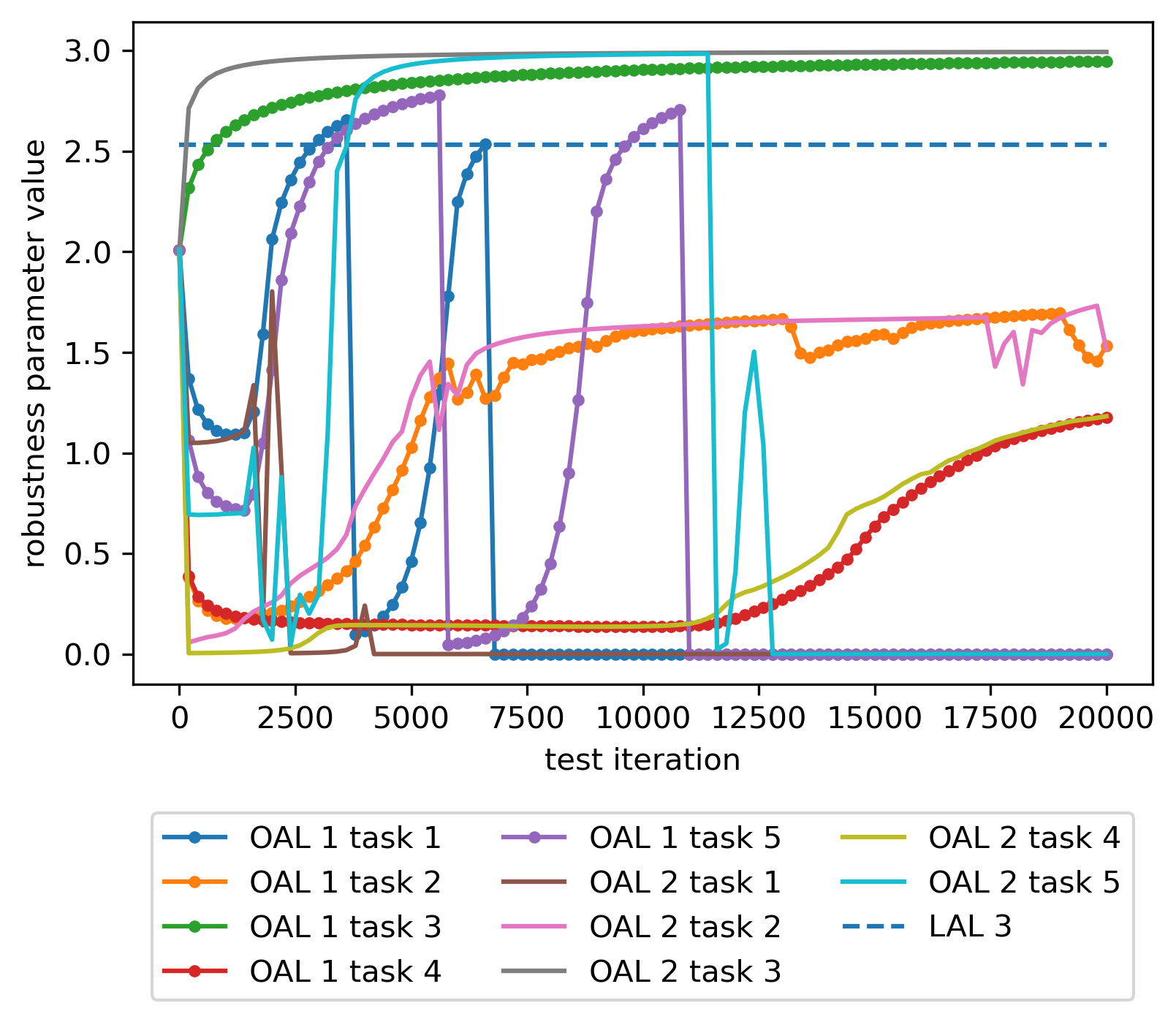}
		\caption{}
		\label{fig:rd:test:oal:params}
	\end{subfigure}
	\caption{Reaction-diffusion equation: Out-of-distribution meta-testing results obtained using SGD with learning rate $0.01$ for $20{,}000$ iterations.
		(a) shows the meta-testing relative $\ell_2$ test error (rl2) trajectories for all standard loss functions considered and for selected learned loss snapshots captured during meta-training.
		(b) shows the robustness parameter trajectories for online adaptive loss functions OAL 1 and OAL 2, with loss-specific learning rates $0.01$ and $0.1$, respectively (see Section~\ref{sec:meta:design:param:LAL}); comparison with LAL 3 also included.
		As shown in (a), the FFN learned loss performs better than most considered standard and adaptive losses, whereas LAL does not generalize well.
		Furthermore, as shown in (b), the online adaptive losses OAL converge to different loss functions for each task, whereas LAL provides a shared learned loss for all tasks. 
	}
	\label{fig:rd:test:misc}
\end{figure}
\begin{figure}[H]
	\centering
	\includegraphics[width=.7\linewidth]{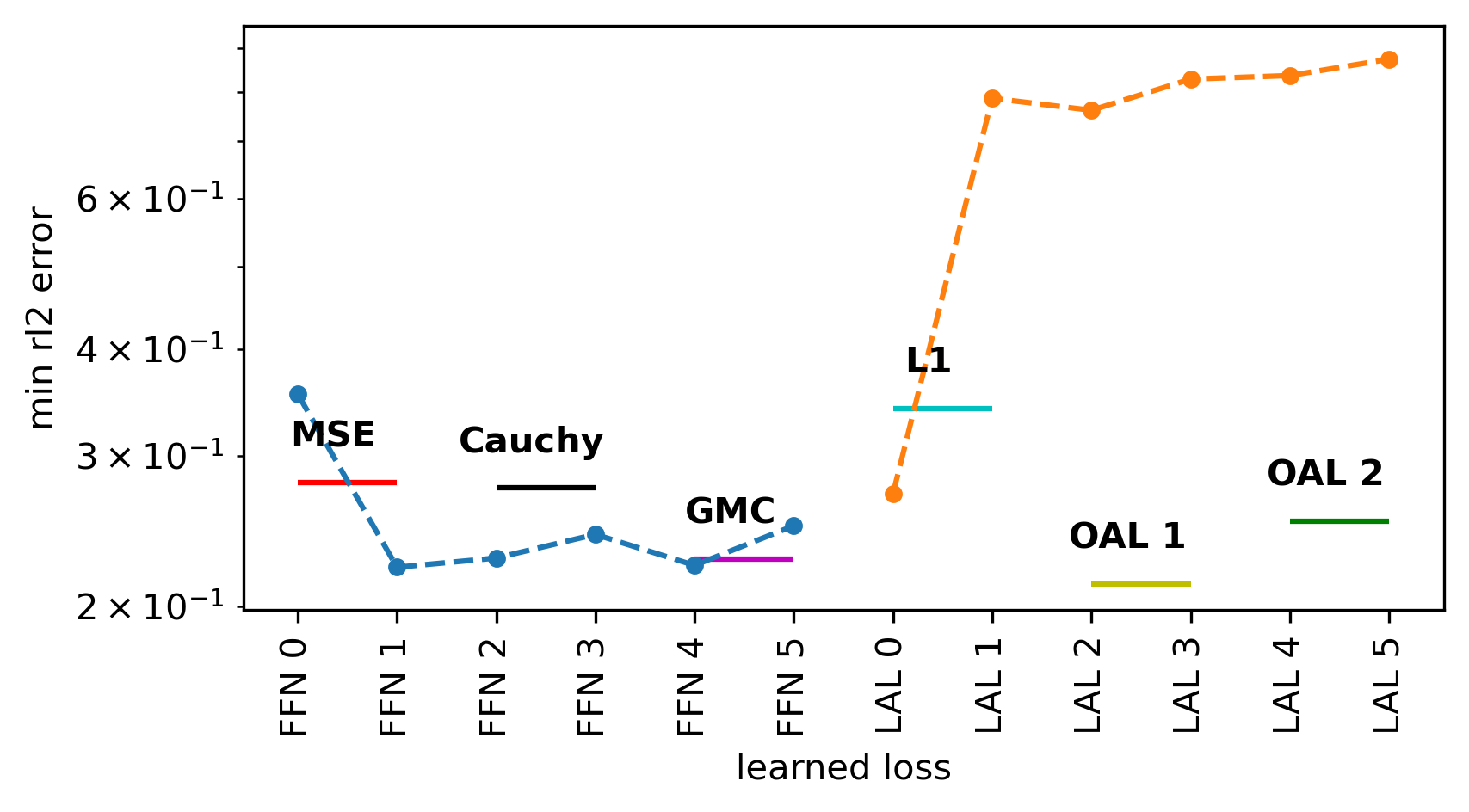}  
	\caption{Reaction-diffusion equation: Minimum relative test $\ell_2$ error (rl2) averaged over $5$ OOD tasks during meta-testing with SGD for $20{,}000$ iterations.
		Learned loss snapshots (FFN 0-6 and LAL 0-6) are compared with standard loss functions of Table~\ref{tab:losses} and with online adaptive loss functions OAL 1 and OAL 2 (2 loss-specific learning rates).
		Although FFN learned losses perform better than most standard losses, LAL does not generalize well.
		This is attributed to the fact that, for this example, potentially a task-specific loss would be more appropriate (as suggested by Fig.~\ref{fig:rd:test:oal:params}), and the LAL parametrization is not flexible enough to provide a shared loss function that performs well across tasks (as opposed to FFN).
	}
	\label{fig:rd:test:min:stats}
\end{figure}

\subsection{Task distributions defined based on Burgers equation with varying viscocity} 
\label{sec:examples:burgers}

Finally, we consider the Burgers equation defined by
\begin{align}\label{eq:burgers}
	&\partial_t u + u \partial_x u = \lambda \partial^2_x u, ~ x \in [-1, 1], t \in [0, 1],\\
	&u(x, 0) = -\sin(\pi x), u(-1, t) = u(1, t) = 0,
\end{align}
where $u$ denotes the flow velocity and $\lambda$ the viscosity of the fluid. 
From a function approximation point of view, solutions corresponding to values of $\lambda$ very close to zero (e.g., $\lambda \approx 10^{-3}$), are expected to have some common characteristics such as steep gradients after some time $t$. 
This fact justifies defining a task distribution comprised of PDEs of the form of Eq.~\eqref{eq:burgers} with, indicatively, $\lambda < 2\times 10^{-3}$ and with the same ICs/BCs.  
A similar explanation can be given for smoother solutions corresponding to values of $\lambda > 10^{-1}$, for example.  

\paragraph*{Meta-training.}
For meta-training, we use two regimes for $\lambda$ as shown in Table~\ref{tab:burgers:params}.
The design options are the same as in Section~\ref{sec:examples:ad} except for the fact that only SGD is used in this example as an inner optimizer candidate.
The number of datapoints used for evaluating the outer objective of Eq.~\eqref{eq:lossml:outer:loss} is shown in Table~\ref{tab:burgers:outer:data}, whereas the corresponding number for the inner objective of Eq.~\eqref{eq:lossml:inner:loss} is the same as the single-data case in Table~\ref{tab:burgers:outer:data}.
Furthermore, the PINN architecture consists of $3$ hidden layers with $20$ neurons each and $tanh$ activation function.

In Fig.~\ref{fig:burgers:comp} we show the final loss functions (FFN and LAL parametrizations) as obtained with SGD as inner optimizer with single data with and without objective function weights meta-learning, and with double data.
The corresponding objective function weights trajectories are shown in Fig.~\ref{fig:burgers:params}.
The learned losses with single data have steeper derivatives because they lack the objective function weights, which are shown in Fig.~\ref{fig:burgers:params} to be greater than 1.
Furthermore, for single data the loss functions obtained for the two regimes are slightly different when objective function weights are not meta-learned but almost identical when they are. 
This means that the meta-learning algorithm compensates for the difference in the two regimes by yielding the same learned loss with different balancing of the PDE, BCs, and ICs terms in the PINNs objective function.
On the other hand, when solution data is used in the outer objective (available via the analytical solution), the obtained loss functions are highly different; see also Fig.~\ref{fig:burgers:snaps:loss:ffn} for the loss function snapshots captured during meta-training for both regimes and for the single and double data cases with the FFN parametrization.

\paragraph*{Connection with theory Section~\ref{sec:meta:design:properties}.}

Finally, see Fig.~\ref{fig:burgers:trainloss:ffn} for the outer objective trajectories for both single and double data training with the FFN parametrization, as well as Fig.~\ref{fig:burgers:traintest} for the the corresponding test performance on $5$ unseen tasks while meta-training for $20$ iterations.
It is shown in Fig.~\ref{fig:burgers:trainloss:ffn:dd:r1} that the outer objective drops significantly during training for double data and regime 1; this can be construed as \textit{meta-training error}.
Furthermore, for the same case it is shown in Fig.~\ref{fig:burgers:traintest:dd} that rl2 drops significantly during training; this can be construed as \textit{meta-validation error}.
However, the learned loss obtained during this training with no regularization does not satisfy \nc~of Section~\ref{sec:meta:design:properties}; see Figs.~\ref{fig:burgers:comp}, \ref{fig:burgers:snaps:loss:ffn} and notice that the stationary point at 0 is not a global minimum.
In line with our theoretical results of Section~\ref{sec:meta:design:properties}, these learned losses have also been found in our experiments to lead to divergence if used for full meta-testing ($20{,}000$ iterations), i.e., they do not generalize well although meta-training and meta-validation performance is satisfactory.
Finally, we also include in Fig.~\ref{fig:rd:snaps:loss:ffn} the regularized loss functions obtained via meta-training with the theoretically-driven gradient penalty of Eq.~\eqref{fig:burgers:final:ffn:dd}, which solves this issue with the FFN parametrization and double data meta-training.

\begin{table}[H]
	\centering
	\caption{Burgers equation: Task distribution values used in meta-training and OOD meta-testing for both task regimes r1 and r2. 
	}
	\begin{tabular}{ccccc}
		\toprule
		&r1 $\lambda_{min}$& r1 $\lambda_{max}$& r2 $\lambda_{min}$& r2 $\lambda_{max}$\\
		\midrule
		meta-training&$10^{-3}$& $2\times 10^{-3}$& $10^{-1}$& $1$\\
		OOD meta-testing&$10^{-3}$& $10^{-2}$& $10^{-2}$& $2$\\
		\bottomrule
	\end{tabular}
	\label{tab:burgers:params}
\end{table}
\begin{table}[H]
	\centering
	\caption{Burgers equation: PDE, BCs, ICs, and solution data considered in the outer objective of Eq.~\eqref{eq:lossml:outer:loss} for single data (outer objective data same as inner objective) and for double data (solution data considered available in the outer objective) in meta-training ($\{N_{f, val}, N_{b, val}, N_{u_0, val}, N_{u, val}\}$), as well as for OOD meta-testing ($\{N_{f}, N_{b}, N_{u_0}, N_{u}\}$).
	}
	\begin{tabular}{c|cccc}
		\toprule
		& $N_{f(, val)}$     & $N_{b(, val)}$     & $N_{u_0(, val)}$ & $N_{u(, val)}$ \\
		single data (SD) meta-training & $1{,}000$ & $200$ & $100$ & $0$ \\
		double data (DD) meta-training & $0$ & $0$ & $0$ & $10{,}000$ \\
		OOD meta-testing & $2{,}000$ & $200$ & $100$ & $0$ \\
		\bottomrule
	\end{tabular}
	\label{tab:burgers:outer:data}
\end{table}
\begin{figure}[H]
	\centering
	\begin{subfigure}[t]{0.48\textwidth}
		\centering
		\includegraphics[width=\textwidth]{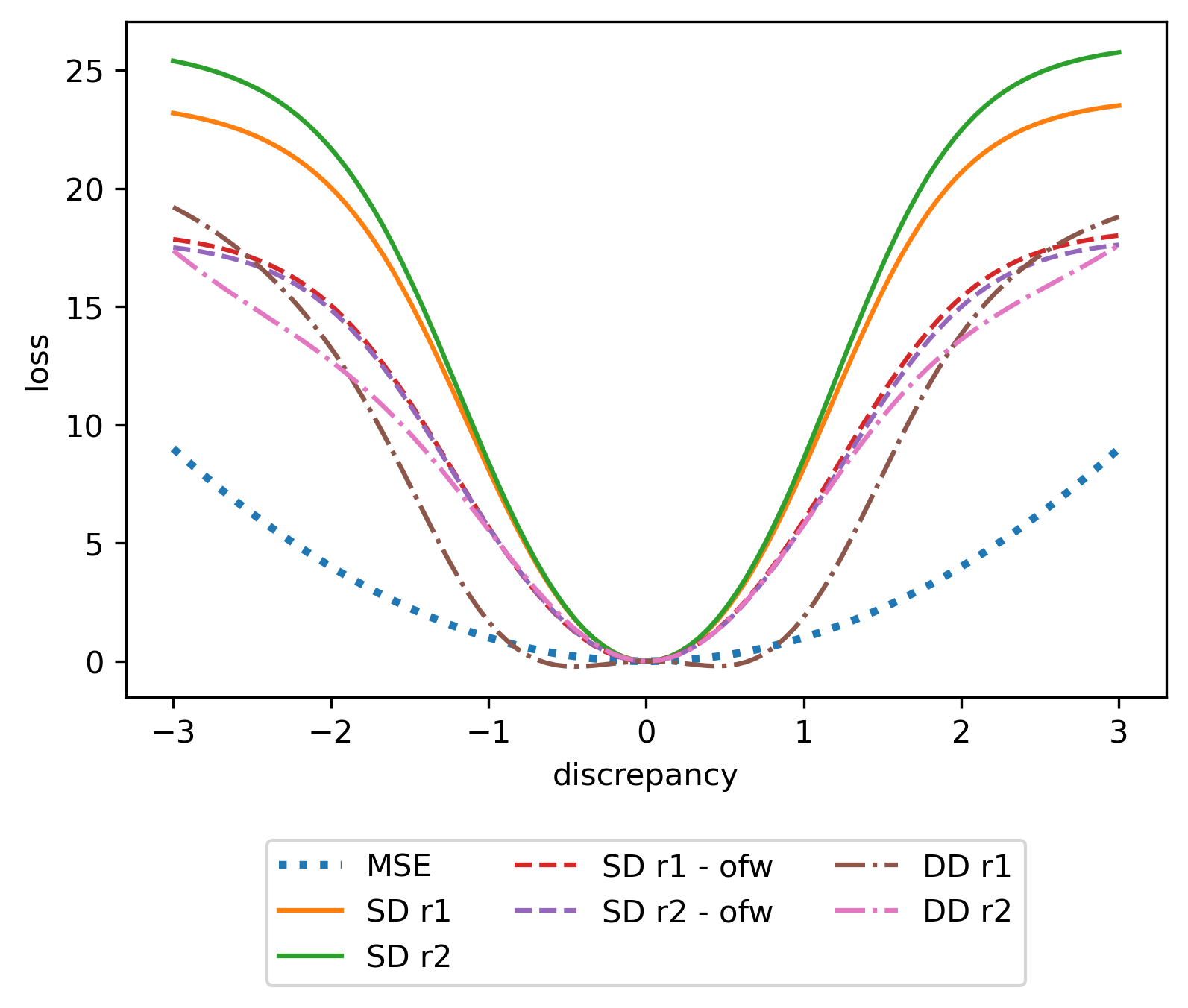}
		\caption{}
		\label{fig:burgers:comp:ffn:loss}
	\end{subfigure}
	\hfill
	\begin{subfigure}[t]{0.48\textwidth}
		\centering
		\includegraphics[width=\textwidth]{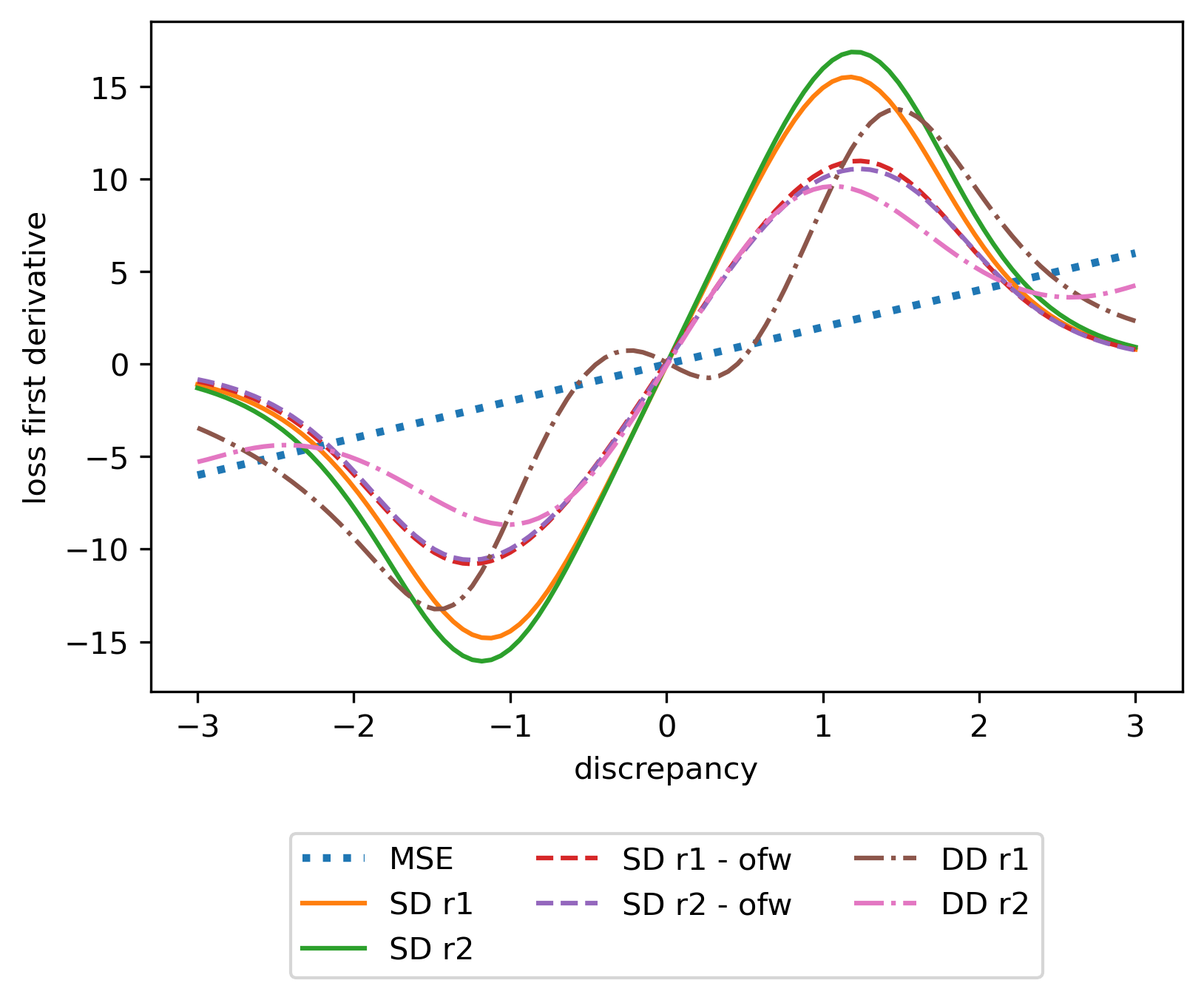}
		\caption{}
		\label{fig:burgers:comp:ffn:lossder}
	\end{subfigure}
	\hfill
	\centering
	\begin{subfigure}[t]{0.48\textwidth}
		\centering
		\includegraphics[width=\textwidth]{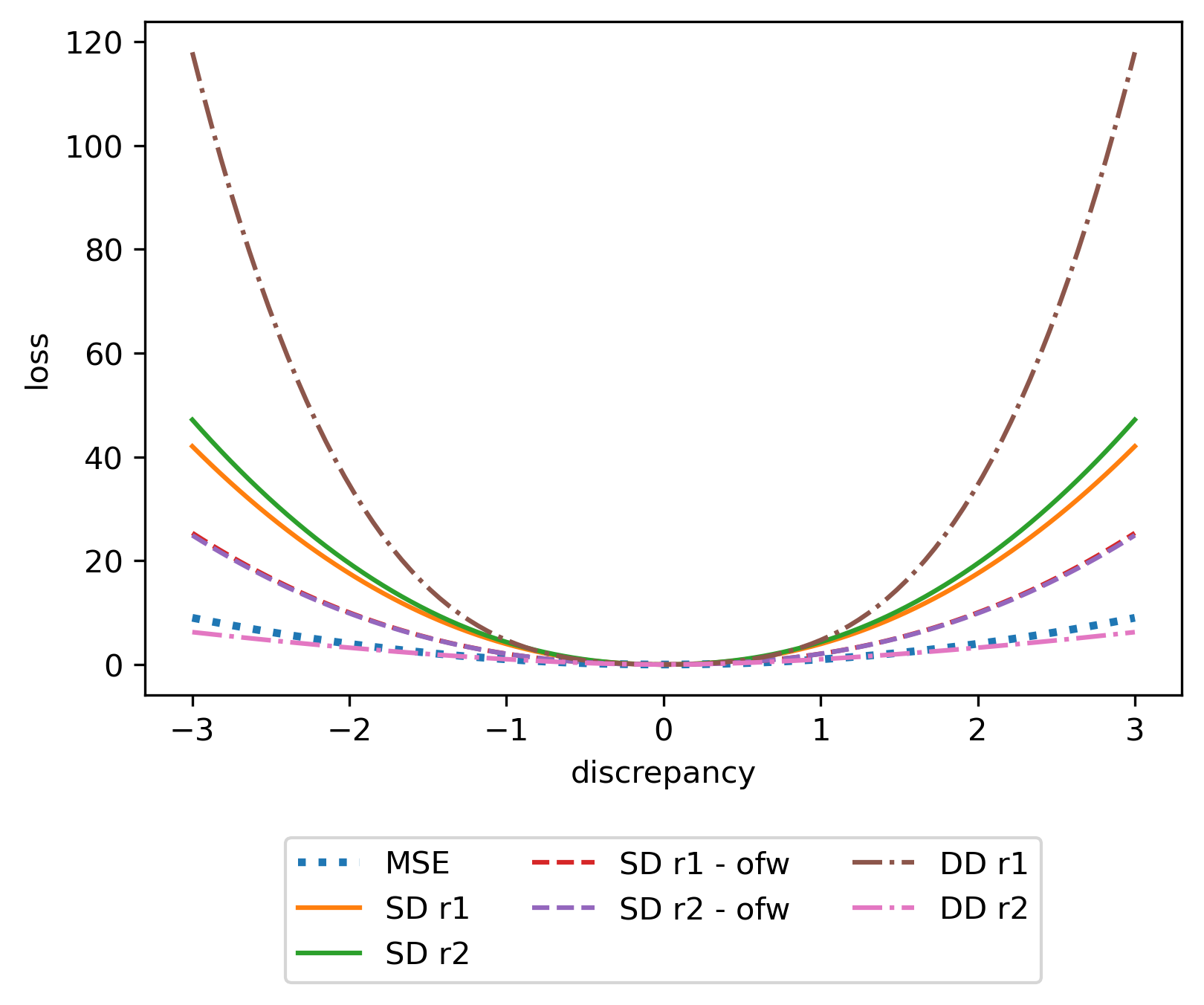}
		\caption{}
		\label{fig:burgers:comp:lal:loss}
	\end{subfigure}
	\hfill
	\begin{subfigure}[t]{0.48\textwidth}
		\centering
		\includegraphics[width=\textwidth]{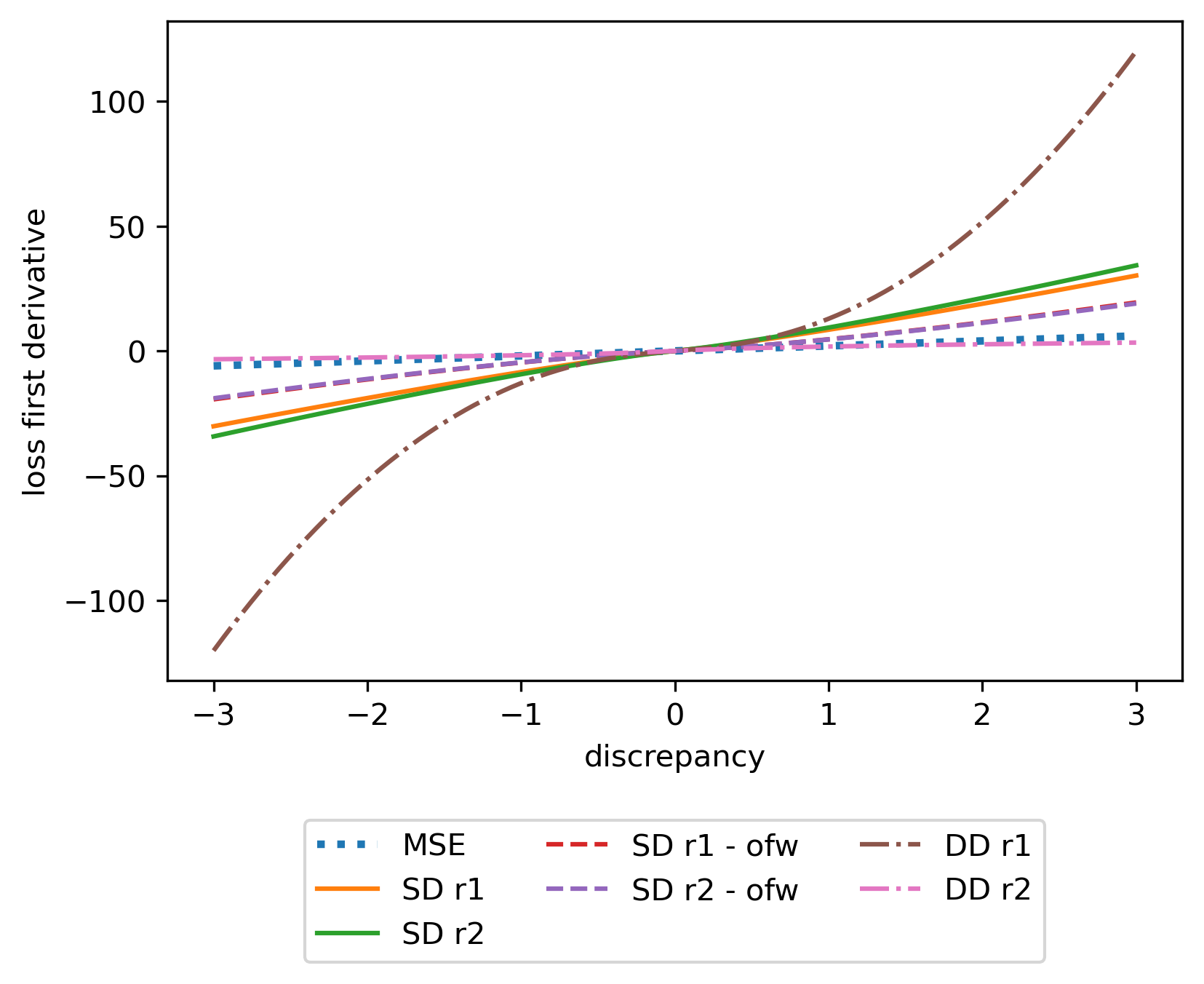}
		\caption{}
		\label{fig:burgers:comp:lal:lossder}
	\end{subfigure}
	\caption{Burgers equation: Final learned losses (a, c) and corresponding first-order derivatives (b, d), with FFN (a, b) and LAL (c, d) parametrizations for task regimes $1$ and $2$ (r1 and r2). 
		Results obtained via meta-training with SGD as inner optimizer (with single data, with and without objective function weights meta-learning, and with double data; SD r1-r2, SD r1-r2 - ofw, DD r1-r2); comparisons with MSE are also included.
		Whereas regularization is not required for LAL (Proposition~\ref{prop:1}), the final learned loss with FFN parametrization and double data (DD r1) does not satisfy \nc~(notice that the stationary point at 0 is not a global minimum).
		Theory-driven regularization as discussed in Section~\ref{sec:meta:design:properties} fixes this issue; see Fig.~\ref{fig:burgers:final:ffn:dd}.
		Furthermore, learned losses with single data (SD r1-r2) have steeper derivatives because they lack the objective function weights, which are shown in Fig.~\ref{fig:burgers:params} to be greater than 1.
	}
	\label{fig:burgers:comp}
\end{figure}
\begin{figure}[H]
	\centering
	\begin{subfigure}[t]{0.48\textwidth}
		\centering
		\includegraphics[width=\textwidth]{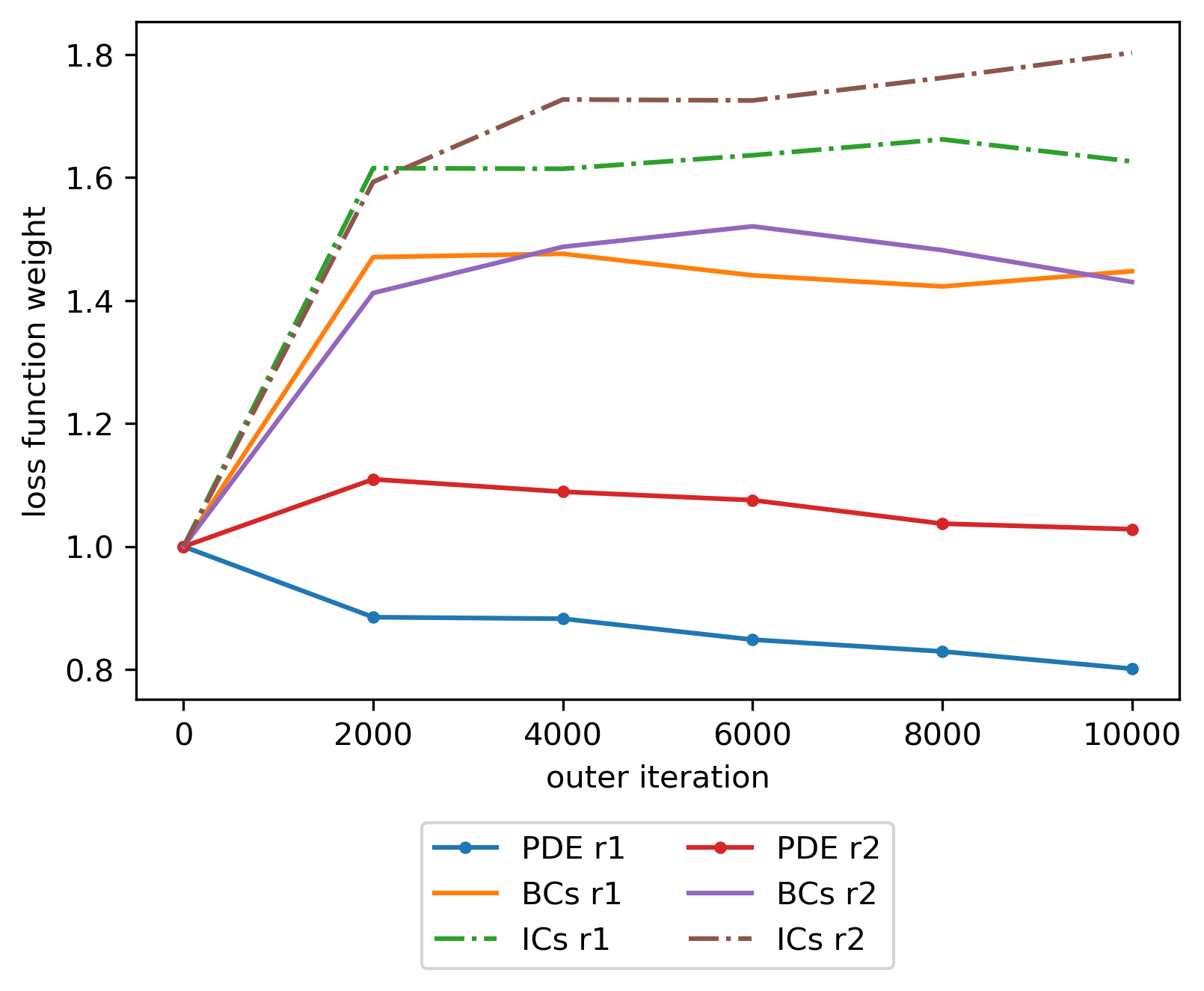}
		\caption{}
		\label{fig:burgers:params:ffn}
	\end{subfigure}
	\hfill
	\begin{subfigure}[t]{0.48\textwidth}
		\centering
		\includegraphics[width=\textwidth]{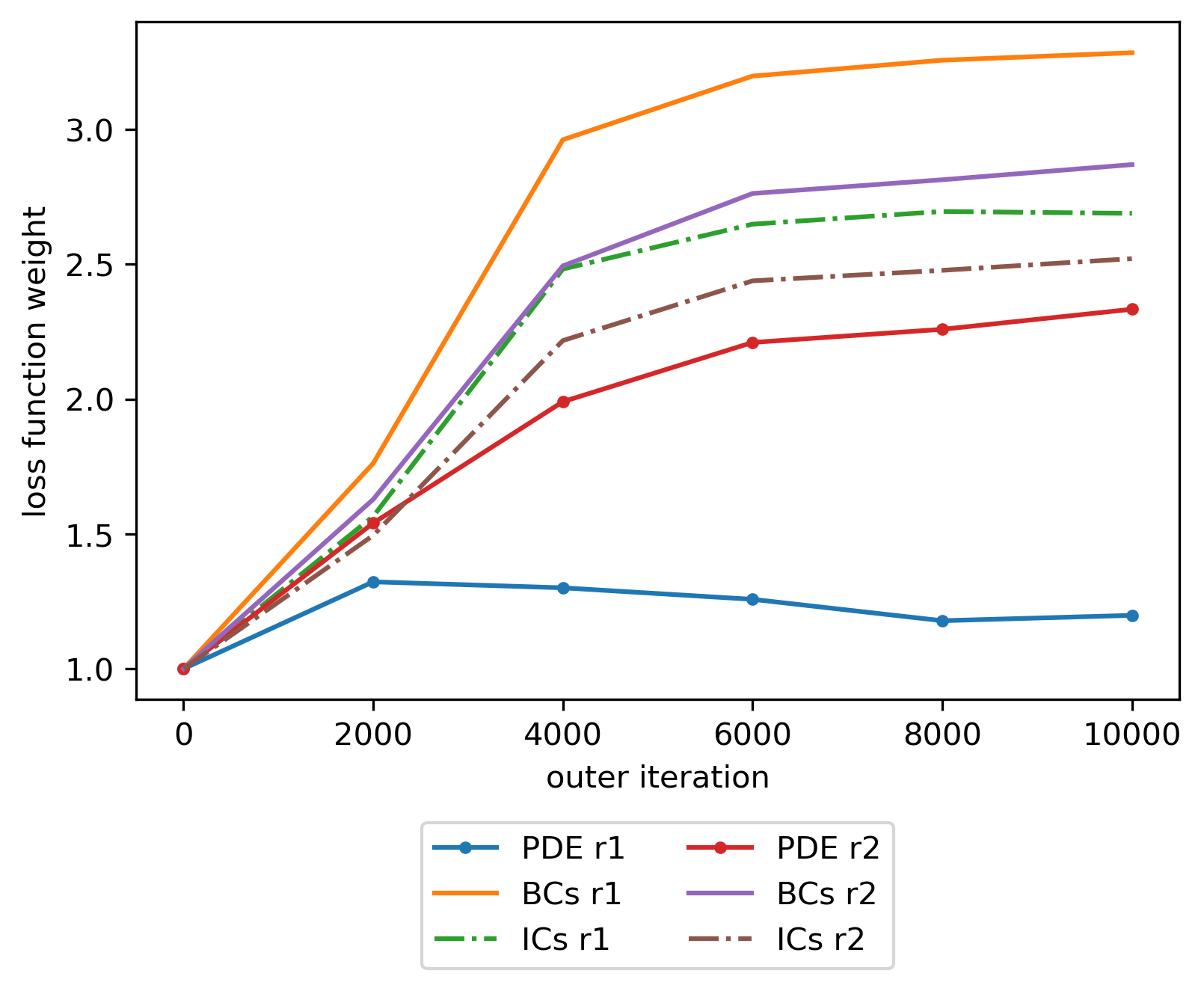}
		\caption{}
		\label{fig:burgers:params:lal}
	\end{subfigure}
	\caption{Burgers equation: Learned objective function weights, pertaining to PDE, BCs, and ICs residuals, as a function of outer iteration in meta-training and task regimes (r1-r2); see Section~\ref{sec:meta:design:param:weights}.
		Results obtained with FFN (a) and LAL (b) parametrizations.
		Most objective function weights increase for both FFN and LAL parametrizations, which translates into learning rate increase, while FFN and LAL disagree on how they balance the various terms.
	}
	\label{fig:burgers:params}
\end{figure}
\begin{figure}[H]
	\centering
	\begin{subfigure}[t]{0.48\textwidth}
		\centering
		\includegraphics[width=\textwidth]{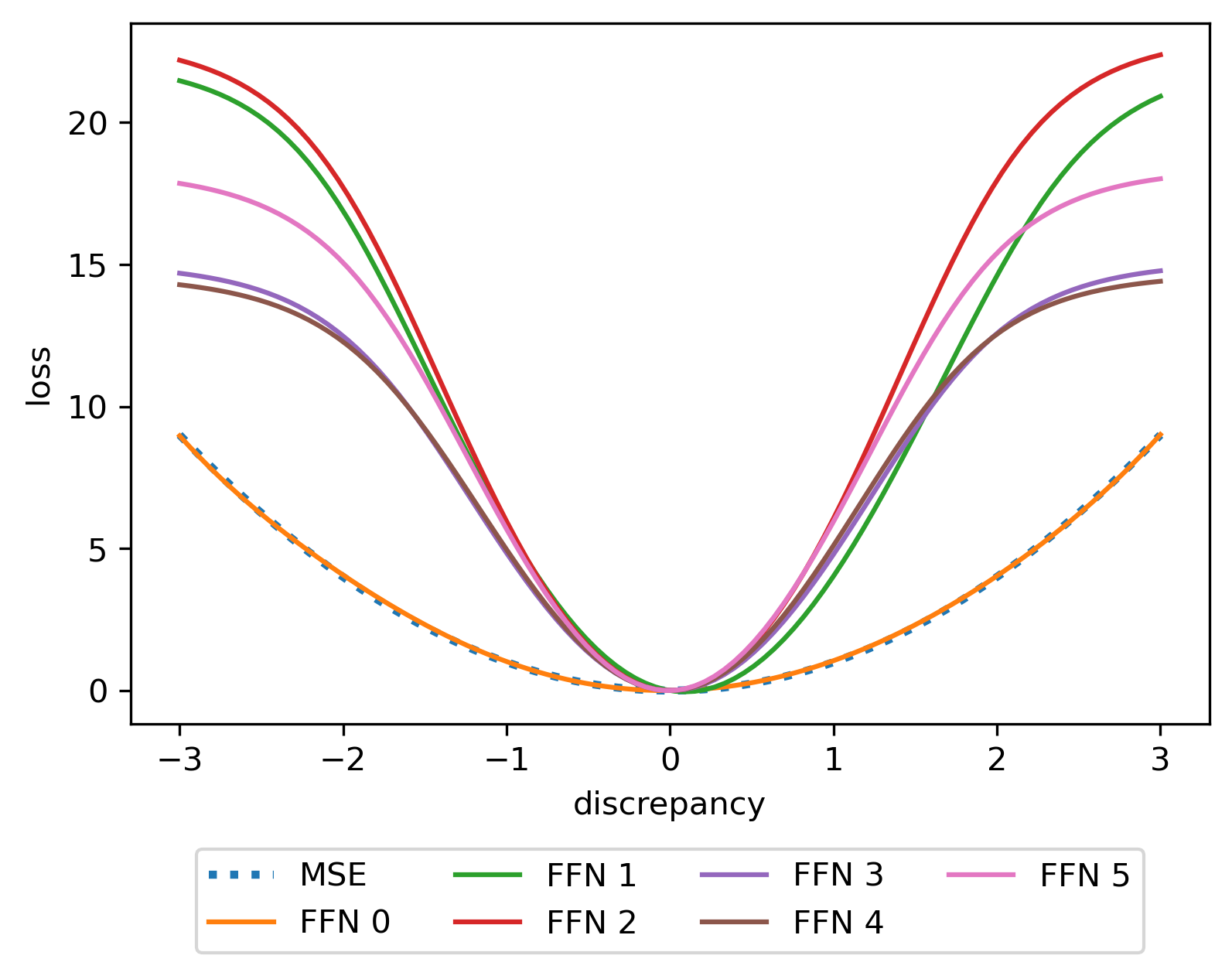}
		\caption{}		
		\label{fig:burgers:snaps:loss:ffn:sdlfw:r1}
	\end{subfigure}
	\hfill
	\begin{subfigure}[t]{0.48\textwidth}
		\centering
		\includegraphics[width=\textwidth]{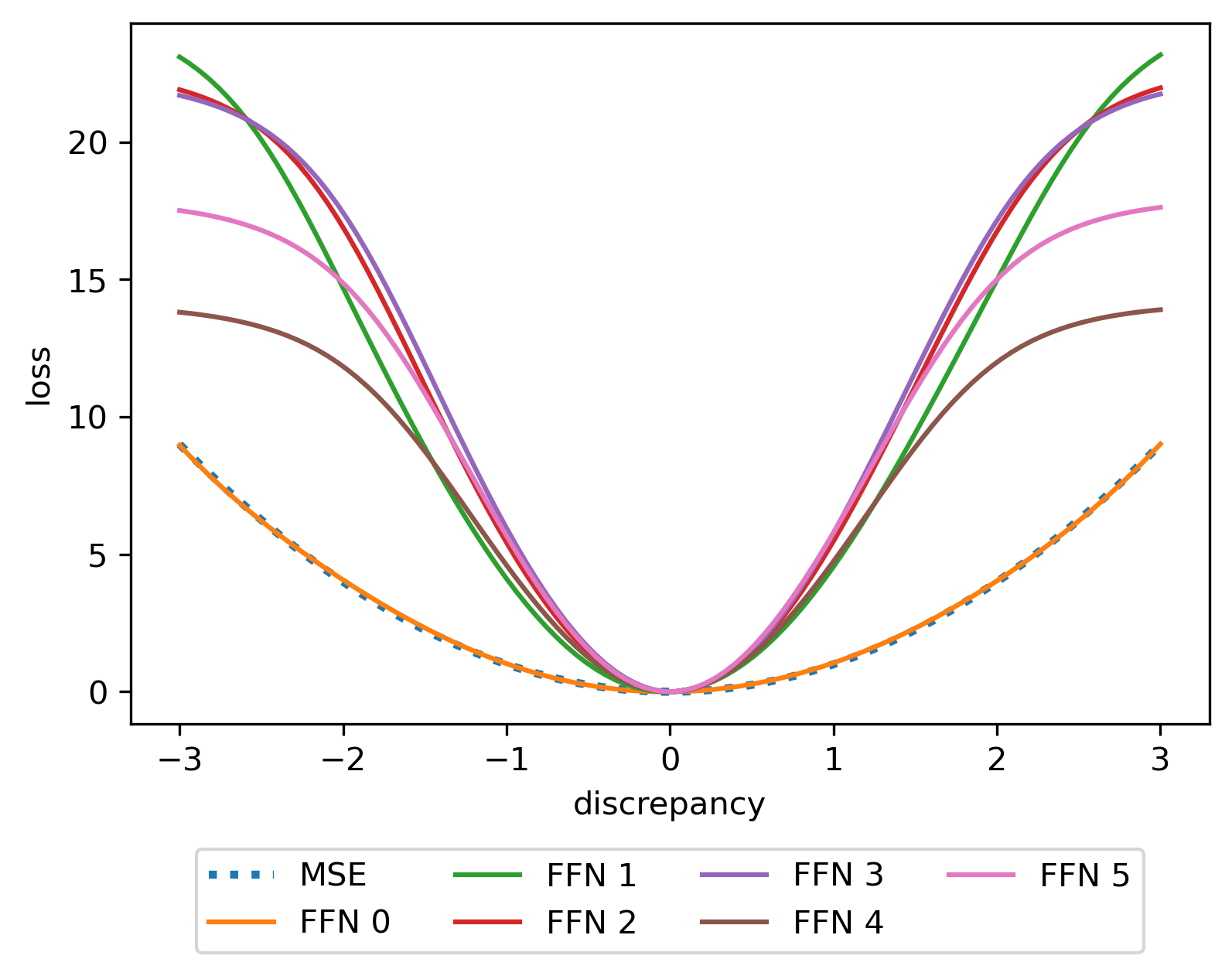}
		\caption{}
		\label{fig:burgers:snaps:loss:ffn:sdlfw:r2}
	\end{subfigure}
	\hfill
	\begin{subfigure}[t]{0.48\textwidth}
		\centering
		\includegraphics[width=\textwidth]{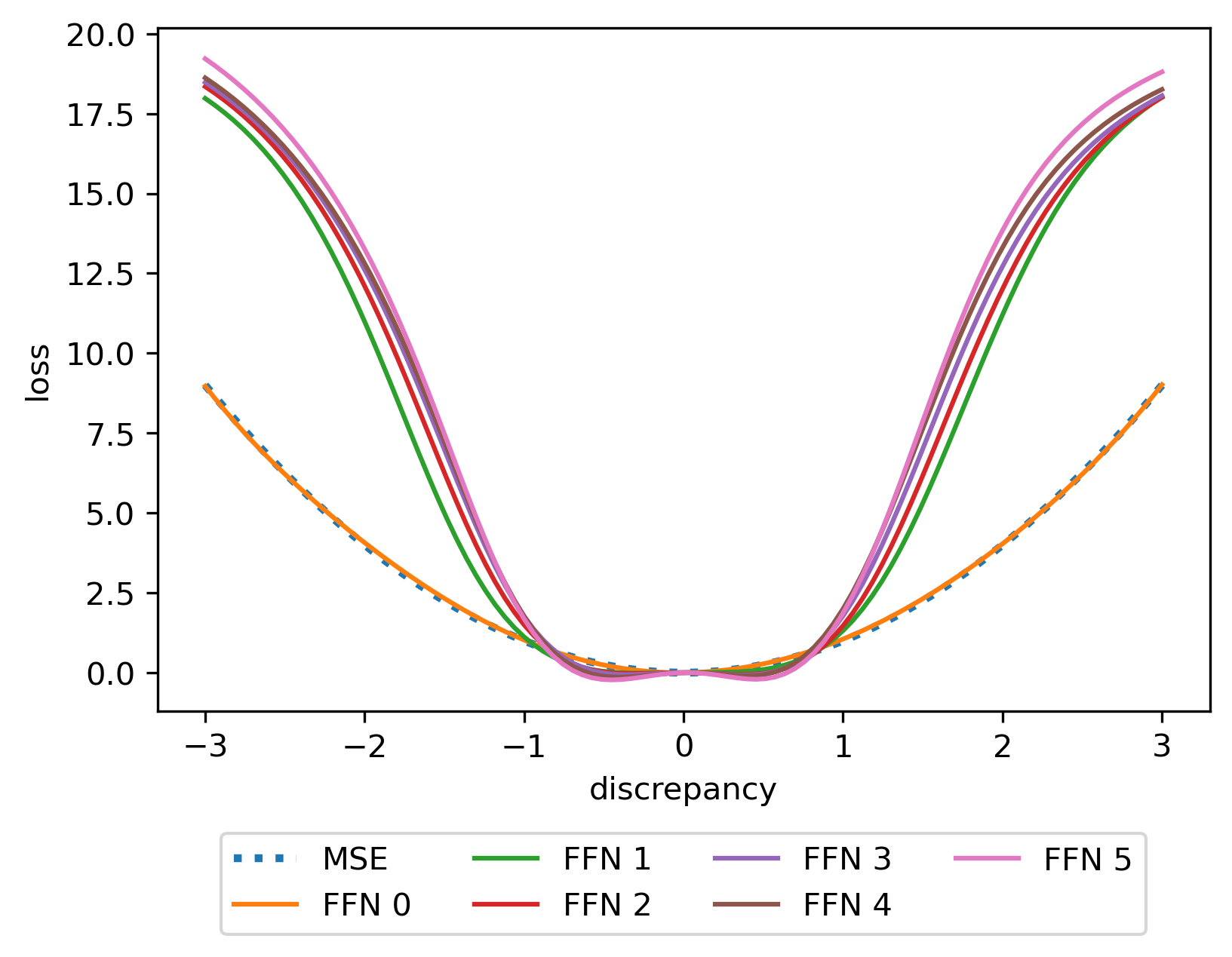}
		\caption{}
		\label{fig:burgers:snaps:loss:ffn:dd:r1}
	\end{subfigure}
	\hfill
	\begin{subfigure}[t]{0.48\textwidth}
		\centering
		\includegraphics[width=\textwidth]{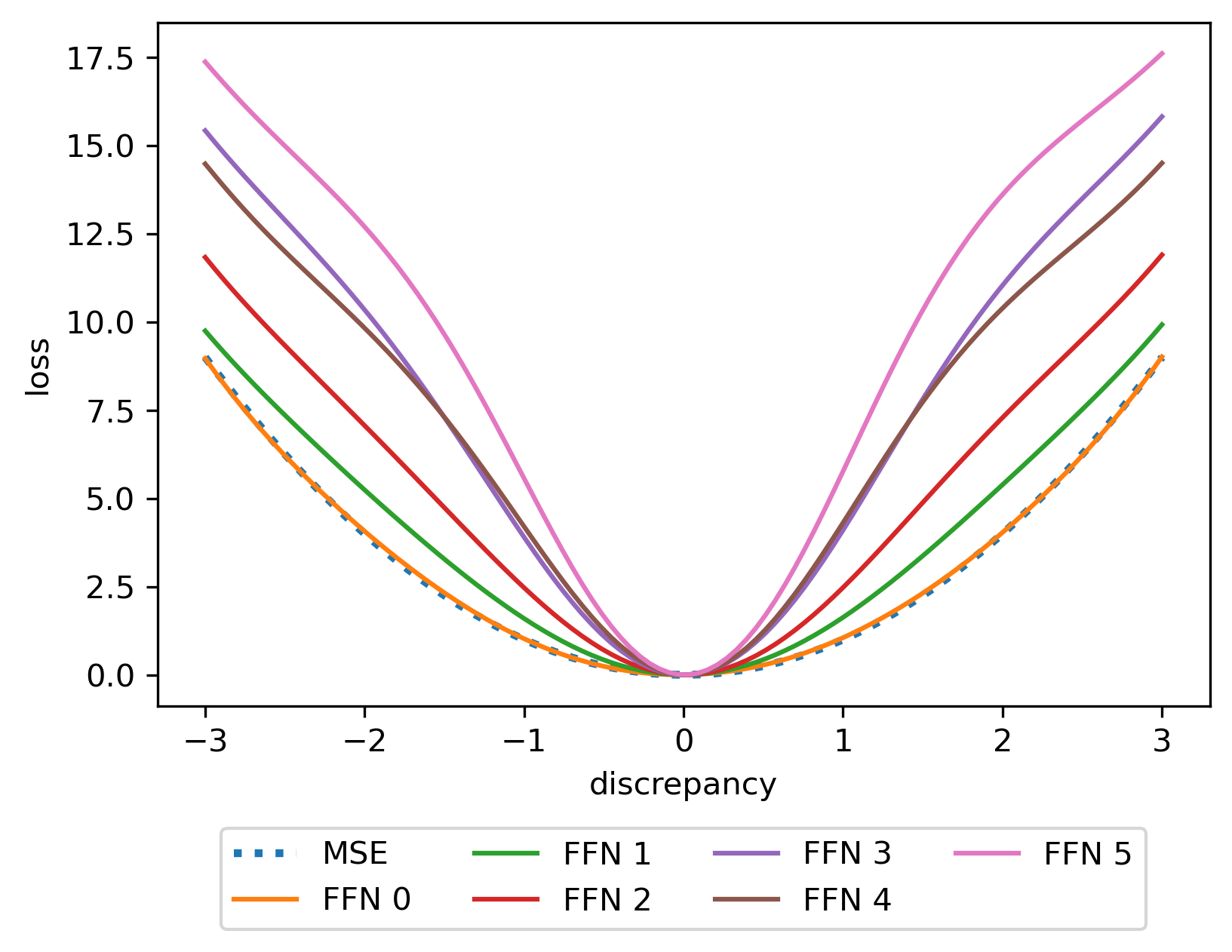}
		\caption{}
		\label{fig:burgers:snaps:loss:ffn:dd:r2}
	\end{subfigure}
	\caption{Burgers equation: Learned loss snapshots captured during meta-training (distributed evenly in $10{,}000$ outer iterations with 0 referring to initialization), with single (a, b) and double data (c, d; see Section~\ref{sec:meta:lossml}).
		FFN parametrization only and both task regimes r1 (a, c) and r2 (b, d) are included.
		The learned losses for single and double data are in general different and the learned losses in part (c) do not satisfy \nc~of Section~\ref{sec:meta:design:properties} (notice that the stationary point at 0 is not a global minimum).
		Theory-driven regularization as discussed in Section~\ref{sec:meta:design:properties} fixes this issue; see Fig.~\ref{fig:burgers:final:ffn:dd}.}
	\label{fig:burgers:snaps:loss:ffn}
\end{figure}
\begin{figure}[H]
	\centering
	\begin{subfigure}[t]{0.48\textwidth}
		\centering
		\includegraphics[width=\textwidth]{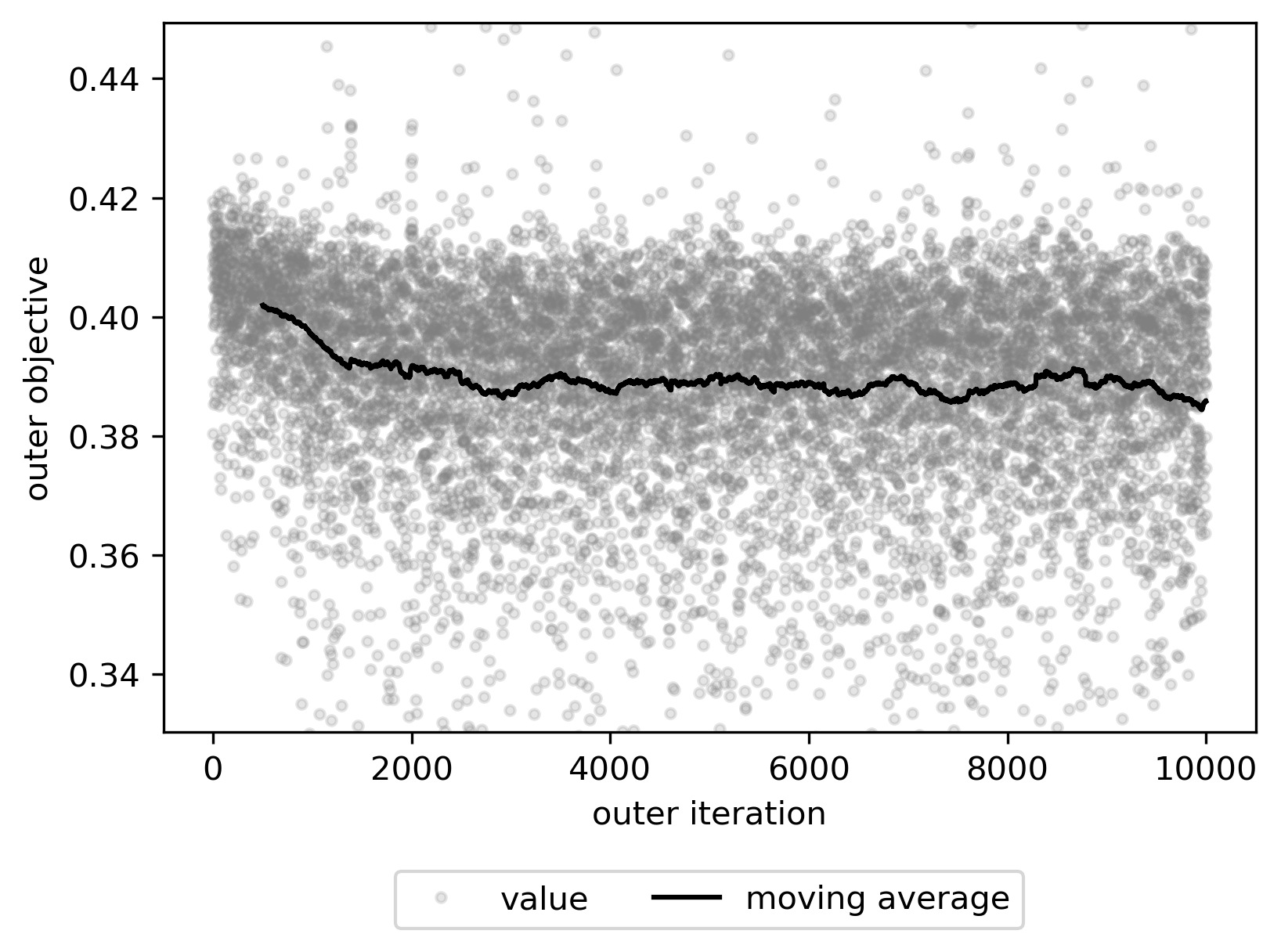}
		\caption{}
		\label{fig:burgers:trainloss:ffn:sdlfw:r1}
	\end{subfigure}
	\hfill
	\begin{subfigure}[t]{0.48\textwidth}
		\centering
		\includegraphics[width=\textwidth]{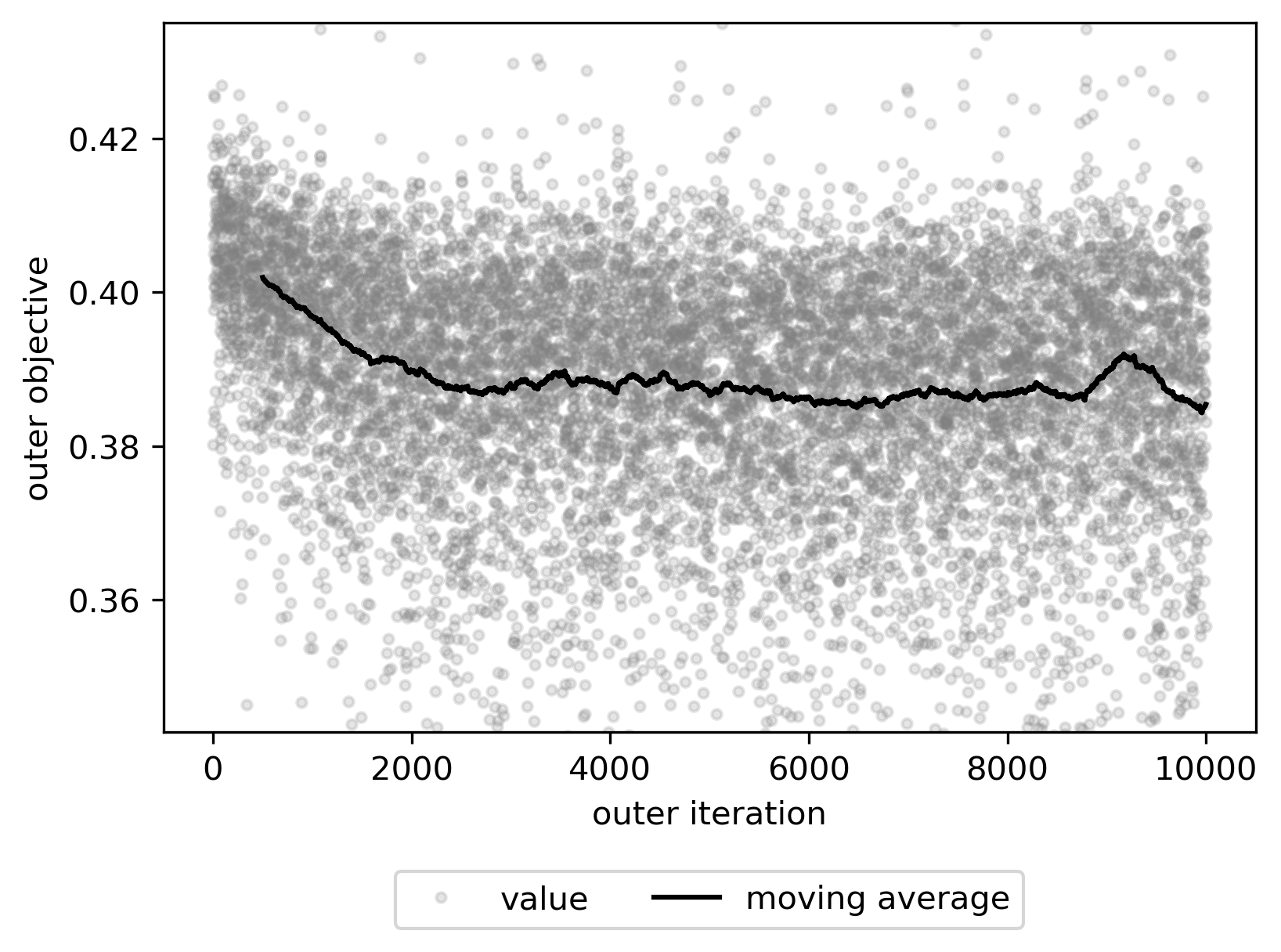}
		\caption{}
		\label{fig:burgers:trainloss:ffn:sdlfw:r2}
	\end{subfigure}
	\hfill
	\begin{subfigure}[t]{0.48\textwidth}
		\centering
		\includegraphics[width=\textwidth]{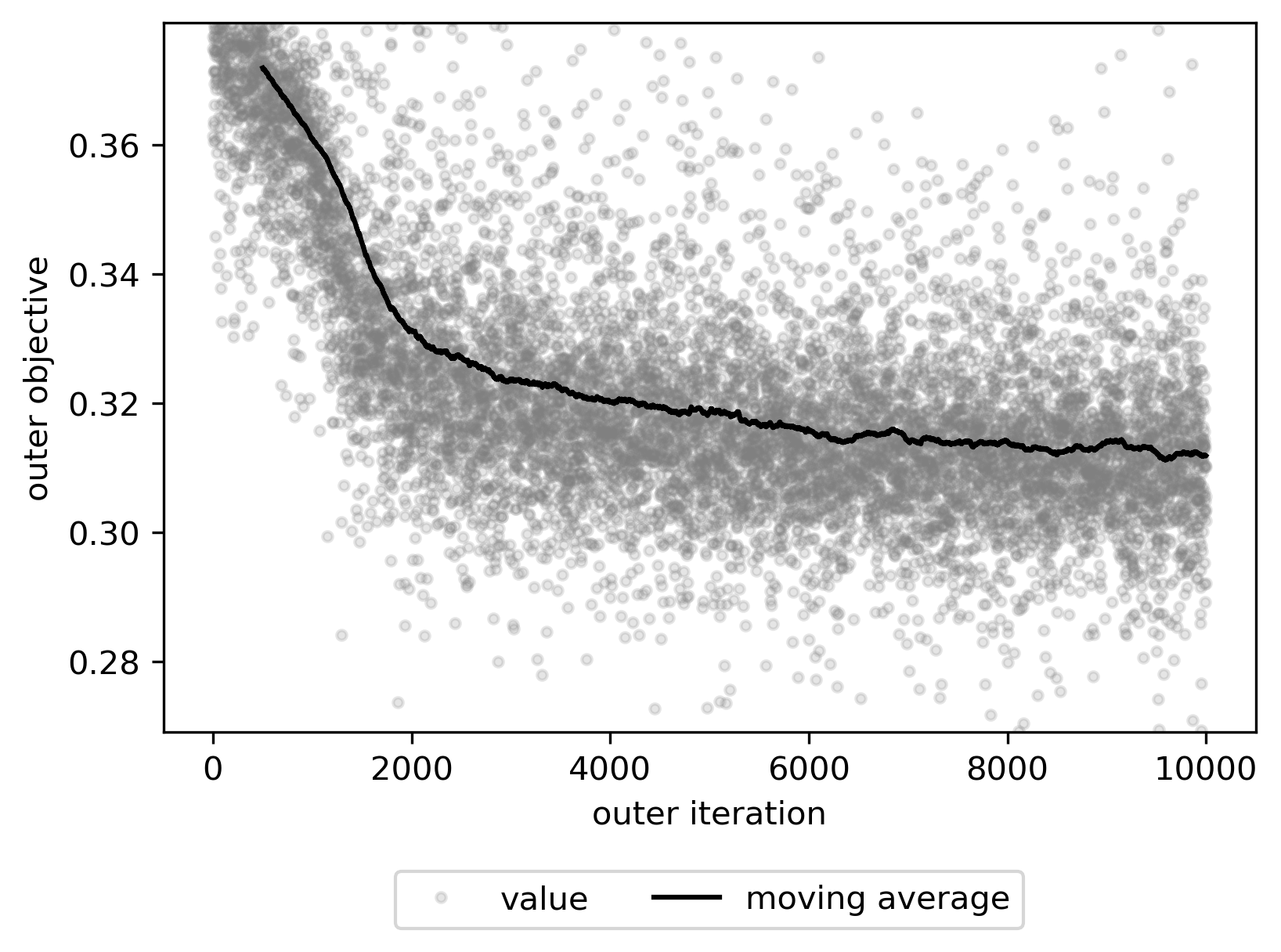}
		\caption{}
		\label{fig:burgers:trainloss:ffn:dd:r1}
	\end{subfigure}
	\hfill
	\begin{subfigure}[t]{0.48\textwidth}
		\centering
		\includegraphics[width=\textwidth]{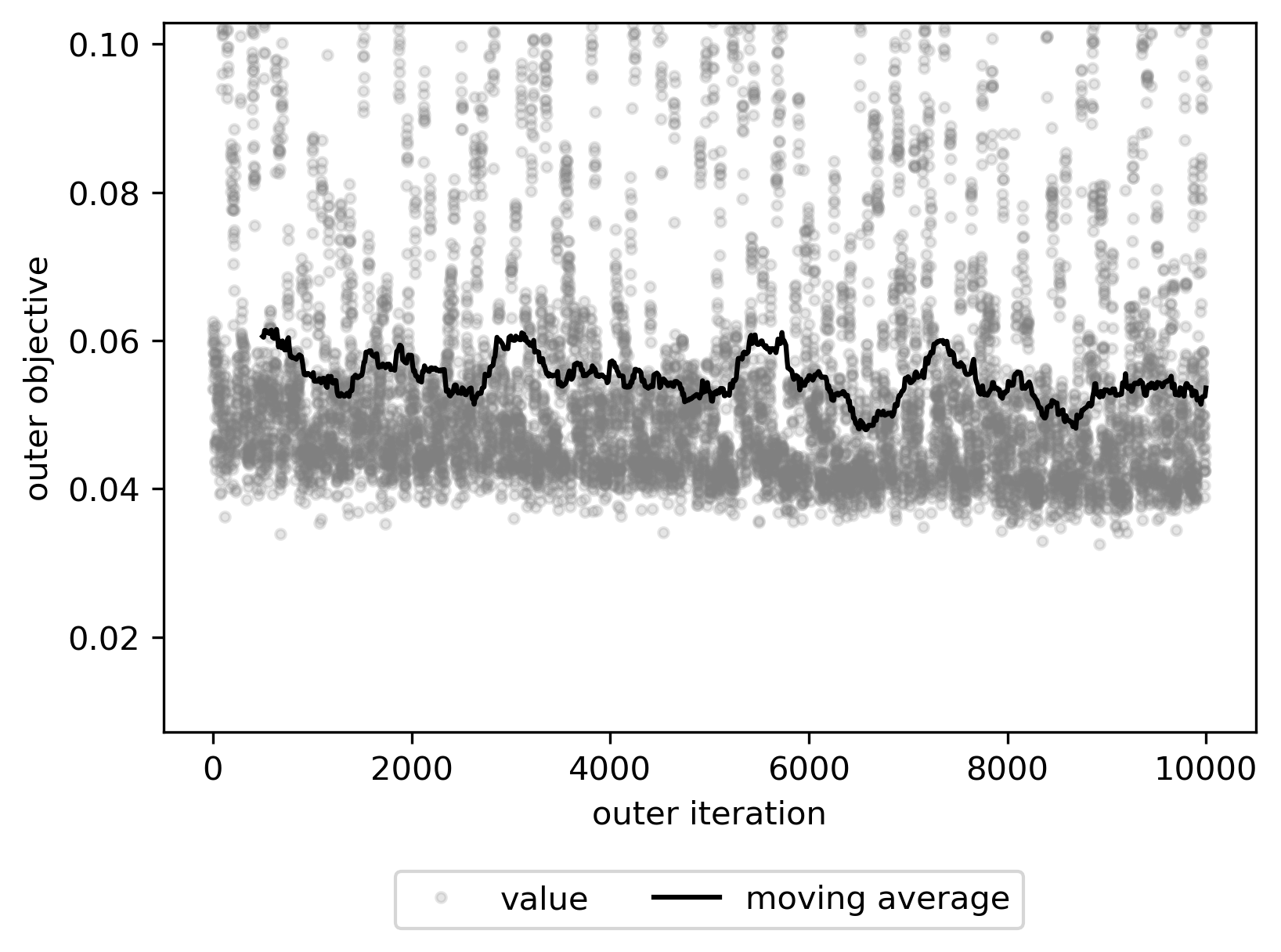}
		\caption{}
		\label{fig:burgers:trainloss:ffn:dd:r2}
	\end{subfigure}
	\caption{Burgers equation: Outer objective values recorded during meta-training as well as corresponding moving averages ($500$ iterations window size); these can be construed as \textit{meta-training error} trajectories. 
		Results related to single data with objective function weights meta-learning (a, b) and double data without objective function weights meta-learning (c, d).
		FFN parametrization only and both task regimes r1 (a, c) and r2 (b, d) are included.
		It is shown in part (c) that the outer objective drops significantly during training for double data and r1; recall that it corresponds to loss after 20 iterations and thus, it cannot drop to very small values as typical machine learning objectives do.
		The learned loss obtained during the training of part (c) does not satisfy \nc; theory-driven regularization as discussed in Section~\ref{sec:meta:design:properties} fixes this issue.
	}
	\label{fig:burgers:trainloss:ffn}
\end{figure}
\begin{figure}[H]
	\centering
	\begin{subfigure}[t]{0.48\textwidth}
		\centering
		\includegraphics[width=\textwidth]{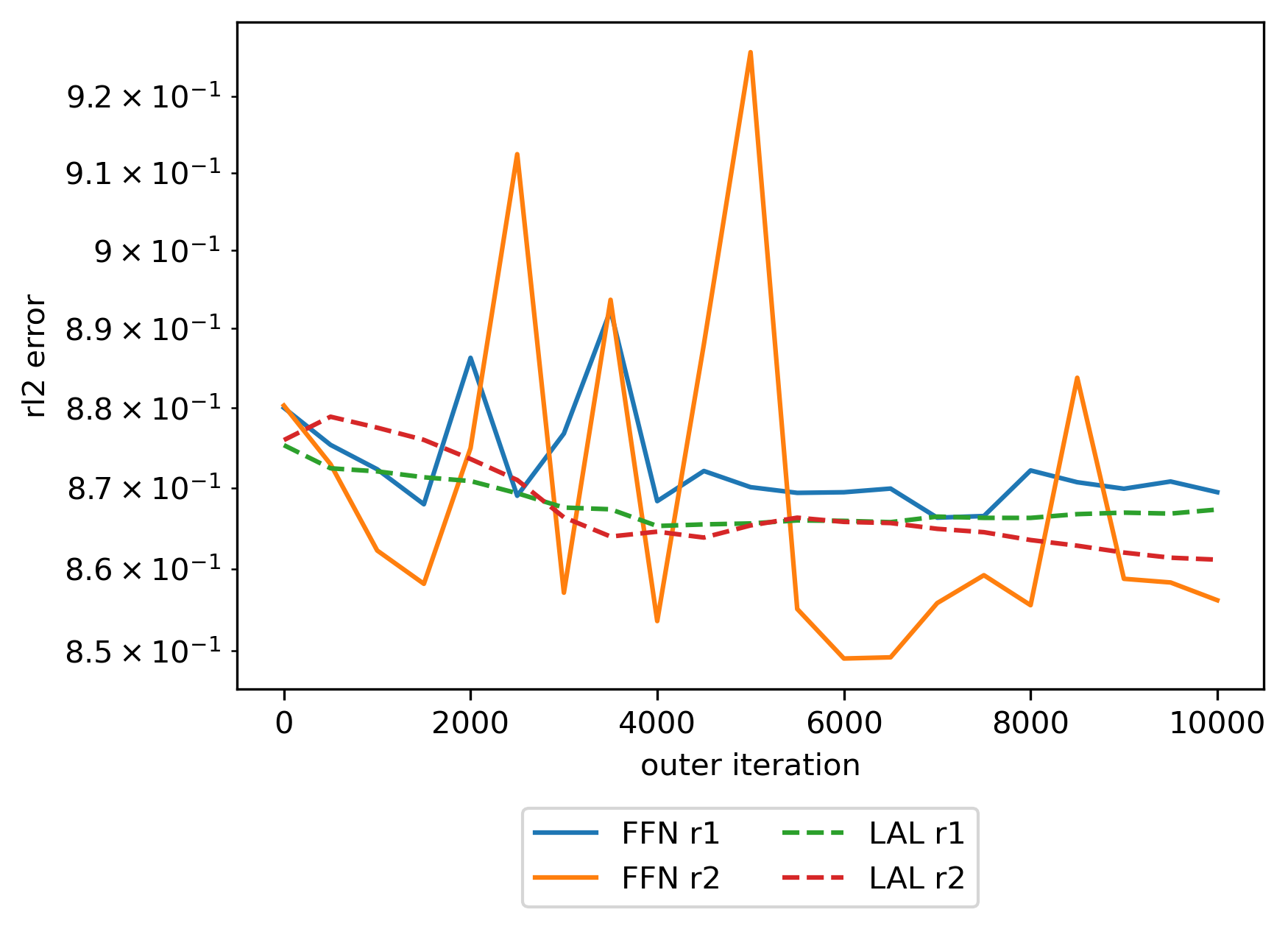}
		\caption{}
		\label{fig:burgers:traintest:lfw}
	\end{subfigure}
	\hfill
	\begin{subfigure}[t]{0.48\textwidth}
		\centering
		\includegraphics[width=\textwidth]{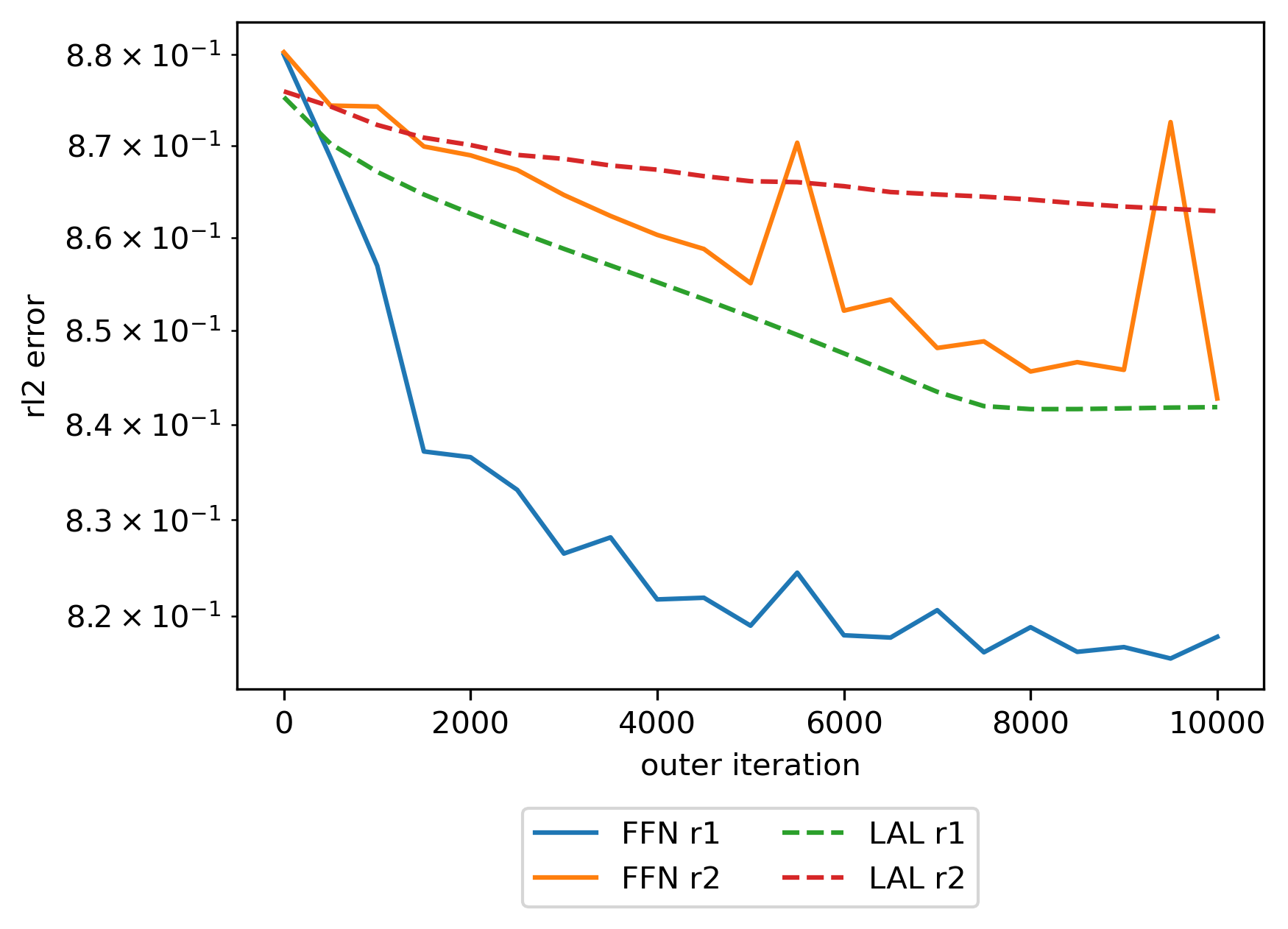}
		\caption{}
		\label{fig:burgers:traintest:dd}
	\end{subfigure}
	\caption{Burgers equation: Meta-testing results (relative $\ell_2$ test error on $5$ unseen tasks after $20$ iterations) obtained using learned loss snapshots and performed during meta-training (every $500$ outer iterations); these can be construed as \textit{meta-validation error} trajectories. 
		Results related to single data with objective function weights meta-learning (a) and double data without objective function weights meta-learning (b).
		FFN and LAL parametrizations and both task regimes r1 and r2 are included.
		It is shown in part (b) that rl2 drops significantly during training for double data and r1.
		In line with our theoretical results of Section~\ref{sec:meta:design:properties}, this learned loss does not satisfy \nc~and leads to divergence if used for full meta-testing ($20{,}000$ iterations), i.e., it does not generalize well although training (Fig.~\ref{fig:burgers:trainloss:ffn:dd:r1}) and validation (part b of present figure) performance is satisfactory.
		Theory-driven regularization as discussed in Section~\ref{sec:meta:design:properties} fixes this issue.}
	\label{fig:burgers:traintest}
\end{figure}
\begin{figure}[H]
	\centering
	\begin{subfigure}[t]{0.48\textwidth}
		\centering
		\includegraphics[width=\textwidth]{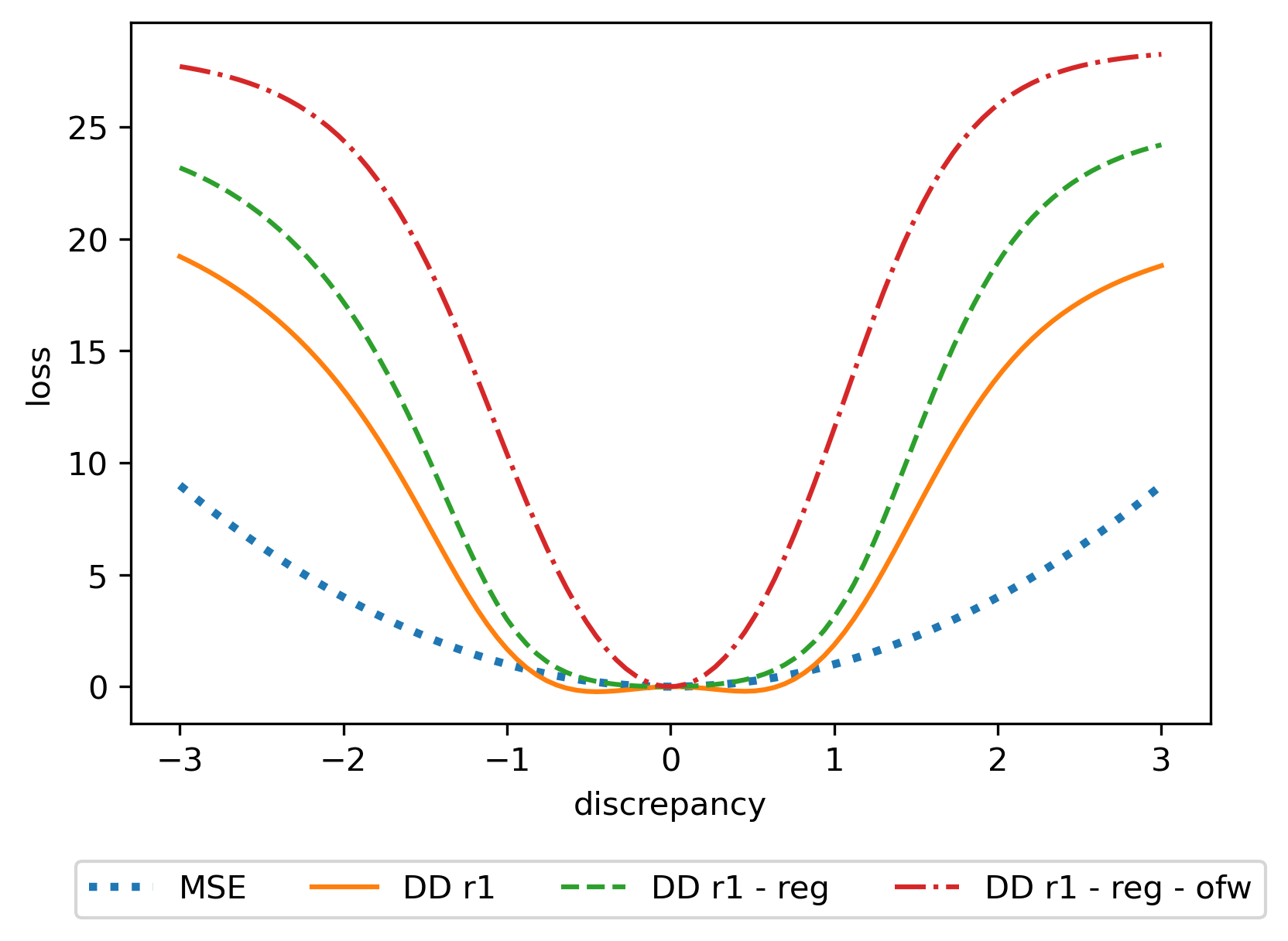}
		\caption{}
		\label{fig:burgers:final:loss:ffn:dd}
	\end{subfigure}
	\hfill
	\begin{subfigure}[t]{0.48\textwidth}
		\centering
		\includegraphics[width=\textwidth]{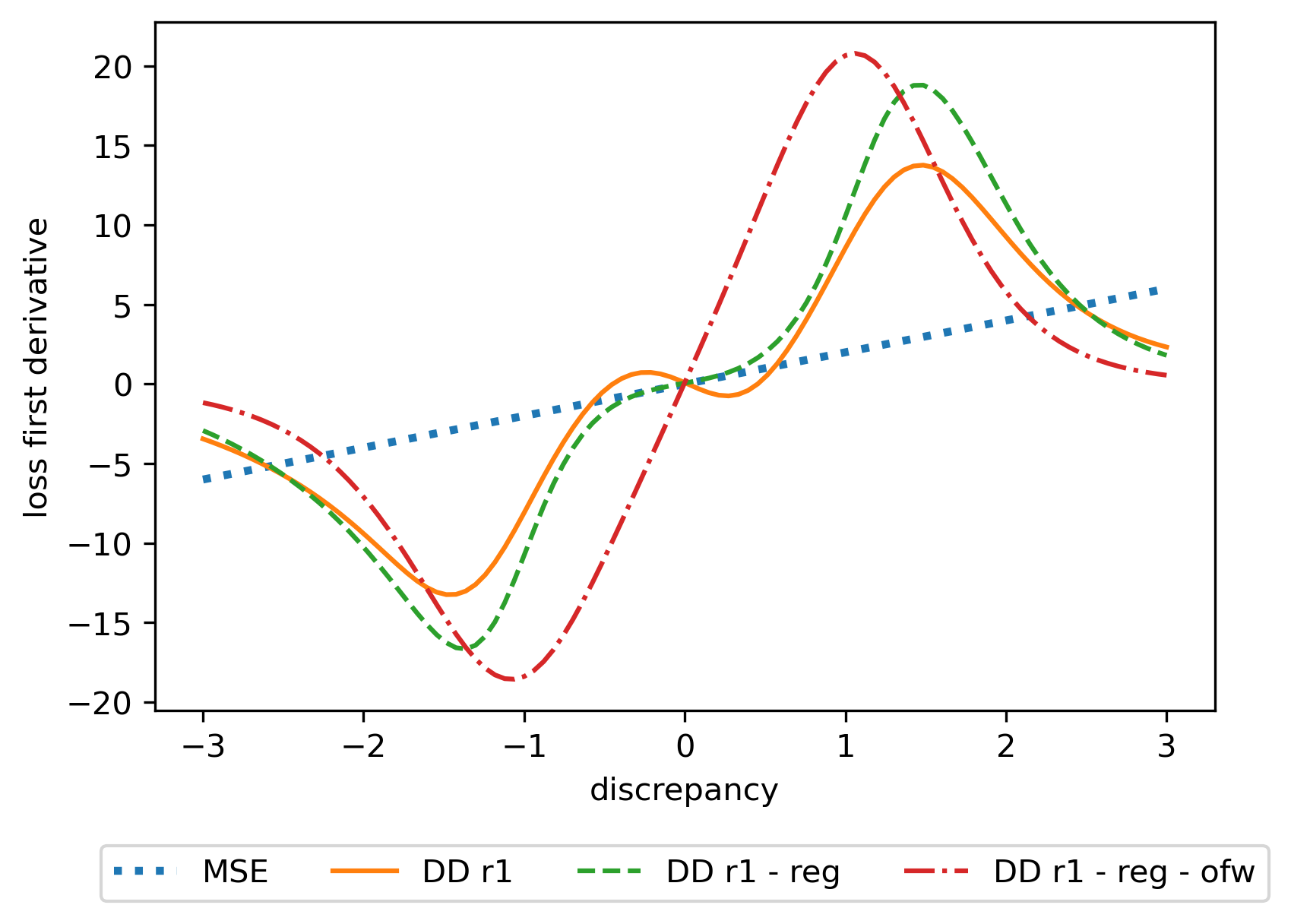}
		\caption{}
		\label{fig:burgers:final:lossder:ffn:dd}
	\end{subfigure}
	\caption{Burgers equation: Final learned losses (a) and corresponding first-order derivatives (b), with FFN parametrization for task regime $1$ with double data meta-training. 
		Results obtained with and without regularization (reg) and with and without objective function weights meta-learning (ofw); comparisons with MSE are also included.
		Whereas regularization is not required for LAL (Proposition~\ref{prop:1}), the final learned loss with FFN parametrization and double data (DD r1) does not satisfy \nc~and leads to divergence in optimization.
		Theory-driven regularization as discussed in Section~\ref{sec:meta:design:properties} fixes this issue.}
	\label{fig:burgers:final:ffn:dd}
\end{figure}

\paragraph*{Meta-testing.}
For evaluating the performance of the captured learned loss snapshots, we train with SGD for $20{,}000$ iterations with learning rate $0.01$ (same as in meta-training) $5$ OOD tasks using learned and standard loss functions and record the rl2 error on $10{,}000$ exact solution datapoints. 
Regarding learned losses, we employ the ones obtained with single data and with objective function weights meta-learning, as they performed better in our experiments.
The OOD test tasks are drawn based on the parameter limits shown in Table~\ref{tab:rd:params} and the random architectures used are drawn in the same way as in Section~\ref{sec:examples:rd}. 

The OAL parameters during training are shown in Fig.~\ref{fig:burgers:test:oal:params} and the minimum rl2 error results are shown in Fig.~\ref{fig:burgers:test:min:stats}. 
For both regimes, the final learned loss LAL 5 is different than both OAL 1 and 2, which converge to a robustness parameter value close to 3 for all tasks.
Finally, the loss functions learned with both parametrizations achieve an average minimum rl2 error during $20{,}000$ iterations that is significantly smaller than most considered losses (even the online adaptive ones), although they have been meta-trained with only $20$ iterations with a different PINN architecture and on a different task distribution.
\begin{figure}[H]
	\centering
	\begin{subfigure}[t]{0.48\textwidth}
		\centering
		\includegraphics[width=\textwidth]{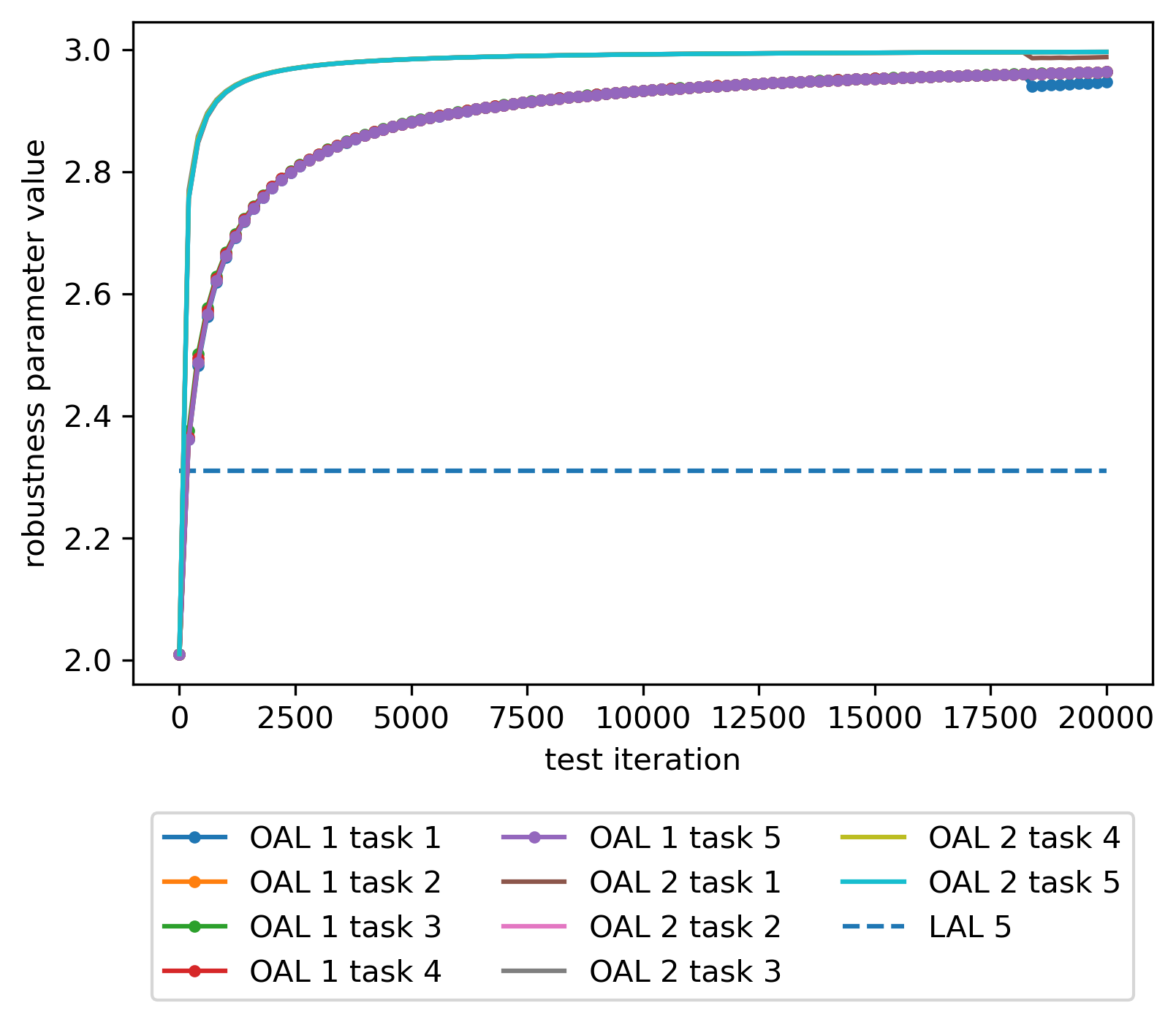}
		\caption{}
		\label{fig:burgers1:test:oal:params}
	\end{subfigure}
	\hfill
	\begin{subfigure}[t]{0.48\textwidth}
		\centering
		\includegraphics[width=\textwidth]{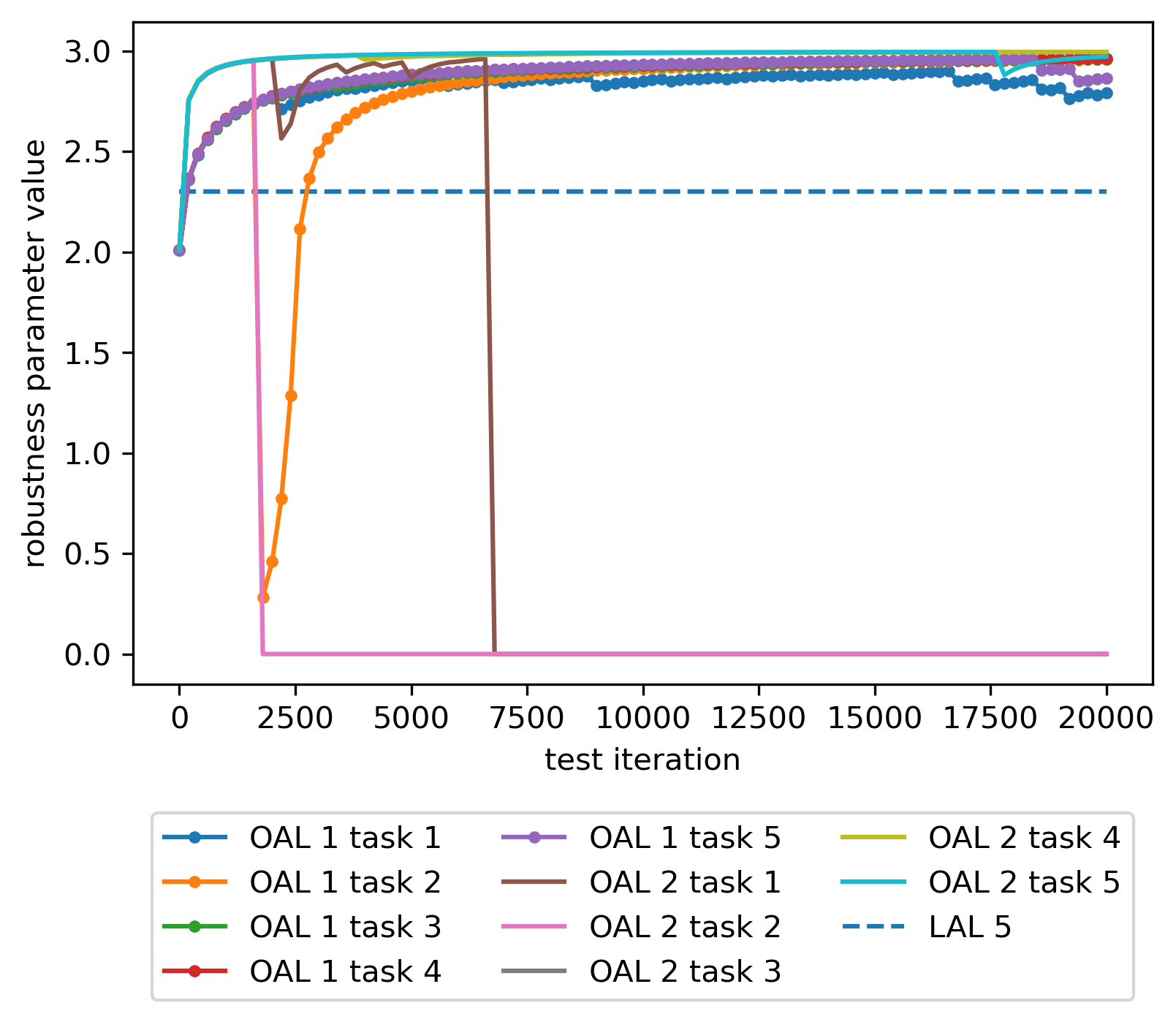}
		\caption{}
		\label{fig:burgers2:test:oal:params}
	\end{subfigure}
	\caption{Burgers equation: Robustness parameter trajectories for online adaptive loss functions OAL 1 and OAL 2 and for both test task regimes. 
		Loss-specific learning rates $0.01$ (OAL 1) and $0.1$ (OAL 2); see Section~\ref{sec:meta:design:param:LAL}.
		Comparison with final learned loss obtained via meta-training with LAL parametrization also included.
		Results correspond to test regimes 1 (a; $10^{-4} \leq  \lambda \leq 10^{-2}$) and 2 (b; $10^{-2} \leq  \lambda \leq 2$).
		For both regimes, final learned loss LAL 5 is different than both OAL 1 and 2, which converge to a robustness parameter value close to 3 for all tasks.}
	\label{fig:burgers:test:oal:params}
\end{figure}

\begin{figure}[H]
	\centering
	\begin{subfigure}[t]{1\textwidth}
		\centering
		\includegraphics[width=.7\linewidth]{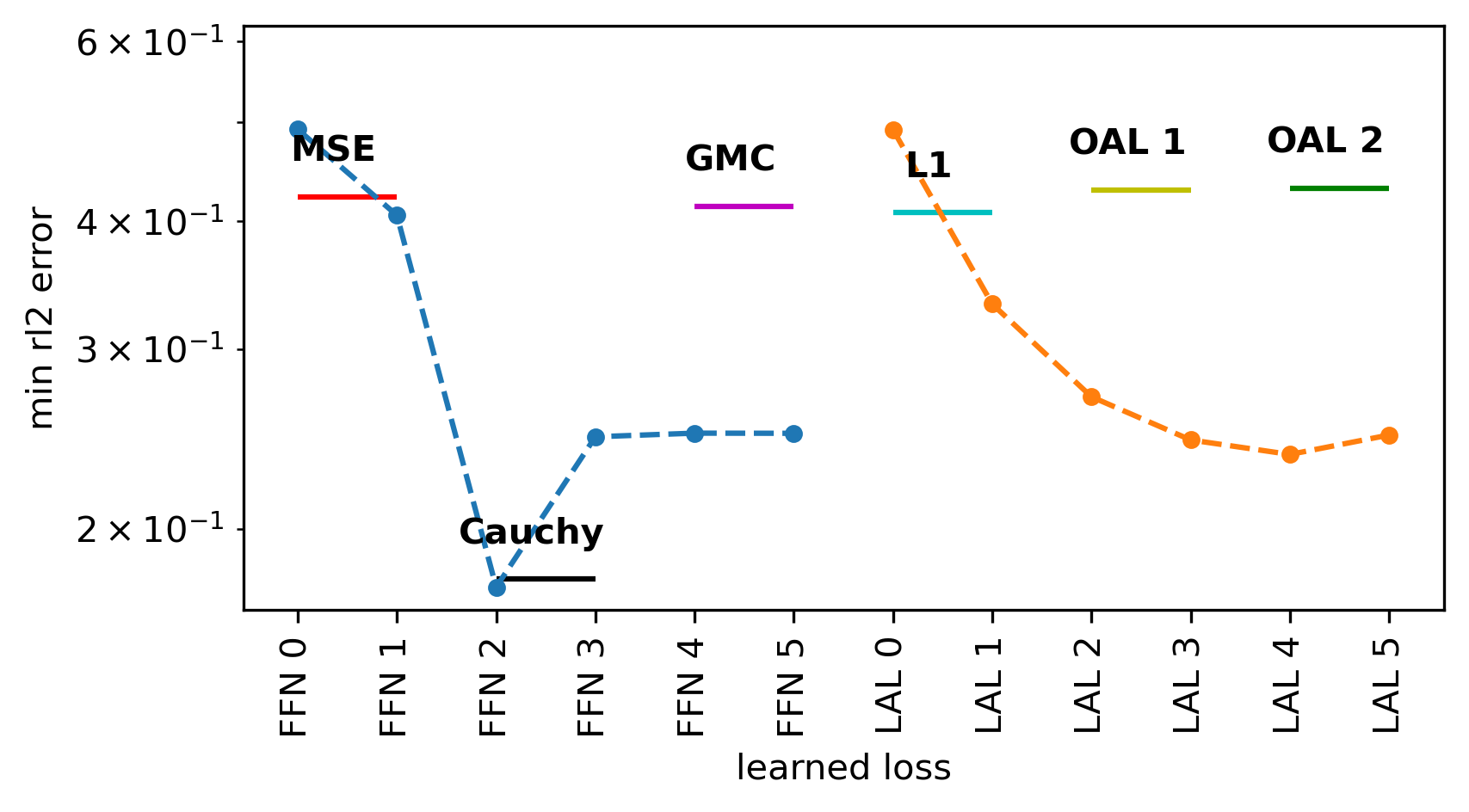}  
		\caption{}
		\label{fig:burgers1:test:min:stats}
	\end{subfigure}
	\hfill
	\begin{subfigure}[t]{1\textwidth}
		\centering
		\includegraphics[width=.7\linewidth]{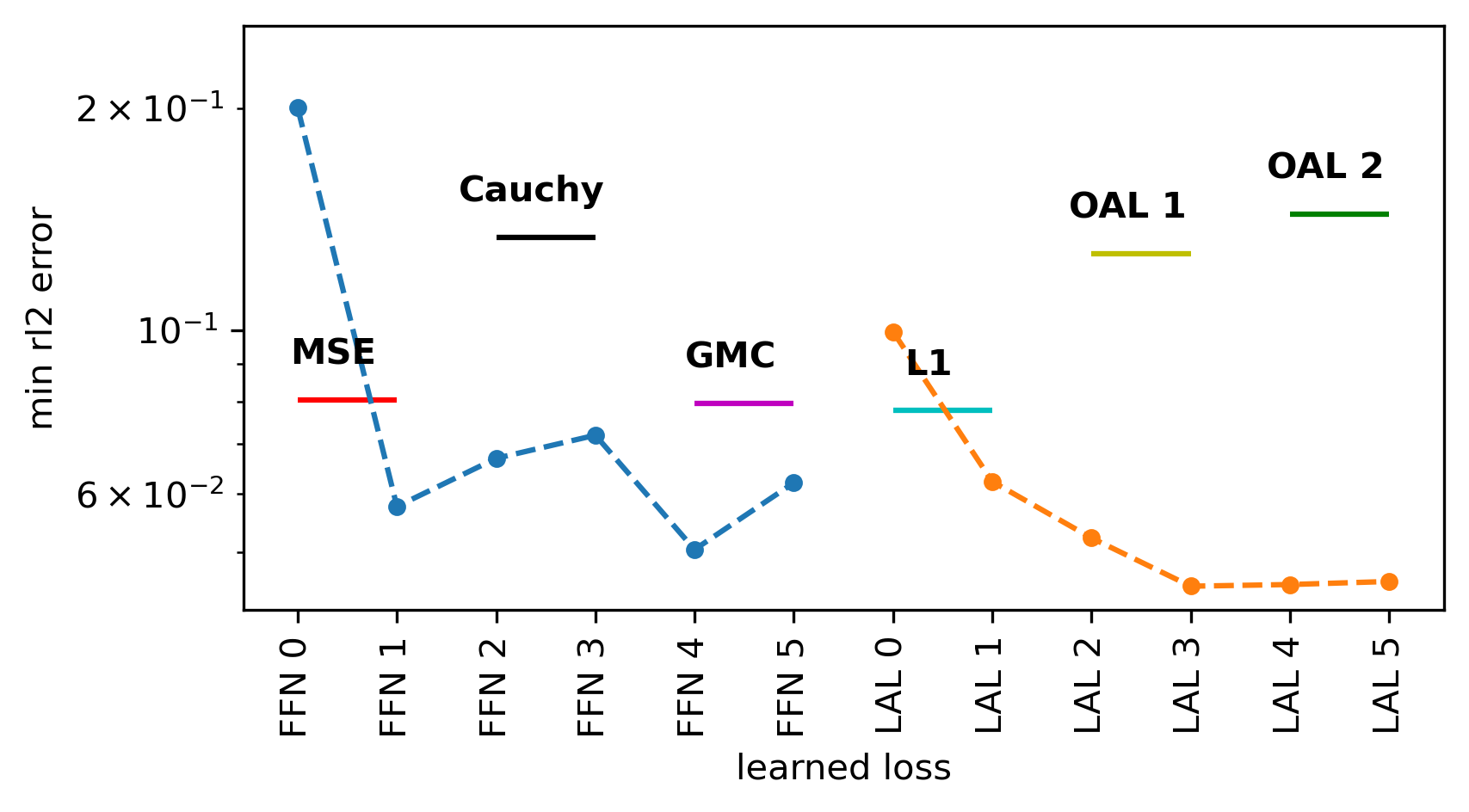}  
		\caption{}
		\label{fig:burgers2:test:min:stats}
	\end{subfigure}
	\caption{Burgers equation: Minimum relative test $\ell_2$ error (rl2) averaged over $5$ OOD tasks during meta-testing with SGD for $20{,}000$ iterations.
		Learned loss snapshots (FFN 0-6 and LAL 0-6) are compared with standard loss functions of Table~\ref{tab:losses} and with online adaptive loss functions OAL 1 and OAL 2 (2 loss-specific learning rates).
		Results correspond to test regimes 1 (a; $10^{-4} \leq  \lambda \leq 10^{-2}$) and 2 (b; $10^{-2} \leq  \lambda \leq 2$).
		For both regimes, the learned loss functions with both FFN and LAL parametrizations achieve an average minimum rl2 error that is significantly smaller than most considered losses (even the online adaptive ones), although they have been meta-trained with only $20$ iterations, with a different PINN architecture and on a different task distribution.
	}
	\label{fig:burgers:test:min:stats}
\end{figure}

\section{Summary}\label{sec:summary}

We have presented a meta-learning method for offline discovery of physics-informed neural network (PINN) loss functions, addressing diverse task distributions defined based on parametrized partial differential equations (PDEs) that are solved with PINNs.
For employing our technique given a PDE task distribution definition, we parametrize and optimize the learned loss via meta-training by following an alternating optimization procedure until a stopping criterion is met. 
Specifically, in each loss function update step, (a) PDE tasks are drawn from the task distribution; 
(b) in the inner optimization, they are solved with PINNs for a few iterations using the current learned loss and the gradient of the learned loss parameters is tracked throughout optimization; 
and (c) in the outer optimization, the learned loss parameters are updated based on MSE of the final (semi-optimized) PINN parameters. 

Furthermore, we have presented and proven two new theorems, involving a condition, namely \nc, that the learned loss should satisfy for successful training. If satisfied, this condition assures that under certain assumptions, a global minimum is reached by using gradient descent with the learned loss. In addition, we have proven that under a mean squared error (MSE) relation condition combined with \nc, any stationary point obtained based on the learned loss function is a global minimum of the MSE-based loss as well. Driven by these theoretical results, we have also proposed a novel regularization method for imposing the above desirable conditions. Finally, we have proven that one of the two parametrizations used in this paper for the learned loss, namely the learned adaptive loss (LAL) discussed in Section~\ref{sec:meta:design:param:LAL} and proposed in \cite{barron2019general}, satisfies automatically these two conditions without any regularization.

In the computational examples, the learned losses have been employed at test time for addressing regression and PDE task distributions. Our results have demonstrated that significant performance improvement can be achieved by using a shared-among-tasks offline-learned loss function even for out-of-distribution meta-testing; i.e., solving test tasks not belonging to the task distribution used in meta-training and utilizing PINN architectures that are different than the PINN architecture used in meta-training.
Note that a learned loss that performs equally well for architectures not used in meta-training is desirable if, for example, we seek to optimize the PINN architecture with a fixed learned loss.
Moreover, improved performance has been demonstrated even compared with adapting the loss function online as proposed in \cite{barron2019general}.
In this regard, we have considered the problems of discontinuous function approximation with varying frequencies, of the advection equation with varying initial conditions and discontinuous solutions, of the reaction-diffusion equation with varying source term, and of the Burgers equation with varying viscosity in two parametric regimes.
We have demonstrated the importance of different loss function parametrizations, as well as of other meta-learning algorithm design options discussed in Section~\ref{sec:meta:design}. 

As of future work, an interesting direction pertains to improving the performance of the technique for addressing inner optimizers with memory, such as RMSProp and Adam. 
Specifically, although the LAL learned losses coupled with Adam performed better than MSE in the function approximation case of Section~\ref{sec:examples:regression}, neither feed-forward NN (FFN) nor LAL learned losses exhibited satisfactory generalization capabilities in the PINN examples of Sections~\ref{sec:examples:ad}-\ref{sec:examples:burgers}.
One reason for this result is the fact that Adam depends on the whole history of gradients during optimization through an exponentially decaying average that only discards far in the past gradients.
To illustrate this, we have shown in Section~\ref{sec:examples:ad} that higher $\beta$ values in Adam corresponding to higher dependency on the far past yield deteriorating performance of the learned losses.  

\section{Acknowledgment}

This work was supported by OSD/AFOSR MURI grant FA9550-20-1-0358.

\bibliographystyle{elsarticle-num}	
\bibliography{refs}

\newpage
\counterwithin{figure}{section}
\counterwithin{table}{section}
\counterwithin{equation}{section}
\begin{appendices}
	\section{Loss functions with multi-dimensional inputs}\label{app:meta:design:multi}
	
	For each $(t, x)$ in Eq.~\eqref{eq:param:pde}, $u(t, x)$, $\pazocal{F}_\lambda[u](t,x)$, and $\pazocal{B}_\lambda[u](t,x)$ belong to $\mathbb{R}^{D_u}$.
	The same holds for the outputs of the NN approximators $\hat{u}_{\theta}$, $\hat{f}_{\theta, \lambda}$ and $\hat{b}_{\theta, \lambda}$.
	As a result, the loss function $\ell_{\eta}$, with $\ell_{\eta}: D_u \times D_u \to \mathbb{R}_{\geq 0}$, outputting the distance between the predicted $\pazocal{F}_\lambda[\hat{u}_{\theta}](t,x)$ and $0$, between the predicted $\pazocal{B}_\lambda[\hat{u}_{\theta}](t,x)$ and $0$, and between the predicted $u_{\theta}(0,x)$ and $u_{0, \lambda}(x)$ in Eq.~\eqref{eq:pinns:loss:2}, takes as input two $D_u$-dimensional vectors and outputs a scalar loss value. 
	Considering for simplicity, as in Section~\ref{sec:prelim:loss}, each part of the objective function of Eqs.~\eqref{eq:pinns:loss:1}-\eqref{eq:pinns:loss:3} separately, the squared $\ell_2$-norm loss between $\hat{u}_{\theta}(t, x)$ and $u(t, x)$ for each datapoint $((t, x), u(t, x))$ is given as 
	\begin{equation}\label{eq:l2:loss}
		||\hat{u}_{\theta}(t, x) - u(t, x)||_2^2 = \sum_{j=1}^{D_u}(\hat{u}_{\theta, j}(t, x)-u_j(t, x))^2.
	\end{equation}   
	In Eq.~\eqref{eq:l2:loss}, the loss for each $j \in \{1,\dots,D_u\}$ depends only on the discrepancy between $\hat{u}_{\theta, j}(t, x)$ and $u_j(t, x)$, and the total loss $||\hat{u}_{\theta}(t, x) - u(t, x)||_2^2$ is given as the sum of the one-dimensional losses (i.e., of the losses pertaining to one-dimensional inputs).
	In this regard, see Table~\ref{tab:losses} for other than squared error candidates for the one-dimensional loss of Eq.~\eqref{eq:l2:loss}.
	
	As a first attempt towards constructing a loss function with multi-dimensional inputs, one can generalize Eq.~\eqref{eq:l2:loss} by considering a parametrized function $\hat{\ell}_{\eta}$ instead of the squared error in the summation.
	This gives rise to a parametrized loss function given as
	\begin{equation}\label{eq:g:sum:loss}
		\ell_{\eta}(\hat{u}_{\theta}(t, x), u(t, x)) = \sum_{j=1}^{D_u}a_j\hat{\ell}_{\eta}(\hat{u}_{\theta, j}(t, x)-u_j(t, x)),
	\end{equation} 
	where, in addition, each directional loss is multiplied by a weight $a_j$ for making the loss function $\ell_{\eta}$ more flexible/expressive. 
	The weighted sum of Eq.~\eqref{eq:g:sum:loss} in conjunction with the meta-learning Algorithm~\ref{algo:general} can lead to a loss function that normalizes each directional loss optimally.
	Instead of using the same $\hat{\ell}_{\eta}$ for each dimension as in Eq.~\eqref{eq:g:sum:loss}, more expressive loss functions can be constructed by considering a different loss function $\hat{\ell}_{\eta}(\hat{u}_{\theta}(t, x) - u(t, x))$ for each dimension, i.e., 
	\begin{equation}\label{eq:g:sum:loss:2}
		\ell_{\eta}(\hat{u}_{\theta}(t, x), u(t, x)) = \sum_{j=1}^{D_u}\hat{\ell}_{\eta, j}(\hat{u}_{\theta, j}(t, x)-u_j(t, x)).
	\end{equation} 
	The latter can be extended further by parametrizing $\hat{\ell}_{\eta}$ in such a way that both $\hat{u}_{\theta}(t, x)$, $u(t, x)$ are given as inputs, and not only the discrepancy $\hat{u}_{\theta}(t, x) - u(t, x)$. 
	For obtaining even more expressive loss functions $\ell_{\eta}$, one can replace the summation formulas of Eqs.~\eqref{eq:g:sum:loss}-\eqref{eq:g:sum:loss:2} by a more general function $\hat{\ell}_{\eta}$ that takes as inputs the vectors $\hat{u}_{\theta}(t, x)$, $u(t, x)$ instead of the corresponding values in each dimension.
	See Fig.~\ref{fig:gNN:higher} for a schematic illustration of the aforementioned indicative options.
	
	\begin{figure}[H]
		\centering
		\includegraphics[width=.8\linewidth]{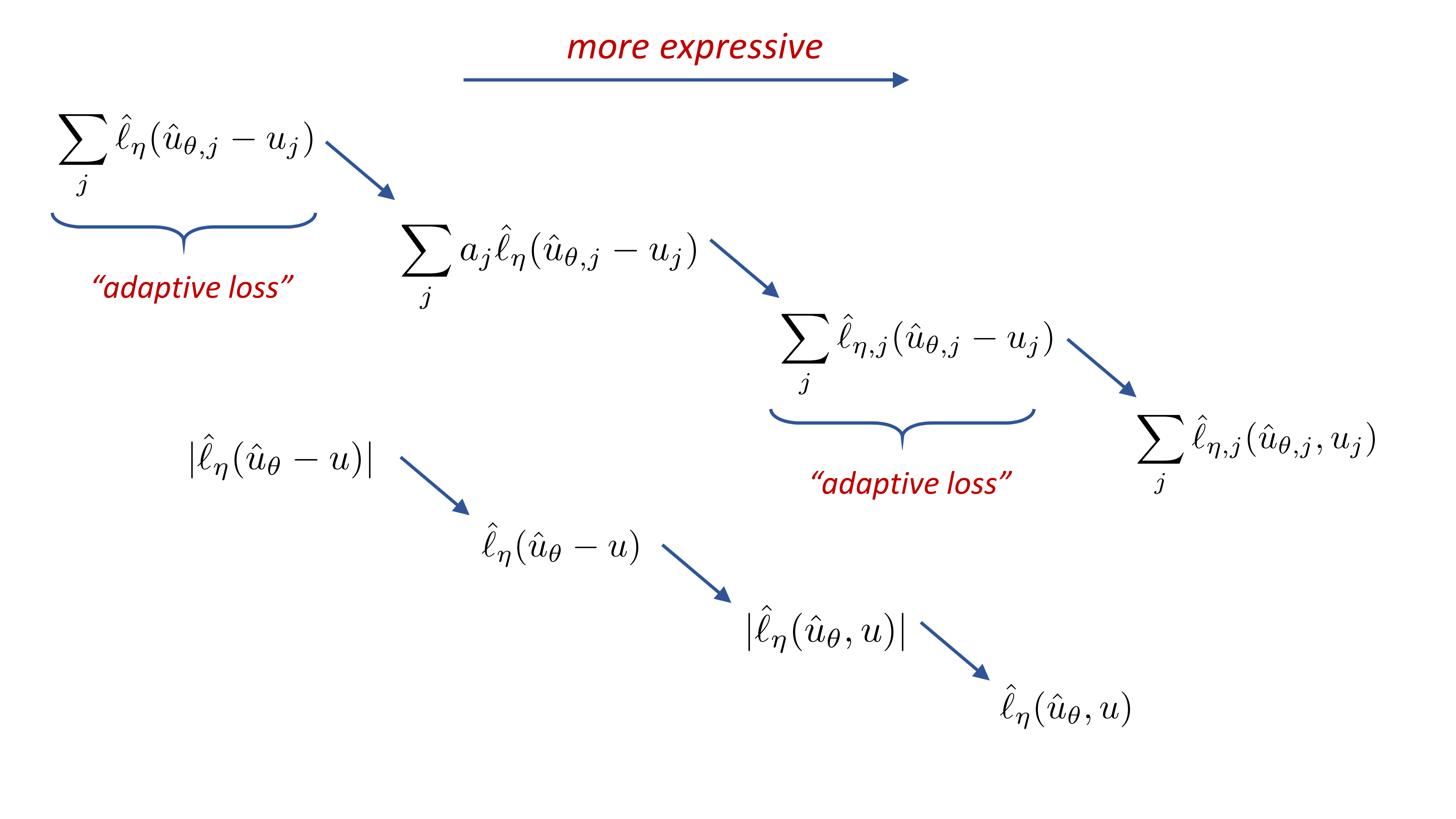}  
		\caption{Various indicative ways to represent loss function $\ell_{\eta}(\hat{u}_{\theta}(t, x)$, $u(t, x))$, ordered based on increasing expressiveness.
			``Adaptive loss'' corresponds to the loss function of \cite{barron2019general}; see also Section~\ref{sec:meta:design:param}.}
		\label{fig:gNN:higher}
	\end{figure}
	\begin{center}
		\begin{table}[H]
			\caption{Standard one-dimensional loss functions (i.e., pertaining to one-dimensional inputs $d=\hat{u}_{\theta, j}(t, x)-u_j(t, x)$ for each $j\in\{1,\dots,D_u\}$), accompanied by the corresponding first-order derivatives.
				Loss function general forms are adopted by \cite{barron2019general}; see more information in Section~\ref{sec:meta:design:param}.
			}
			\centering
			\begin{tabular}{ | c | c | c | c |} 
				\hline
				\textbf{Name} & \textbf{General form} & 
				\textbf{Loss sketch} & \textbf{Derivative sketch} \\
				\hline
				&&&\\[-1.67em]
				Squared error & $\frac{1}{2}(d/c)^2$  & 
				\begin{minipage}[]{.1\textwidth}
					\includegraphics[width=\linewidth, height=10mm]{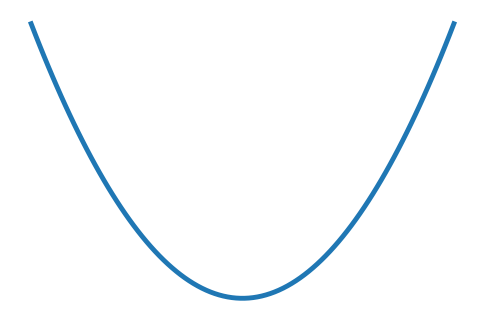}
				\end{minipage}&  
				\begin{minipage}[]{.1\textwidth}
					\includegraphics[width=\linewidth, height=10mm]{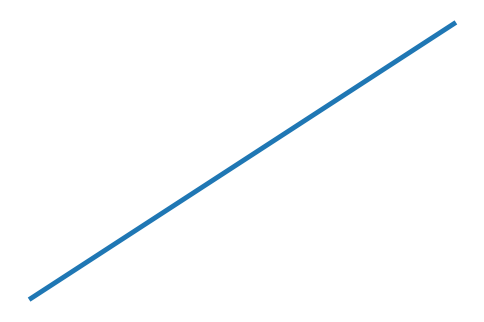}
				\end{minipage} \\
				\hline
				&&&\\[-1.67em]
				Absolute error & $|d|$  
				&  \begin{minipage}[]{.1\textwidth}
					\includegraphics[width=\linewidth, height=10mm]{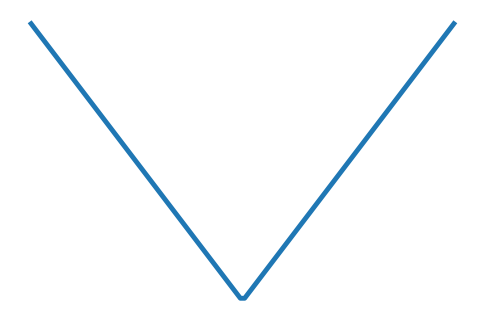}
				\end{minipage}&  
				\begin{minipage}[]{.1\textwidth}
					\includegraphics[width=\linewidth, height=10mm]{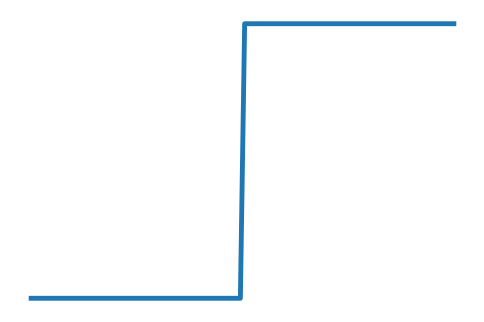}
				\end{minipage} \\
				\hline
				&&&\\[-1.67em]
				Huber & 
				\makecell{
					\shortstack{$0.5d^2 / c, \ \text{if } |x| < c$ \\ $\text{else} \ |x| - 0.5c$}
				}
				&  \begin{minipage}[]{.1\textwidth}
					\includegraphics[width=\linewidth, height=10mm]{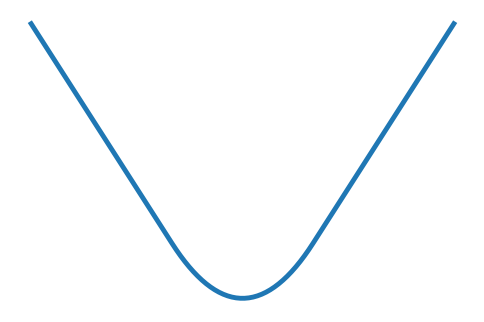}
				\end{minipage}&  
				\begin{minipage}[]{.1\textwidth}
					\includegraphics[width=\linewidth, height=10mm]{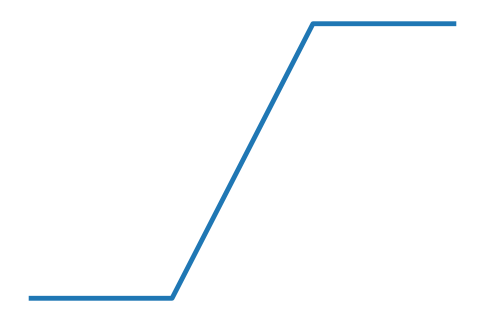}
				\end{minipage} \\
				\hline
				&&&\\[-1.67em]
				pseudo-Huber & $\sqrt{(d/c)^2 + 1} - 1$  
				&  \begin{minipage}[]{.1\textwidth}
					\includegraphics[width=\linewidth, height=10mm]{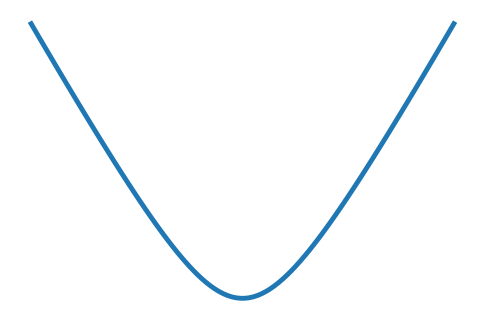}
				\end{minipage}&  
				\begin{minipage}[]{.1\textwidth}
					\includegraphics[width=\linewidth, height=10mm]{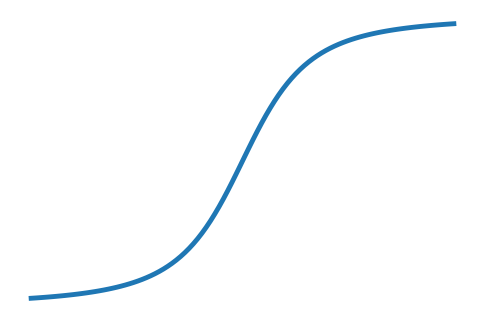}
				\end{minipage} \\
				\hline
				&&&\\[-1.67em]
				Cauchy & $\log\left(\frac{1}{2}(d/c)^2 + 1\right)$ 
				&  \begin{minipage}[]{.1\textwidth}
					\includegraphics[width=\linewidth, height=10mm]{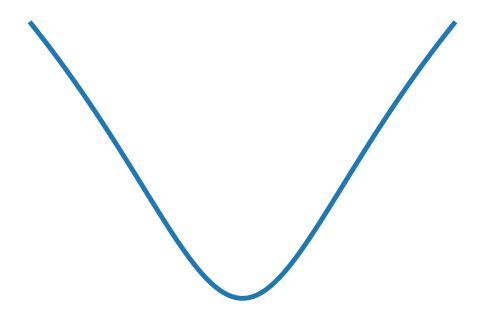}
				\end{minipage}&  
				\begin{minipage}[]{.1\textwidth}
					\includegraphics[width=\linewidth, height=10mm]{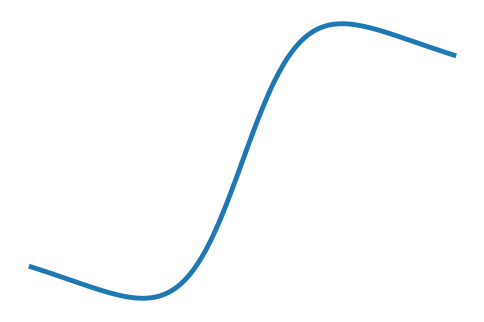}
				\end{minipage} \\
				\hline
				&&&\\[-1.67em]
				Geman-McClure & $\frac{2(d/c)^2}{(d/c)^2 + 4}$    
				&  \begin{minipage}[]{.1\textwidth}
					\includegraphics[width=\linewidth, height=10mm]{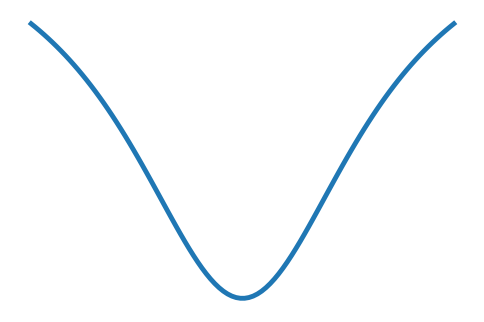}
				\end{minipage}&  
				\begin{minipage}[]{.1\textwidth}
					\includegraphics[width=\linewidth, height=10mm]{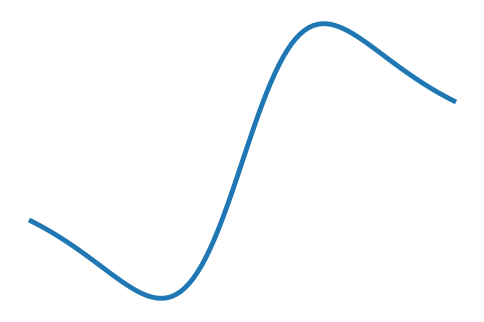}
				\end{minipage} \\
				\hline
				&&&\\[-1.67em]
				Welsch & $1 - \exp\left(-\frac{1}{2}(d/c)^2\right)$  
				&  \begin{minipage}[]{.1\textwidth}
					\includegraphics[width=\linewidth, height=10mm]{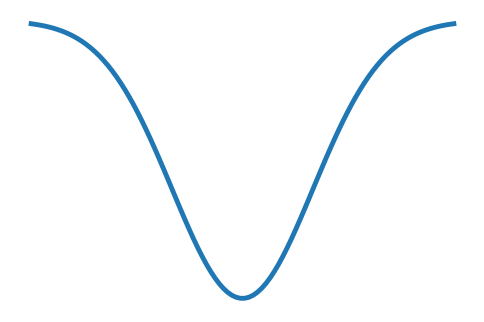}
				\end{minipage}&  
				\begin{minipage}[]{.1\textwidth}
					\includegraphics[width=\linewidth, height=10mm]{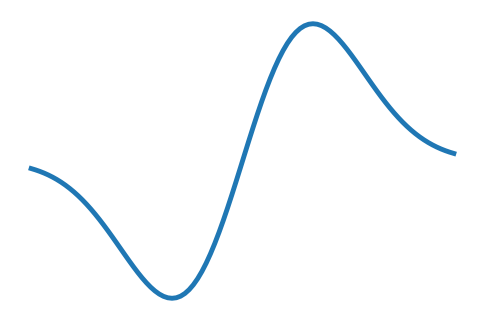}
				\end{minipage} \\
				\hline
			\end{tabular}
			\label{tab:losses}
		\end{table}
	\end{center}

\section{Proofs}\label{app:proofs}

\subsection{Proof of Theorem \ref{thm:1}}
\begin{proof}
	We now prove the first statement  of this theorem.
	Let $\theta$ be an arbitrary stationary point.
	Then, we have that 
	\begin{align}
		0=\frac{\partial \pazocal{L}(\theta)}{\partial \theta} &=\frac{1}{N} \sum_{i=1}^N \left(\frac{\partial \ell(q,u_{i})}{\partial q} \Big\vert_{q=\hat{u}_{\theta}(x_{i})} \right) \frac{\partial \hat{u}_{\theta}(x_{i})}{\partial \theta}
		\\ & =\frac{1}{N} \sum_{i=1}^N \sum_{j=1}^{D_u} \left(\frac{\partial \ell(q,u_{i})}{\partial q_j} \Big\vert_{q=\hat{u}_{\theta}(x_{i})} \right) \frac{\partial \hat{u}_{\theta}(x_{i})_{j}}{\partial \theta} .
	\end{align}
	By rearranging for the gradient, 
	\begin{equation} \label{eq:new:1}
		0=N\nabla\pazocal{L}(\theta)=\sum_{i=1}^N \sum_{j=1}^{D_u}  \left(\frac{\partial \hat{u}_{\theta}(x_{i})_{j}}{\partial \theta}\right)\T \left(\frac{\partial \ell(q,u_{i})}{\partial q_j} \Big\vert_{q=\hat{u}_{\theta}(x_{i})} \right)\T.
	\end{equation}    
	By rearranging the double sum into the matrix-vector product, 
	\begin{equation}
		0 =\left(\frac{\partial \hat{u}_X (\theta)}{\partial \theta}\right)\T v,
	\end{equation}
	where 
	\begin{equation}
		v=\vect\left(\left(\frac{\partial \ell(q,u_{1})}{\partial q} \Big\vert_{q=\hat{u}_{\theta}(x_{1})} \right)\T,\dots,\left(\frac{\partial \ell(q,u_{N})}{\partial q} \Big\vert_{q=\hat{u}_{\theta}(x_{N})} \right)\T \right)\in \RR^{N D_u}.
	\end{equation}
	Therefore, if $\rank(\frac{\partial \hat{u}_X (\theta)}{\partial \theta})=ND_u$, we have that $v=0$. 
	Here, by the definition of the $v$, the fact of $v=0$ implies that 
	\begin{equation}
		\frac{\partial \ell(q,u_{i})}{\partial q} \Big\vert_{q=\hat{u}_{\theta}(x_{i})}=0,
	\end{equation}
	for all $i =1,\dots, N$. 
	Since the loss $\ell$ satisfies \nc, this implies that for all $i =1,\dots, N$,
	\begin{equation}
		\ell(\hat{u}_{\theta}(x_{i}),u_{i})\le \ell(q',u_i), \qquad \forall q' \in \RR^{D_u}.
	\end{equation}  
	Since the objective function $\pazocal{L}$ is the sum of these terms, 
	\begin{equation}
		\pazocal{L}(\theta)=\frac{1}{N} \sum_{i=1}^N \ell(\hat{u}_{\theta}(x_{i}),u_{i})\le\frac{1}{N} \sum_{i=1}^N \ell(q'_{i},u_i), \qquad \forall q'_{1},\dots, q'_{N}\in \RR^{D_u}.
	\end{equation}
	This implies that 
	\begin{equation}
		\pazocal{L}(\theta)\le \inf_{ q'_{1},\dots, q'_{N}\in \RR^{D_u}}\frac{1}{N} \sum_{i=1}^N \ell(q'_{i},u_i) \le\pazocal{L}(\theta'), \qquad \forall \theta \in \RR^{D_{\theta}},
	\end{equation}
	where $D_{\theta}$ is the dimension of $\theta$.
	This shows that an arbitrary stationary point $\theta$ is a global minimum of $\pazocal{L}$ if $\rank(\frac{\partial \hat{u}_X (\theta)}{\partial \theta})=ND_u$. This proves the first statement of this theorem. 
	
	We now proceed to prove the second statement  of this theorem. Using Eq.~\eqref{eq:new:1}, for any $\theta$ such that $(\frac{\partial \ell(q,u_{i})}{\partial q_j} \vert_{q=\hat{u}_{\theta}(x_{i})})=0$ for all $i \in \{1,\dots,N\}$,
	we have $\nabla\pazocal{L}(\theta)=0$. In other words, every  $\theta$ such that $(\frac{\partial \ell(q,u_{i})}{\partial q_j} \vert_{q=\hat{u}_{\theta}(x_{i})})=0$ for all $i \in \{1,\dots,N\}$ is a stationary point of  $\pazocal{L}$. Using the assumption of  $\{\hat{u}_X (\theta)\in \RR^{N D_u}:\theta\in \RR^{D_{\theta}}\}=\RR^{N D_u}$,  if there exists a global minimum of $\pazocal{L}$ (it is possible that a global minimum does not exist), then it achieves the global minimum values of $\ell(\cdot,u_i)$ for all $i \in \{1,\dots,N\}$. 
	Thus, all we need to show is the existence of a stationary point of $\pazocal{L}$  that does not achieve the global minimum values of $\ell(\cdot,u_i)$ for all $i \in \{1,\dots,N\}$. Using the assumption of  $\{\hat{u}_X (\theta)\in \RR^{N D_u}:\theta\in \RR^{D_{\theta}}\}=\RR^{N D_u}$ and the assumption of having a point $q \in \RR^{D_u}$ such that $\nabla_{q} \ell(q,u)=0$ and   $\ell(q,u) > \ell(q',u)$ for some $q' \in \RR^{D_u}$, there exists a $\bar \theta$ such that for all $i \in \{1,\dots,N\},$  
	$$
	\frac{\partial \ell(q,u_{i})}{\partial q_j} \Big\vert_{q=\hat{u}_{\bar \theta}(x_{i})}=0 \text{ and } \ell(\hat{u}_{\bar \theta}(x_{i}),u_{i}) > \ell(q'_{i},u_{i}) \text{ for some } q'_{i} \in \RR^{D_u}.
	$$ 
	Here,  $\bar \theta$ is a stationary point of $\pazocal{L}$   since  $\frac{\partial \ell(q,u_{i})}{\partial q_j} \vert_{q=\hat{u}_{\bar \theta}(x_{i})}=0$  for all $i \in \{1,\dots,N\}$. Moreover,   $\bar \theta$ is not a global minimum  of $\pazocal{L}$   since $\ell(\hat{u}_{\bar \theta}(x_{i}),u_{i}) > \ell(q'_{i},u_{i}) \text{ for some } q'_{i} \in \RR^{D_u}$. This proves  the second statement of this theorem. 
\end{proof} 

\subsection{Proof of Theorem \ref{thm:2}}
To prove this theorem, we utilize the following known lemma:

\begin{lemma} \label{lemma:known_1}
	For any differentiable function $\varphi: \dom(\varphi) \rightarrow \RR$ with an open  convex domain $\dom(\varphi) \subseteq \RR^{D_\varphi}$, if   $\|\nabla \varphi(z') - \nabla \varphi(z)\| \le L_{ \varphi} \|z'-z\|$ for all $z,z' \in \dom(\varphi)$, then
	\begin{align}
		\varphi(z') \le \varphi(z) + \nabla \varphi(z)\T (z'-z) + \frac{L_{ \varphi}}{2} \|z'-z\|^2 \quad   \text{for all $z,z' \in \dom(\varphi) $}.
	\end{align}
\end{lemma}

\begin{proof}[Proof of Lemma \ref{lemma:known_1}]
	Fix $z,z'\in \dom(\varphi) \subseteq \RR^{D_{\varphi}}$. Since $\dom(\varphi)$ is a convex set,  $z+r(z'-z)\in\dom(\varphi)$ for all $r \in [0, 1]$.
	Since $\dom(\varphi) $ is open, there exists $\zeta >0$ such that $z+(1+\zeta')(z'-z)\in\dom(\varphi)$ and $z+(0-\zeta')(z'-z)\in\dom(\varphi)$ for all $\zeta'\le \zeta$. Fix $\zeta>0$ to be such a number. Combining these, $z+r(z'-z)\in\dom(\varphi)$ for all $r \in [0-\zeta, 1+\zeta]$.
	
	Accordingly, we can define  a function $\bar \varphi: [0-\zeta, 1+\zeta] \rightarrow \RR$ by $\bar \varphi(r)=\varphi(z+r(z'-z))$. Then, $\bar \varphi(1)=\varphi(z')$, $\bar \varphi(0)=\varphi(z)$, and $\nabla \bar \varphi(r)=\nabla\varphi(z+r(z'-z))\T (z'-z)$ for $r \in [0, 1] \subset (0-\zeta,1+\zeta)$. Since $\|\nabla \varphi(z') - \nabla \varphi(z)\| \le L_{ \varphi} \|z'-z\|$, 
	\begin{align}
		\|\nabla\bar  \varphi(r')-\nabla\bar  \varphi(r) \| &=\|[\nabla\varphi(z+r'(z'-z)) -\nabla\varphi(z+r(z'-z))\T (z'-z) \|
		\\ &\le \|z'-z\|\|\nabla\varphi(z+r'(z'-z)) -\nabla\varphi(z+r(z'-z))  \| \\ & \le  L_{ \varphi}\|z'-z\|\|(r'-r)(z'-z)   \|
		\\ & \le L_{ \varphi}\|z'-z\|^{2}\|r'-r   \|.
	\end{align}
	Thus, $\nabla \bar \varphi:[0, 1]\rightarrow \RR$ is Lipschitz continuous with the Lipschitz constant $L_{ \varphi}\|z'-z\|^{2}$, and hence $\nabla\bar  \varphi$ is continuous. 
	
	By using the fundamental theorem of calculus with the continuous function $\nabla\bar  \varphi:[0, 1]  \rightarrow \RR$,
	\begin{align}
		\varphi(z')&=\varphi(z)+ \int_0^1 \nabla\varphi(z+r(z'-z))\T (z'-z)dr
		\\ &=\varphi(z)+\nabla\varphi(z)\T (z'-z)+ \int_0^1 [\nabla\varphi(z+r(z'-z))-\nabla\varphi(z)]\T (z'-z)dr
		\\ & \le \varphi(z)+\nabla\varphi(z)\T (z'-z)+ \int_0^1 \|\nabla\varphi(z+r(z'-z))-\nabla\varphi(z)\| \|z'-z \|dr
		\\ & \le \varphi(z)+\nabla\varphi(z)\T (z'-z)+ \int_0^1 r L_{ \varphi}\|z'-z\|^{2}dr
		\\ & =  \varphi(z)+\nabla\varphi(z)\T (z'-z)+\frac{L_{ \varphi}}{2}\|z'-z\|^{2}. 
	\end{align}
	
\end{proof} 

\begin{proof}[Proof of Theorem \ref{thm:2}] 
	The function $\pazocal{L}$ is differentiable since $\ell_{i}$ is differentiable, $\theta_{}\mapsto \hat{u}_{\theta}(x_{i})$ is differentiable, and a composition of differentiable functions is differentiable. We will first show that in both cases of (i) and (ii) for the learning rates, we have $\lim_{r \rightarrow \infty}\nabla\pazocal{L}(\theta_{}^{(r)}) = 0$. If $\nabla\pazocal{L}(\theta_{}^{(r)})=0$ at any $r\ge 0$, then
	Assumption \ref{assump:5} ($\|\bg^{(r)}\|_2^2 \le \bc \|\nabla\pazocal{L}(\theta_{}^{(r)})\|_2^2$) implies  $\bg^{(r)}=0$, which implies 
	\begin{equation}
		\theta_{}^{(r+1)}= \theta^{(r)}_{} \text{ and } \nabla\pazocal{L}(\theta_{}^{(r+1)})=\nabla\pazocal{L}(\theta_{}^{(r)})=0.
	\end{equation}
	This means that  if $\nabla\pazocal{L}(\theta_{}^{(r)})=0$ at any $r\ge 0$, we have that  $\bg^{(r)}=0$ and $\nabla\pazocal{L}(\theta_{}^{(r')})=0$ for all $r '\ge r$ and hence
	\begin{align} 
		\lim_{r \rightarrow \infty}\nabla\pazocal{L}(\theta_{}^{(r)}) = 0,
	\end{align}
	as desired. Therefore, we now focus on the remaining scenario where $\nabla\pazocal{L}(\theta_{}^{(r)})\neq 0$ for all $r \ge0$. 
	
	By using Lemma \ref{lemma:known_1}, 
	\begin{equation}
		\pazocal{L}(\theta_{}^{(r+1)})\le \pazocal{L}(\theta_{}^{(r)})-\epsilon^{(r)}  \nabla\pazocal{L}(\theta_{}^{(r)})  \T\bg^{(r)}  + \frac{L(\epsilon^{(r)})^2  }{2} \|\bg^{(r)} \|^2.
	\end{equation}
	By rearranging and using Assumption \ref{assump:5},
	\begin{align}
		\pazocal{L}(\theta_{}^{(r)})-\pazocal{L}(\theta_{}^{(r+1)}) &\ge \epsilon^{(r)} \nabla\pazocal{L}(\theta_{}^{(r)})  \T\bg^{(r)}  - \frac{L(\epsilon^{(r)})^2  }{2} \|\bg^{(r)} \|^2
		\\ & \ge\epsilon^{(r)}  \uc \|\nabla\pazocal{L}(\theta_{}^{(r)})\|^2 -\frac{L(\epsilon^{(r)})^2  }{2}\bc \|\nabla\pazocal{L}(\theta_{}^{(r)})\|^2.
	\end{align}
	By simplifying the right-hand-side, 
	\begin{align} \label{eq:9_2}
		\pazocal{L}(\theta_{}^{(r)})-\pazocal{L}(\theta_{}^{(r+1)}) &\ge\epsilon^{(r)} \|\nabla\pazocal{L}(\theta_{}^{(r)})\|^2 (\uc-\frac{L\epsilon^{(r)}  }{2}\bc).
	\end{align}
	
	Let us now focus on case (i). Then, using $ \epsilon^{(r)} \le \frac{\uc\ (2-\zeta)}{L\bc}$, 
	\begin{equation}
		\frac{L\epsilon^{(r)}}{2}\bc\le\frac{L\uc\ (2-\zeta)}{2L\bc}\bc =\uc-\frac{\zeta}{2}\uc \text{ .}
	\end{equation}
	Using this inequality and using $\zeta\le \epsilon^{(r)}$ in Eq.~\eqref{eq:9_2},
	\begin{align} \label{eq:10_2}
		\pazocal{L}(\theta_{}^{(r)})-\pazocal{L}(\theta_{}^{(r+1)}) &\ge\frac{\uc\zeta^{2}}{2}
		\|\nabla\pazocal{L}(\theta_{}^{(r)})\|^2.
	\end{align}
	Since $\nabla\pazocal{L}(\theta_{}^{(r)})\neq 0$ for any $r\ge 0$ (see above) and $\zeta >0 $, this means that the sequence $(\pazocal{L}(\theta_{}^{(r)}))_{r}$ is monotonically decreasing. Since $\pazocal{L}(q) \ge 0$ for any $q$ in its domain, this implies that the sequence $(\pazocal{L}(\theta_{}^{(r)}))_{r}$ converges. Therefore, $\pazocal{L}(\theta_{}^{(r)})-\pazocal{L}(\theta_{}^{(r+1)}) \rightarrow 0$  as $r \rightarrow \infty$. Using Eq.~\eqref{eq:10_2}, this implies that 
	\begin{equation}
		\lim_{r \rightarrow \infty}\nabla\pazocal{L}(\theta_{}^{(r)}) = 0,
	\end{equation}
	which proves the desired result for the case (i). 
	
	We now focus on the case (ii). Then, we still have Eq.~\eqref{eq:9_2}. Since $\lim_{r \rightarrow \infty}\epsilon^{(r)} =0$ in Eq.~\eqref{eq:9_2}, the first order term in $\epsilon^{(r)}$ dominates after sufficiently large $r$: i.e., there exists $\bar r \ge0$ such that for any $r\ge \bar r$, 
	\begin{align} \label{eq:11_2}
		\pazocal{L}(\theta_{}^{(r)})-\pazocal{L}(\theta_{}^{(r+1)}) \ge c \epsilon^{(r)}   \|\nabla\pazocal{L}(\theta_{}^{(r)})\|^2
	\end{align}
	for some constant $c>0$. Since $\nabla\pazocal{L}(\theta_{}^{(r)})\neq 0$ for any $r\ge 0$ (see above) and $c \epsilon^{(r)}>0$, this means that the sequence $(\pazocal{L}(\theta_{}^{(r)}))_{r}$ is monotonically decreasing.  Since $\pazocal{L}(q) \ge 0$ for any $q$ in its domain, this implies that the sequence $(\pazocal{L}(\theta_{}^{(r)}))_{r}$ converges to a finite value.
	Thus, by adding Eq. \eqref{eq:11_2} both sides over all $r \ge \bar r$, 
	\begin{align} 
		\infty >  \pazocal{L}(\theta_{}^{(\bar r)})-\lim_{r \rightarrow \infty}\pazocal{L}(\theta_{}^{(r)}) \ge c  \sum _{r=\bar r}^\infty \epsilon^{(r)}   \|\nabla\pazocal{L}(\theta_{}^{(r)})\|^2.
	\end{align}        
	Since $\sum _{r=0}^\infty \epsilon^{(r)}  = \infty$, this implies that  $\liminf _{r\to \infty }\|\nabla\pazocal{L}(\theta^{(r)})\|=0$. We now show that $ \limsup_{r\to \infty }\|\nabla\pazocal{L}(\allowbreak \theta^{(r)})\|=0$ by contradiction. Suppose that $\limsup_{r\to \infty }\|\nabla\pazocal{L}(\theta^{(r)})\| > 0$. Then, there exists $\delta>0$ such that $\limsup_{r\to \infty }\|\nabla\pazocal{L}(\theta^{(r)})\|\ge \delta$. Since $\liminf _{r\to \infty }\|\nabla\pazocal{L}(\theta^{(r)})\|=0$ and $\limsup_{r\to \infty }\|\nabla\pazocal{L}(\theta^{(r)})\|\ge \delta$, let ${\rho_j}_{j}$ and ${\rho'_j}_j$ be  sequences of indexes such that $\rho_j<\rho'_j<\rho_{j+1}$, $\|\nabla\pazocal{L}(\theta^{(r)})\|>\frac{\delta}{3}$ for $\rho_j \le r < \rho_j'$, and  $\|\nabla\pazocal{L}(\theta^{(r)})\|\le \frac{\delta}{3}$ for $\rho_j '\le r < \rho_{j+1}$. Since $\sum _{r=\bar r}^\infty \epsilon^{(r)}   \|\nabla\pazocal{L}(\theta_{}^{(r)})\|^2< \infty$, let $\bar j$ be sufficiently large such that  $\sum _{r=\rho_{\bar j}}^\infty \epsilon^{(r)}   \|\nabla\pazocal{L}(\theta_{}^{(r)})\|^2<\frac{\delta^2}{9L\sqrt{\bc}}$. Then, for any $j \ge \bar j$ and any $\rho$ such that $\rho_j \le \rho \le \rho'_j -1$, we have that 
	\begin{align}
		\|\nabla\pazocal{L}(\theta_{}^{(\rho)})\|-\|\nabla\pazocal{L}(\theta_{}^{(\rho_j')})\|&\le\|\nabla\pazocal{L}(\theta_{}^{(\rho_j')})-\nabla\pazocal{L}(\theta_{}^{(\rho)})\|
		\\ &=\left\|\sum_{r=\rho}^{\rho'_j-1}\nabla\pazocal{L}(\theta_{}^{(r+1)})-\nabla\pazocal{L}(\theta_{}^{(r)})\right\|
		\\ & \le \sum_{r=\rho}^{\rho'_j-1}\left\|\nabla\pazocal{L}(\theta_{}^{(r+1)})-\nabla\pazocal{L}(\theta_{}^{(r)})\right\|  \\ & \le L  \sum_{r=\rho}^{\rho'_j-1}\left\|\theta_{}^{(r+1)}-\theta_{}^{(r)}\right\|
		\\ & \le L\sqrt{\bc}  \sum_{r=\rho}^{\rho'_j-1} \epsilon^{(r)}   \left\|\nabla\pazocal{L}(\theta_{}^{(r)})\right\|,
	\end{align}
	where the first and third lines use the triangle inequality (and symmetry), the forth line uses the assumption that     $\|\nabla\pazocal{L}(\theta)-\nabla\pazocal{L}(\theta')\|\le L \|\theta-\theta'\|$, and the last line follows the  definition of  $\theta_{}^{(r+1)}-\theta_{}^{(r)}=-\epsilon^{(r)} \bg$ and the assumption of $\|\bg^{(r)}\|^2 \le \bc \|\nabla\pazocal{L}(\theta_{}^{(r)})\|^2
	$.  Then, by using the definition of the sequences of the indexes, 
	\begin{align}
		\|\nabla\pazocal{L}(\theta_{}^{(\rho)})\|-\|\nabla\pazocal{L}(\theta_{}^{(\rho_j')})\|&\le \frac{3L\sqrt{\bc}}{\delta}  \sum_{r=\rho}^{\rho'_j-1} \epsilon^{(r)}   \left\|\nabla\pazocal{L}(\theta_{}^{(r)})\right\|^2 \le \frac{\delta}{3}.
	\end{align}
	Here, since  $\|\nabla\pazocal{L}(\theta_{}^{(\rho_j')})\|\le \frac{\delta}{3}$, by rearranging the inequality, we have that for any $\rho\ge \rho_{\bar j}$, 
	\begin{align}
		\|\nabla\pazocal{L}(\theta_{}^{(\rho)})\|\le\frac{2\delta}{3}.
	\end{align}
	This contradicts the inequality of  $\limsup_{r\to \infty }\|\nabla\pazocal{L}(\theta^{(r)})\|\ge \delta$. Thus, we have 
	\begin{align}
		\limsup_{r\to \infty }\|\nabla\pazocal{L}(\theta^{(r)})\| =  \liminf _{r\to \infty }\|\nabla\pazocal{L}(\theta^{(r)})\|= 0.
	\end{align}
	This implies that 
	\begin{equation}
		\lim_{r \rightarrow \infty}\nabla\pazocal{L}(\theta_{}^{(r)}) = 0,
	\end{equation}
	which proves the desired result for the case (ii). 
	Therefore, in both cases of (i) and (ii) for the learning rates, we have $\lim_{r \rightarrow \infty}\nabla\pazocal{L}(\theta^{(r)}) = 0$. From Theorem \ref{thm:1}, this implies that an arbitrary limit point $\theta$ of the sequence $\{\theta^{(r)}\}_{r=0}$ is a global minimum of $\pazocal{L}$ if $\rank(\frac{\partial \hat{u}_X (\theta)}{\partial \theta})=ND_u$.   
	
\end{proof} 

\subsection{Proof of Corollary \ref{corollary:1}}
\begin{proof}
	By following  the proof of the first statement of Theorem \ref{thm:1} (see Eq. \eqref{eq:new:2}), we have that at any stationary point $\theta$ of $\pazocal{L}$,  
	\begin{equation} \label{eq:new:2}
		\frac{\partial \ell(q,u_{i})}{\partial q} \Big\vert_{q=\hat{u}_{\theta}(x_{i})}=0 \;\text{ for all } i =1,\dots, N.
	\end{equation}
	By the assumption of  $\nabla_{q} \ell(q,u)=0$ if and only if $q=u$, this implies that at any stationary point $\theta$ of $\pazocal{L}$, 
	$$
	\hat{u}_{\theta}(x_{i})=u_{i}  \;\text{ for all } i =1,\dots, N.
	$$
	This implies that at any stationary point $\theta$ of $\pazocal{L}$, we have that  $\pazocal{L}_{\mathrm{MSE}}(\theta)=\frac{1}{N}\sum_{i=1}^N\|\hat u_\theta(x_{i})-u_{i}\|_2^2=0$, which is the global minimum value of $\pazocal{L}_{\mathrm{MSE}}$.
\end{proof}

\subsection{Proof of Proposition \ref{prop:1}}
\begin{proof} Let $\ell(q,u)=\rho_{\alpha,c}(q-u) = \frac{|\alpha - 2|}{\alpha}((\frac{((q-u)/c)^2}{|\alpha - 2|} + 1)^{\alpha/2} - 1)$.
	Let $c$ and $\alpha$ be real numbers such that $c>0$, $\alpha\neq0$, and  $\alpha\neq 2$. Then,\begin{align}
		\frac{\partial \ell(q,u)}{\partial q} &=\frac{|\alpha - 2|}{\alpha} \left(\frac{\alpha}{2} \left(\frac{((q-u)/c)^2}{|\alpha - 2|} + 1\right)^{(\alpha/2)-1} \right)\frac{1}{|\alpha - 2|} \frac{1}{c^2} 2(q-u)
		\\ & =\left(\frac{1}{c^2}\left(\frac{((q-u)/c)^2}{|\alpha - 2|} + 1\right)^{(\alpha/2)-1}\right) (q-u). \label{eq:new:3}
	\end{align}
	Here, since any (real) power of strictly positive real number is strictly positive, we have 
	\begin{align} \label{eq:new:4}
		\frac{1}{c^2}\left(\frac{((q-u)/c)^2}{|\alpha - 2|} + 1\right)^{(\alpha/2)-1} >0.
	\end{align}
	By combining Eq. \eqref{eq:new:3} and Eq. \eqref{eq:new:4}, we have that  $\frac{\partial \ell(q,u)}{\partial q}=0$ implies $q-u =0$. On the other hand, using Eq. \eqref{eq:new:3}, we have that   $q-u =0$ implies   $\frac{\partial \ell(q,u)}{\partial q}=0$. In other words, 
	\begin{align} \label{eq:new:5}
		\frac{\partial \ell(q,u)}{\partial q}=0  \qquad \Longleftrightarrow \qquad  q = u.
	\end{align}
	This proves the statement for  the second condition that   $\nabla_{q} \ell(q,u)=0$ if and only if $q=u$. We now prove the  statement for \nc\ by showing that $q = u$ implies that  $\ell(q,u) \le \ell(q',u)$ for all $q' \in \RR$. If $q = u$, then 
	\begin{align}
		\ell(q,u)=\frac{|\alpha - 2|}{\alpha}\left(\left( 1\right)^{\alpha/2} - 1\right)=0.
	\end{align}  
	On the other hand, we have that
	\begin{align}
		\ell(q,u) \ge 0, \qquad \forall q,u\in \RR,     
	\end{align}
	since 
	$
	((\frac{(d/c)^2}{|\alpha - 2|} + 1)^{\alpha/2} - 1) \ge 0
	$ if $\alpha \ge 0$ and  
	$
	((\frac{(d/c)^2}{|\alpha - 2|} + 1)^{\alpha/2} - 1) \le 0
	$ if  $\alpha \le 0$. Therefore, for any $q,u \in \RR$, having $q = u$ implies that   $\ell(q,u) \le \ell(q',u)$ for all $q' \in \RR$. By using Eq. \eqref{eq:new:5}, this proves the statement for \nc. 
\end{proof}

\section{Other algorithm design options}\label{app:meta:design:other}

Apart from the most important options presented in Section~\ref{sec:meta:design} and Appendix~\ref{app:meta:design:multi} that pertain to the loss function parametrization, the number of inner optimization steps and imposing desirable loss function properties, some additional design options are presented in this section. 
First, consider the options of sampling new tasks $\Lambda$ from the task distribution $p(\lambda)$ (line 4 in Algorithm~\ref{algo:general}) and initializing the approximator NN parameters $\theta_{\tau}$ for $\tau \in \{1, \dots, T\}$ (line 6 in Algorithm~\ref{algo:general}). 
Resampling a set $\Lambda$ of $T$ tasks in every outer iteration exposes the learned loss to more samples from the task distribution. 
Similarly, solving these tasks with $T$ new randomly initialized NNs exposes the loss function to more samples from the NN initialization distribution.  
Although such introduced randomness is generally expected to improve test performance, it leads to unstable training that depends also on the number of inner optimization steps. 
As a result, we have the option to resample new tasks and new NN parameter initializations every $I'$ and $I''$ outer iterations, respectively, instead of every single iteration; i.e., $\theta_{\tau}$ is re-initialized in line 6 of Algorithm~\ref{algo:general} with a setting $\theta_{\tau}^{(0)}$ that is replaced every $I''$ iterations. 
An indicative experiment for demonstrating the effect of these options on training and test performance of the learned loss is performed in Appendix~\ref{app:add:results} for the regression example of Section~\ref{sec:examples:regression}. 

Finally, as a stopping criterion for Algorithm~\ref{algo:general}, i.e., for selecting the maximum number of outer iterations $I$, a performance metric can be recorded during meta-training and the algorithm can be stopped if no progress is observed for a number of iterations.  
Two candidate options for this metric are the following: (a) the outer objective of Eq.~\eqref{eq:lossml:outer:loss}, corresponding to a \textit{meta-training error}, and (b) the meta-test performance on a few test iterations with learned loss snapshots captured during meta-training, corresponding to a \textit{meta-validation error}. 
However, because of the aforementioned induced randomness during meta-training, the outer objective may be too noisy for it to be a useful metric, especially when only one task is used for each outer update ($T = 1$ in Algorithm~\ref{algo:general}); thus, either a less noisy option can be utilized (see Fig.~\ref{fig:reg:exp:train}) or a moving average can be recorded (see Fig.~\eqref{fig:burgers:trainloss:ffn}).
Furthermore, depending on the parametrization and other design options, the meta-validation error can also be noisy (see Figs.~\ref{fig:reg:exp:test}, \ref{fig:ad:traintest}, and \ref{fig:burgers:traintest}).
For this reason, in the computational examples of Section~\ref{sec:examples}, we meta-train for a sufficiently large number of outer iterations ($10{,}000$) that works reasonably well according to both metrics. 
We also capture snapshots of the learned loss and use all of them in meta-testing for gaining a deeper understanding of the algorithm's applicability for solving PDEs with PINNs. 

\section{Additional computational results related to the function approximation example}\label{app:add:results}

\paragraph*{Effect of loss function initialization.}
To demonstrate the effect of the initialization of the NN-parametrized loss function, the outer learning rate, as well as the gradient clipping approach discussed in Section~\ref{sec:meta:design:inner}, we first consider employing Algorithm~\ref{algo:general} with a randomly initialized loss function, a large learning rate equal to $5\times 10^{-2}$ and $J = 20$ inner optimization steps. 
SGD is considered as both the inner and outer optimizer, the approximator NN architecture consists of $3$ hidden layers with $40$ neurons each and $tanh$ activation function, the number of datapoints used for inner training is $N_u = 100$, and the number of datapoints used for updating the loss is $N_{u, val} = 1{,}000$; the approximator NN as well as $N_u$ and $N_{u, val}$ are the same for all experiments in this section.
Note that this is the only case in the computational examples that SGD is used as outer optimizer; in all other cases Adam is used.
In Fig.~\ref{fig:reg:expl:grad:info}, the norm of the loss parameters as well as the norm and the maximum of their gradient for each outer iteration are shown. 
The gradient on iteration $4$ explodes and leads to a large jump on the loss parameters norm as well.
Furthermore, Fig.~\ref{fig:reg:expl:grad:inner} shows the loss gradient norm for each outer iteration and as a function of inner iterations. 
Clearly, although for increasing inner steps $J$ we differentiate over a longer optimization path as explained in Section~\ref{sec:meta:design:inner}, the loss gradient norm does not necessarily increase with increasing $J$. 
For this reason, throughout the rest of the computational examples we use a gradient clipping approach for addressing the exploding gradient issue, instead of, for example, uniformly dividing all gradients by $J$. 
\begin{figure}[H]
	\centering
	\includegraphics[width=.5\linewidth]{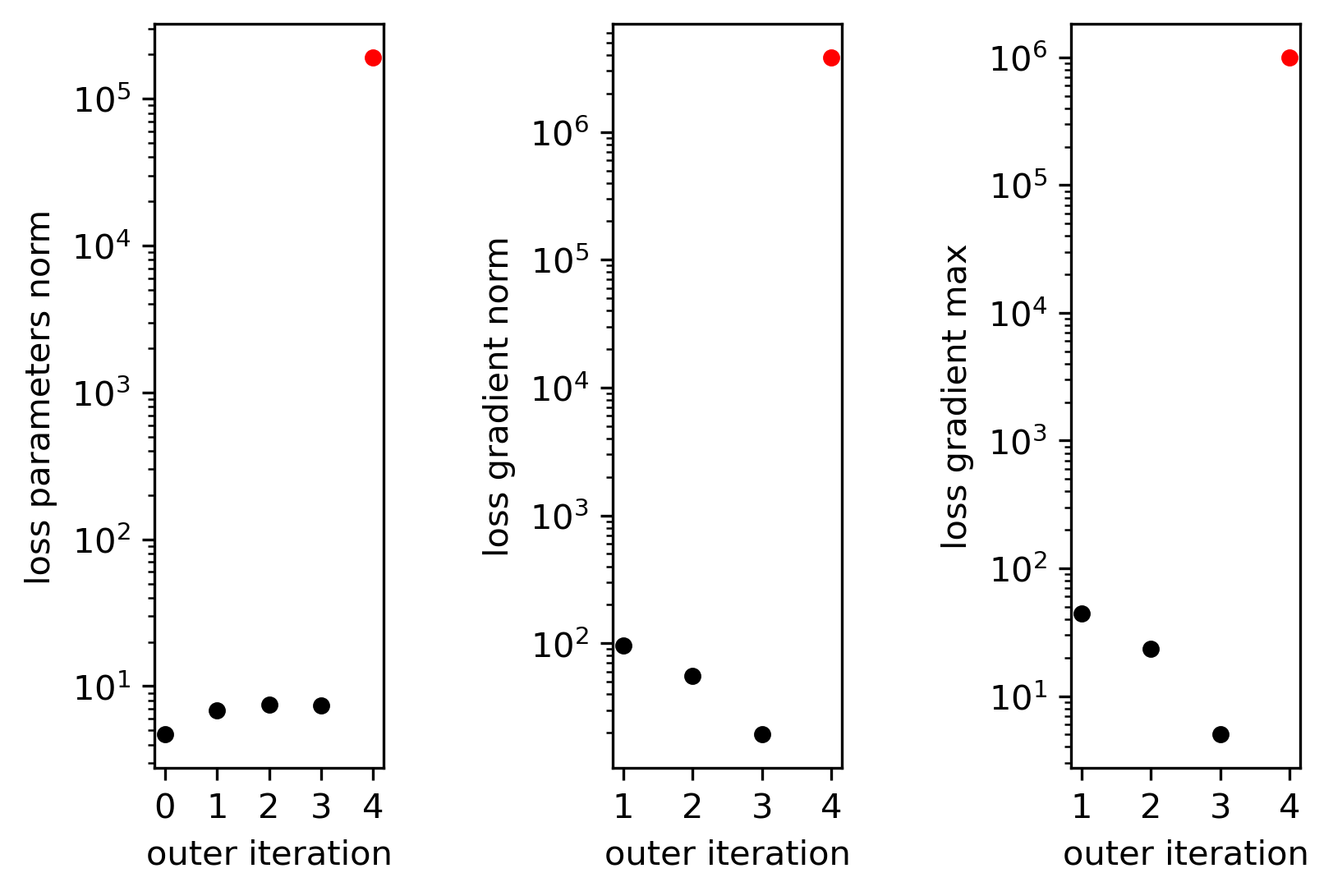}  
	\caption{Function approximation: Loss parameters norm as well as norm and maximum of their gradient for each outer iteration for the case of randomly initializing the loss function parameters and not using gradient clipping.
		The gradient on iteration $4$ explodes and leads to a large jump on the loss parameters norm as well.}
	\label{fig:reg:expl:grad:info}
\end{figure}
\begin{figure}[H]
	\centering
	\includegraphics[width=.5\linewidth]{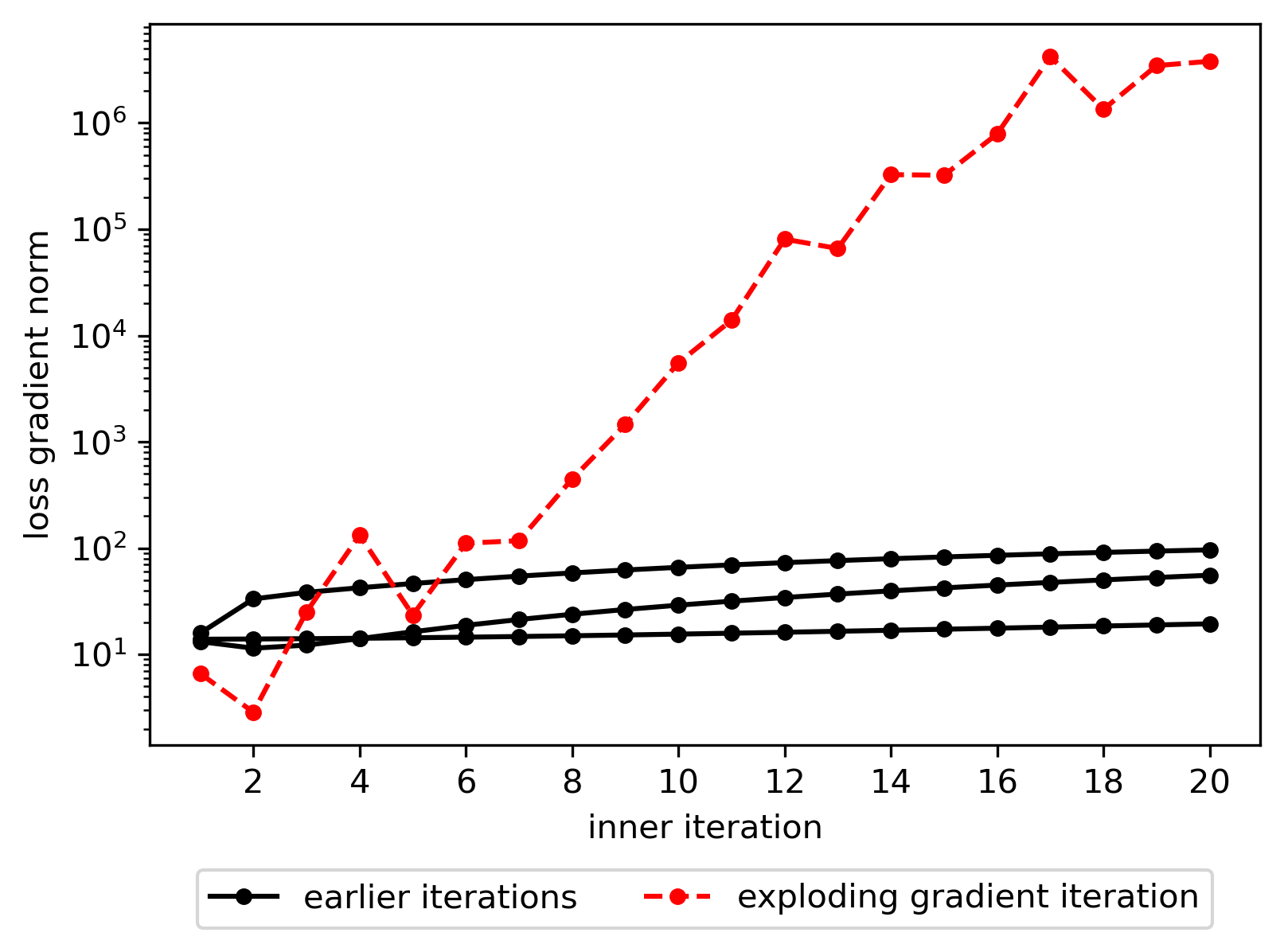}  
	\caption{Function approximation: Loss parameters gradient norm as a function of inner and outer iterations. 
		Gradient norm does not increase with increasing number of inner iterations for all outer iterations.}
	\label{fig:reg:expl:grad:inner}
\end{figure}

\paragraph*{Design options experiment.}
Next, to demonstrate the effect of the design options discussed in Section~\ref{sec:meta:design} and to evaluate the generalization capabilities of the algorithm we consider the following experiment: 
Algorithm~\ref{algo:general} is employed with an LAL-parametrized loss function initialized with an MSE approximation (see Section~\ref{sec:meta:design:param:LAL}), with Adam as inner optimizer, and with $I = 1{,}000$ outer steps. 
Furthermore, the number of inner steps $J$ varies between values in $\{1, 20\}$, the frequency according to which we sample new tasks varies between values in $\{1, 10, 100, 1000 = \text{no resampling}\}$, and whether in each outer iteration we use a newly initialized approximator NN is either True or False; see Appendix~\ref{app:meta:design:other} for relevant discussion. 
In Fig.~\ref{fig:reg:exp:train} we show the outer objective during meta-training and as a function of outer iteration for $J=1$ and $J=20$.
Clearly, resampling tasks and re-initializing the approximator NN introduces noise to the training as depicted by the corresponding outer objective trajectories shown with gray lines in Fig.~\ref{fig:reg:exp:train}; see also relevant results in Fig.~\ref{fig:burgers:trainloss:ffn} pertaining to the Burgers equation example.

Furthermore, in Fig.~\ref{fig:reg:exp:test}, the generalization capacity of the obtained loss functions is evaluated by performing meta-testing on $5$ ID unseen tasks for $100$ and $500$ test iterations. 
Note that each line in Fig.~\ref{fig:reg:exp:test} corresponds to the performance of $101$ different loss functions and is obtained by meta-training with different design options; i.e., we perform a test every $10$ outer iterations for each meta-training session. 
Overall, increasing the number of inner iterations improves the learned loss test performance as well as its robustness with increasing outer iterations and with varying design options.
Moreover, the patterns observed for 100 test iterations are almost identical to the ones observed for 500 test iterations, which means that if we attempt to optimize the design options based on meta-testing with 100 or 500 iterations we will end up with the same optimal ones.

\begin{figure}[H]
	\centering
	\begin{subfigure}[t]{0.48\textwidth}
		\centering
		\includegraphics[width=\textwidth]{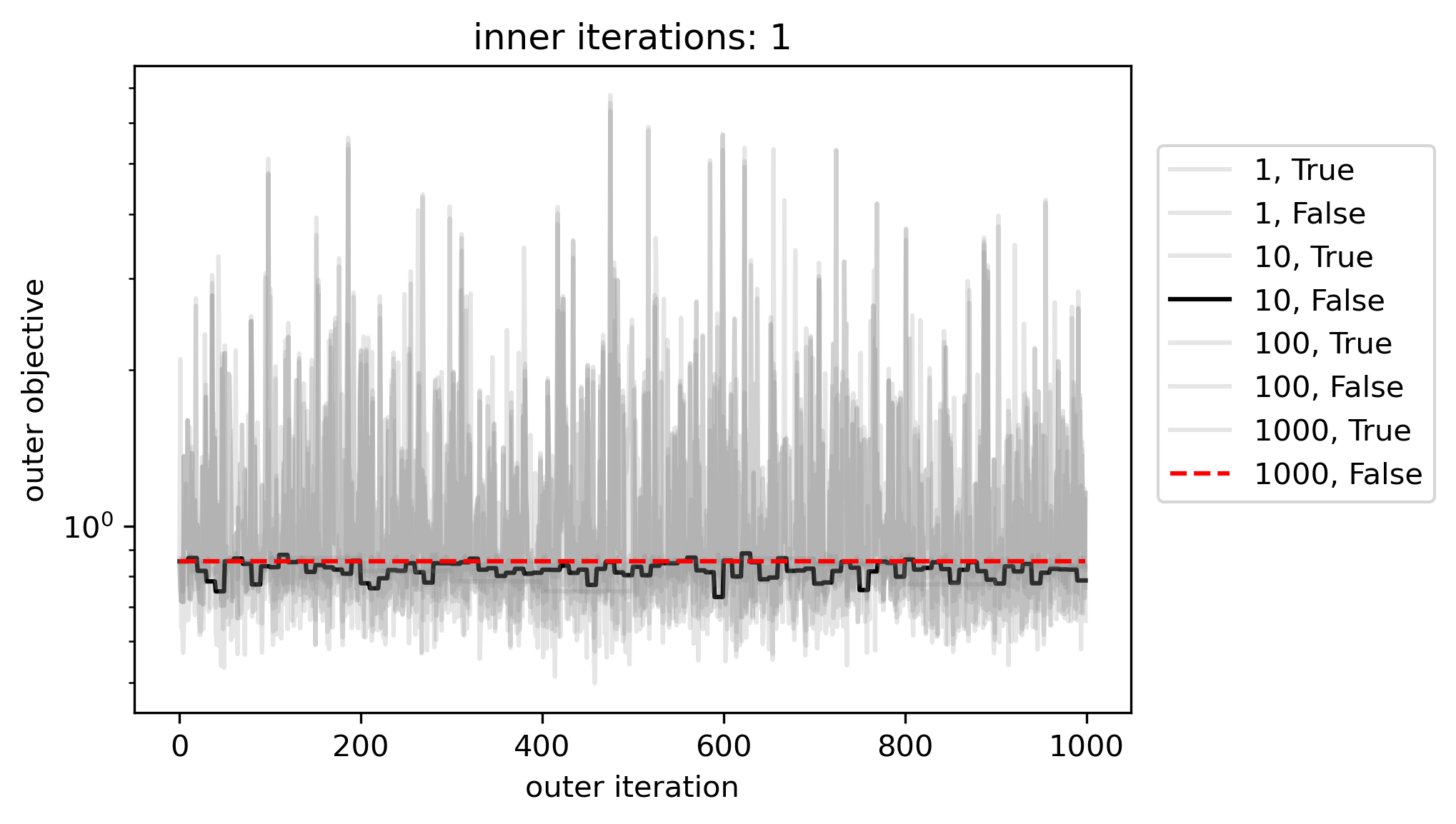}
		\caption{}
		\label{fig:reg:exp:train:1}
	\end{subfigure}
	\hfill
	\begin{subfigure}[t]{0.48\textwidth}
		\centering
		\includegraphics[width=\textwidth]{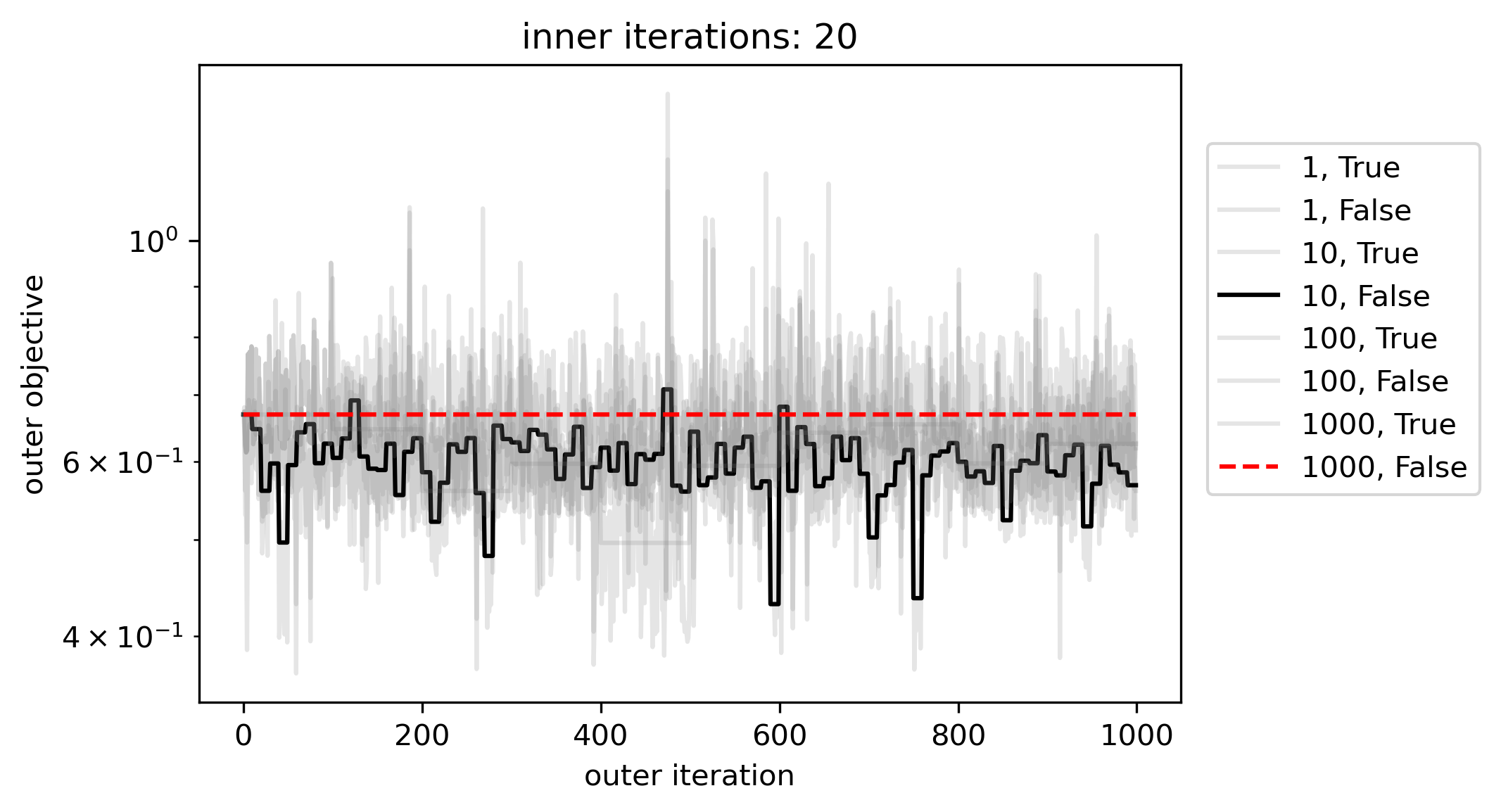}
		\caption{}
		\label{fig:reg:exp:train:20}
	\end{subfigure}
	\caption{Function approximation: Outer objective values recorded during meta-training ($1{,}000$ outer iterations) as a function of design options; these can be construed as \textit{meta-training error} trajectories.
		Results obtained using 1 inner iteration (a) and 20 inner iterations (b) in meta-training. 
		Design options considered are number of inner iterations $J \in \{1, 20\}$, frequency of resampling tasks $I' \in \{1, 10, 100, 1000 = \text{no resampling}\}$, and whether approximator NN is re-initialized with a new initialization setting or with the same one (True, False); see more information in Appendix~\ref{app:meta:design:other}.
		Resampling tasks and re-initializing the approximator NN introduces noise to the training.
		Black lines (task resampling every 10 outer iterations) are more noisy than the red lines (no task resampling), whereas the rest of the lines are shown with the same gray color for indicating noisy lines.} 
	\label{fig:reg:exp:train}
\end{figure}
\begin{figure}[H]
	\centering
	\begin{subfigure}[t]{0.48\textwidth}
		\centering
		\includegraphics[width=\textwidth]{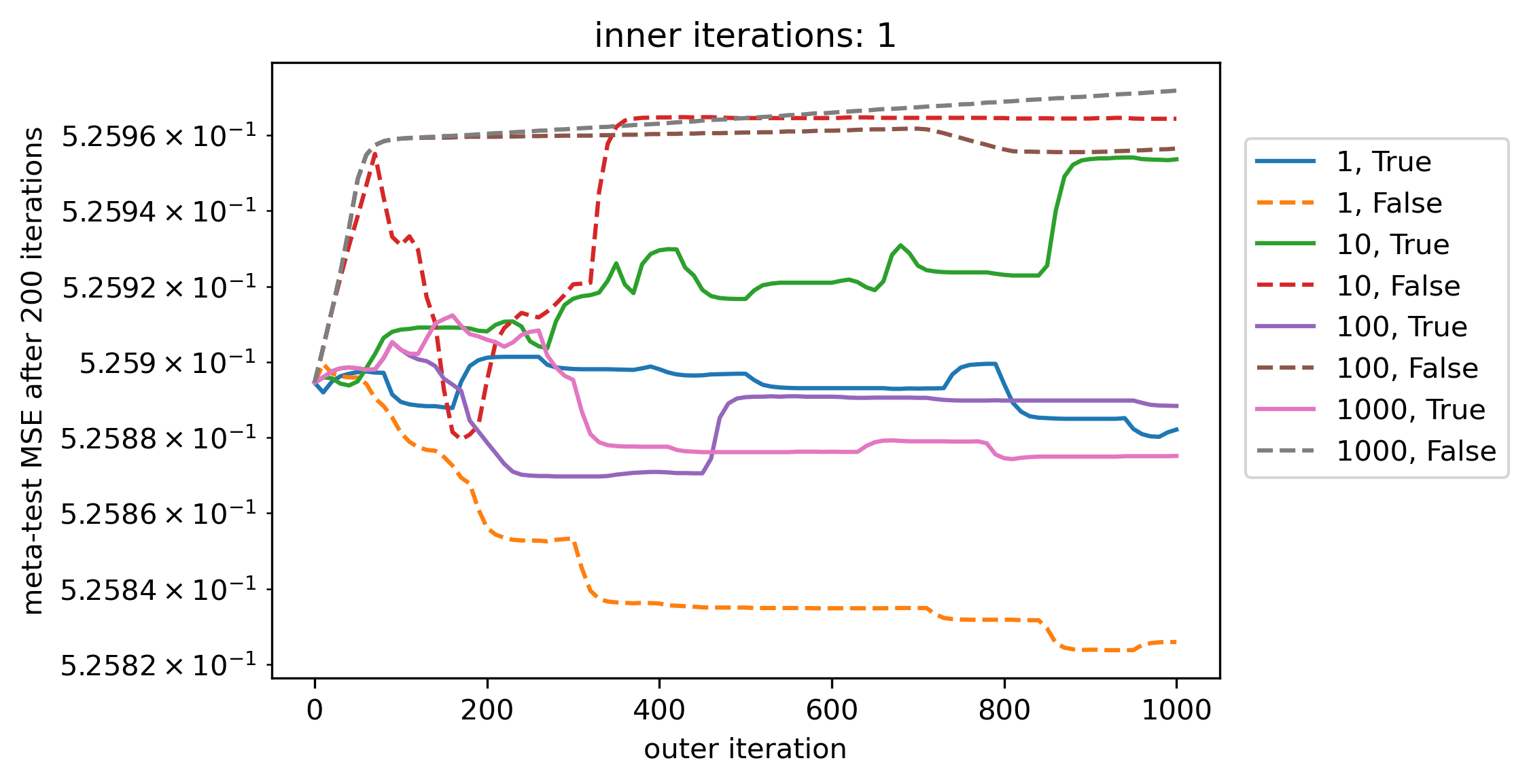}
		\caption{}
		\label{fig:reg:exp:test:1:100}
	\end{subfigure}
	\hfill
	\begin{subfigure}[t]{0.48\textwidth}
		\centering
		\includegraphics[width=\textwidth]{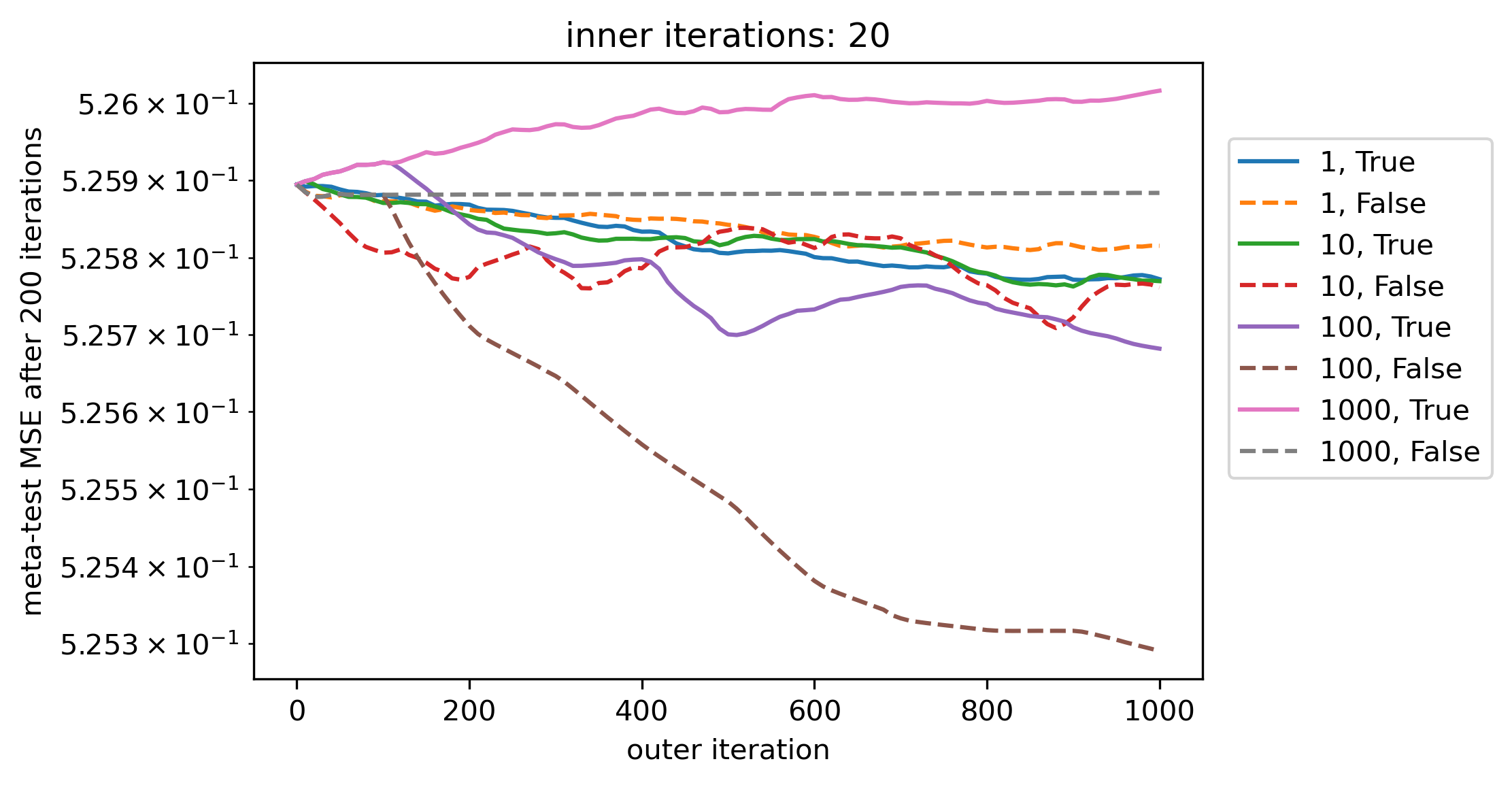}
		\caption{}		
		\label{fig:reg:exp:test:20:100}
	\end{subfigure}
	\begin{subfigure}[t]{0.48\textwidth}
		\centering
		\includegraphics[width=\textwidth]{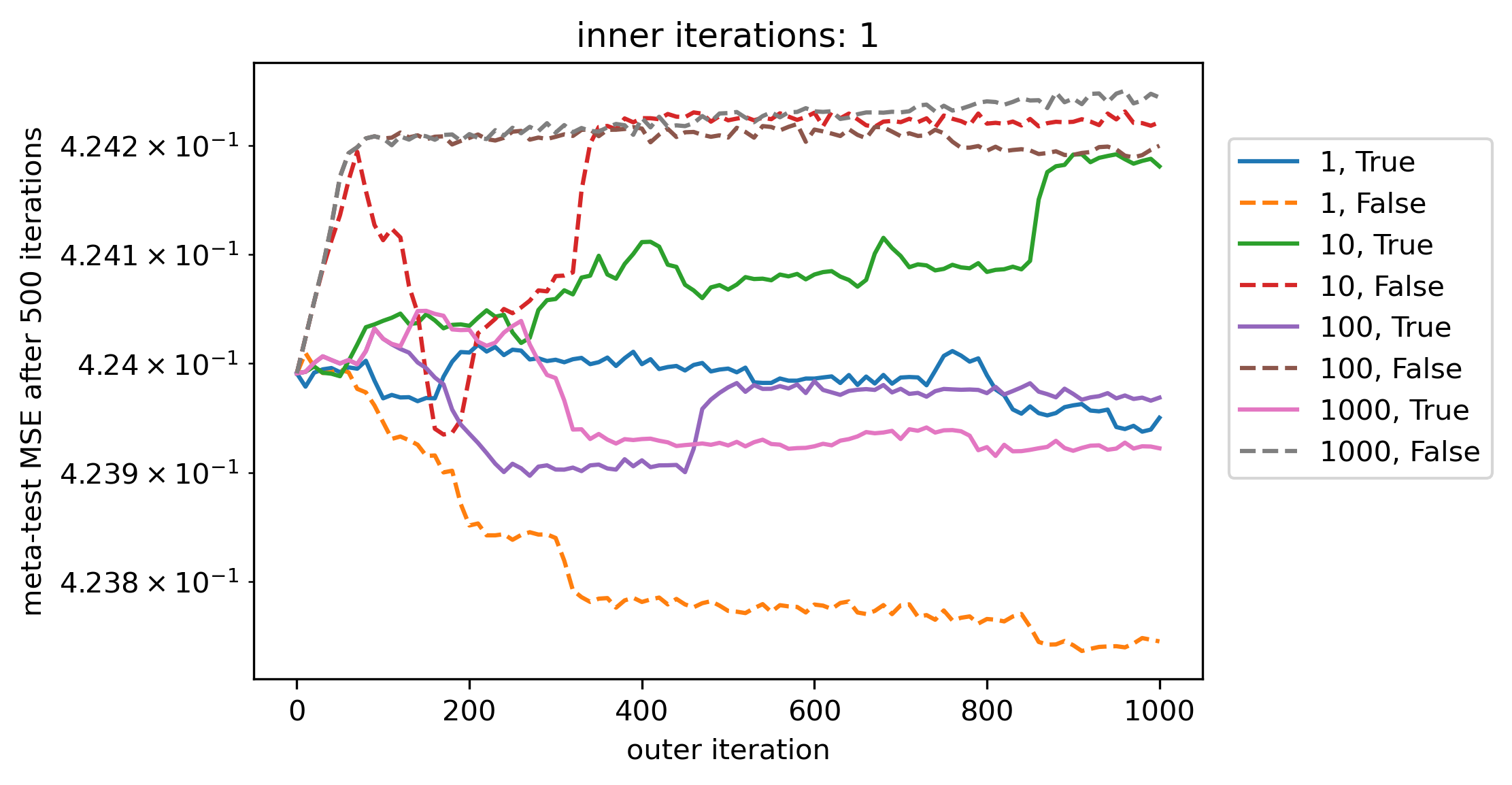}
		\caption{}
		\label{fig:reg:exp:test:1:500}
	\end{subfigure}
	\begin{subfigure}[t]{0.48\textwidth}
		\centering
		\includegraphics[width=\textwidth]{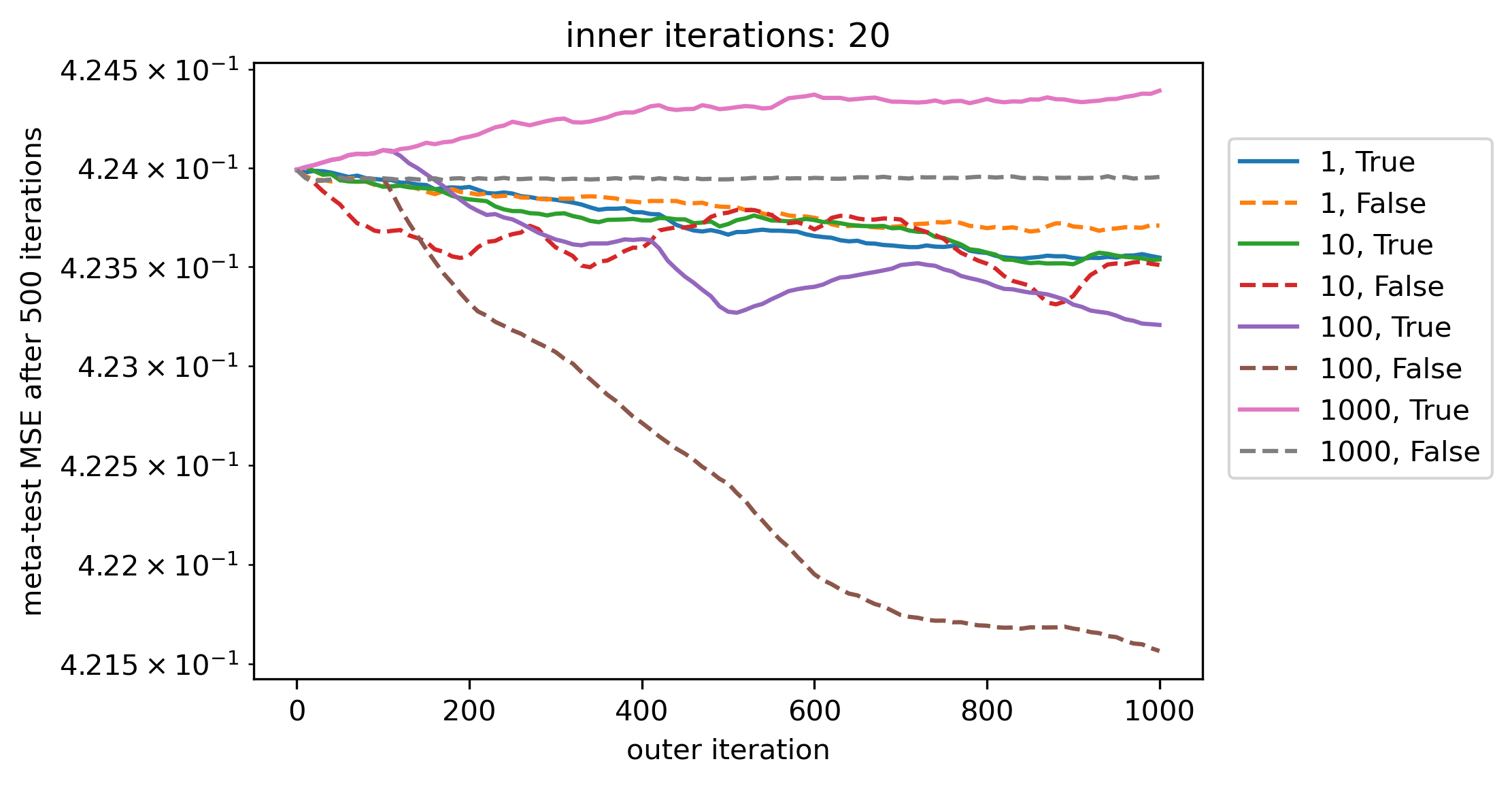}
		\caption{}
		\label{fig:reg:exp:test:20:500}
	\end{subfigure}
	\caption{Function approximation: Meta-testing performance (100 (a, b) and 500 (c, d) test iterations) as a function of outer iteration during meta-training ($1{,}000$ outer iterations) and of design options; these can be construed as \textit{meta-validation error} trajectories.
		Results obtained using 1 inner iteration (a, c) and 20 inner iterations (b, d) in meta-training. 
		Design options considered are number of inner iterations $J \in \{1, 20\}$, frequency of resampling tasks $I' \in \{1, 10, 100, 1000 = \text{no resampling}\}$, and whether approximator NN is re-initialized with a new initialization setting or with the same one (True, False); see more information in Appendix~\ref{app:meta:design:other}. 
		Clearly, increasing the number of inner iterations improves the learned loss test performance as well as its robustness with increasing outer iterations and with varying design options.
		Furthermore, the patterns observed for 100 test iterations are almost identical to the ones observed for 500 test iterations.
	}
	\label{fig:reg:exp:test}
\end{figure}
\end{appendices}

\end{document}